\documentclass[letterpaper, 10 pt, conference]{ieeeconf}  

\IEEEoverridecommandlockouts                              
\usepackage{cite} 
\usepackage{graphicx}
\usepackage{graphics} 
\usepackage{epsfig} 
\usepackage{mathptmx} 
\usepackage{times} 
\usepackage{amsmath} 
\usepackage{amssymb}  
\usepackage{url}
\usepackage{rotating}
\usepackage[font=footnotesize]{caption}
\usepackage[font=footnotesize]{subcaption}
\usepackage{cite}
\usepackage{ifthen,version}
\usepackage{color}
\usepackage{array}
\usepackage[usenames,dvipsnames,svgnames,table]{xcolor}
\usepackage{booktabs}

\newtheorem{theorem} {Theorem}
\newtheorem{lemma} {Lemma}

\usepackage[
  pdfauthor={Shervin Javdani, Matthew Klingensmith, J. Andrew Bagnell, Nancy S. Pollard, Siddhartha S. Srinivasa},
  pdftitle={Efficient Touch Based Localization through Submodularity},
  pdftex,
]{hyperref}
\hypersetup{letterpaper,bookmarksopen,bookmarksnumbered,
pdfpagemode=UseOutlines,
colorlinks=false,
pdfborder={0 0 0},
linkcolor=blue,
anchorcolor=blue,
citecolor=blue,
filecolor=blue,
menucolor=blue,
urlcolor=blue
}

\newboolean{include-notes}
\setboolean{include-notes}{true}
\newcommand{\ssnote}[1]{\ifthenelse{\boolean{include-notes}}%
 {\textcolor{blue}{\textbf{SS: #1}}}{}}
\newcommand{\sjnote}[1]{\ifthenelse{\boolean{include-notes}}%
 {\textcolor{purple}{\textbf{SJ: #1}}}{}}

\title{\LARGE \bf
Efficient Touch Based Localization through Submodularity%
\vspace{-1em}
}

\author{ \parbox{\linewidth}{\centering Shervin Javdani, Matthew Klingensmith, J. Andrew Bagnell, Nancy S. Pollard, Siddhartha S. Srinivasa\\
         The Robotics Institute, 
         Carnegie Mellon University\\
         {\tt \small \{sjavdani, mklingen, dbagnell, nsp, siddh\}@cs.cmu.edu}}
}

\newcommand{\eref}[1]{(\ref{#1})}
\newcommand{\sref}[1]{Section~\ref{#1}}
\newcommand{\figref}[1]{Fig.~\ref{#1}}

\newcommand{\tabref}[1]{Table~\ref{#1}}

\newcommand{\ttexttilde}{\hbox{$\scriptstyle\sim$}}

\newcommand{\allactionset}{\mathbb{A}}
\newcommand{\actionset}{A}
\newcommand{\tactionset}{\tilde{\actionset}}
\newcommand{\actionitem}{a}
\newcommand{\tactionitem}{\tilde{\actionitem}}
\newcommand{\allobservationset}{\mathbb{O}}
\newcommand{\allobservationsetaction}{\mathbb{O}_\actionitem}
\newcommand{\observationset}{\mathit{O}}
\newcommand{\tobservationset}{\tilde{\observationset}}
\newcommand{\observationitem}{o}
\newcommand{\tobservationitem}{\tilde{\observationitem}}

\newcommand{\realization}{\phi}
\newcommand{\noisyrealization}{ {\hat{\phi}}}
\newcommand{\partialrealization}{\psi}
\newcommand{\partialrealizationrandom}{\Psi}
\newcommand{\partialgivenactions}{\partialrealization_\actionset}
\newcommand{\partialgivenactionsrandom}{\partialrealizationrandom_\actionset}
\newcommand{\randrealization}{\Phi}
\newcommand{\noisyrandrealization}{\widehat{\Phi}}
\newcommand{\policy}{\pi}
\newcommand{\greedypolicy}{\pi^{greedy}}
\newcommand{\optpolicy}{\policy^*}
\newcommand{\infogain}{IG}

\newcommand{\weighting}{\omega}
\newcommand{\noisingfunc}{\Omega}
\newcommand{\noisingfuncaction}{\noisingfunc_\actionitem}
\newcommand{\noisingfuncactionset}{\noisingfunc_\actionset}
\newcommand{\noisingfuncallaction}{\noisingfunc_\allactionset}
\newcommand{\noisingfuncallotheraction}{\noisingfunc_{\allactionset \backslash \actionset}}
\newcommand{\numnoisy}{K}
\newcommand{\probmass}{m}
\newcommand{\proofprobmass}{m}
\newcommand{\probmassall}{M}
\newcommand{\probmasspartial}{\probmassall_{\partialrealization}}
\newcommand{\probmasspartialarg}[1]{\probmassall_{\partialrealization_{#1}}}
\newcommand{\probmasspartialaction}{\probmass_{\partialrealization,\actionitem,\observationitem}}
\newcommand{\probmasspartialargaction}[1]{\probmass_{\partialrealization_{#1},\actionitem,\observationitem}}
\newcommand{\probmasspartialactionprime}{\probmass_{\partialrealization,\actionitem,\observationitem'}}

\newcommand{\noisyprobwithweightingnumer}{\weighting_{\actionitem_{\noisyrealization'}} (\actionitem_{\realization})  } 
\newcommand{\noisyprobwithweightingnumerobs}{\weighting_{\observationitem} (\actionitem_{\realization})  } 
\newcommand{\noisyprobwithweightingnumert}{\weighting_{\tactionitem_{\noisyrealization'}} (\tactionitem_{\realization})  } 
\newcommand{\noisyprobwithweightingnumertobs}{\weighting_{\tobservationitem} (\tactionitem_{\realization})  } 
\newcommand{\noisyprobwithweightingnumernoapos}{\weighting_{\actionitem_{\noisyrealization}} (\actionitem_{\realization})  } 
\newcommand{\noisyprobwithweightingdenom}{\sum_{\noisyrealization'' \in \noisingfuncaction(\realization)} \weighting_{\actionitem_{\noisyrealization''} } (\actionitem_{\realization}) }
\newcommand{\noisyprobwithweightingdenomt}{\sum_{\noisyrealization'' \in \noisingfunc_{\tactionitem}(\realization)} \weighting_{\tactionitem_{\noisyrealization''} } (\tactionitem_{\realization}) }
\newcommand{\noisyprobwithweighting}{\frac{\noisyprobwithweightingnumer}{\noisyprobwithweightingdenom} }
\newcommand{\noisyprobwithweightingobs}{\frac{\noisyprobwithweightingnumerobs}{\noisyprobwithweightingdenom} }
\newcommand{\noisyprobwithweightingmaxval}{\max_{\noisyrealization' \in \noisingfuncaction(\realization)} \weighting_{\actionitem_{\noisyrealization'} } (\actionitem_{\realization}) }

\newcommand{\sumrealization}{\sum_{\realization \in \randrealization}}
\newcommand{\sumnoisyfromrealization}[1]{\sum_{\noisyrealization \in \noisingfunc_{#1}(\realization)}}
\newcommand{\sumnoisyfromnoisyrealization}[1]{\sum_{\noisyrealization' \in \noisingfunc_{#1}(\noisyrealization)}}

\newcommand{\prodpartial}{\prod_{ \{\tactionitem,\tobservationitem\} \in \partialrealization}}
\newcommand{\prodpartialdelta}{\prodpartial \delta_{\tactionitem_{\noisyrealization} \tobservationitem}}

\newcommand{\prodpartialweightobs}{\prodpartial \frac{\noisyprobwithweightingnumertobs}{\noisyprobwithweightingdenomt} }
\newcommand{\prodpartialweightobskappa}{\prodpartial \frac{\noisyprobwithweightingnumertobs}{\kappa} }
\newcommand{\prodpartialdeltaweight}{\prodpartial \frac{\noisyprobwithweightingnumert}{\noisyprobwithweightingdenomt} \delta_{\tactionitem_{\noisyrealization'} \tobservationitem}}

\newcommand{\footlabel}[2]{%
    \addtocounter{footnote}{1}%
    \footnotetext[\thefootnote]{%
        \addtocounter{footnote}{-1}%
        \refstepcounter{footnote}\label{#1}%
        #2%
    }%
    $^{\ref{#1}}$%
}

\newcommand{\utilthresh}{Q}

\newcommand{\scalestopsyellowcol}{0.34}
\newcommand{\scalestopsyellowcolleft}{310}
\newcommand{\scalestopsyellowcolbottom}{180}
\newcommand{\scalestopsyellowcolright}{375}
\newcommand{\scalestopsyellowcoltop}{180}

\newcommand{\actionfig}{0.22}
\newcommand{\actionfigleft}{310}
\newcommand{\actionfigbottom}{250}
\newcommand{\actionfigright}{360}
\newcommand{\actionfigtop}{20}

\newcommand{\setcoverfigscale}{0.32\columnwidth}
\newcommand{\setcoverfigleft}{287}
\newcommand{\setcoverfigbottom}{140}
\newcommand{\setcoverfigright}{360}
\newcommand{\setcoverfigtop}{290}

\newcommand{\robotresultsscale}{0.32\columnwidth}
\newcommand{\robotresultsleft}{750}
\newcommand{\robotresultsbottom}{0}
\newcommand{\robotresultsright}{500}
\newcommand{\robotresultstop}{250}

\begin{document}

\maketitle
\thispagestyle{empty}
\pagestyle{empty}

\begin{abstract}
Many robotic systems deal with uncertainty by performing a sequence of information gathering actions. In this work, we focus on the problem of efficiently constructing such a sequence by drawing an explicit connection to submodularity. Ideally, we would like a method that finds the \emph{optimal} sequence, taking the minimum amount of time while providing sufficient information. Finding this sequence, however, is generally intractable. As a result, many well-established methods select actions greedily. Surprisingly, this often performs well. Our work first explains this high performance -- we note a commonly used metric, reduction of Shannon entropy, is submodular under certain assumptions, rendering the greedy solution comparable to the optimal plan in the \emph{offline} setting. However, reacting \emph{online} to observations can increase performance. Recently developed notions of \emph{adaptive submodularity} provide guarantees for a greedy algorithm in this \emph{online} setting. In this work, we develop new methods based on adaptive submodularity for selecting a sequence of information gathering actions online. In addition to providing guarantees, we can capitalize on submodularity to attain additional computational speedups. We demonstrate the effectiveness of these methods in simulation and on a robot.
\end{abstract}

\section{Introduction}

Uncertainty is a fundamental problem in robotics. It accumulates from various sources such as noisy sensors, inaccurate models, and poor calibration. This is particularly problematic for \emph{fine manipulation tasks}~\cite{lazano_1984_fine_motion}, such as grasping and pushing a small button on a drill, hooking the fingers of a hand around a door handle and turning it (two running examples in our paper), or inserting a key into a keyhole. Because these tasks require high accuracy, failing to account for uncertainty often results in catastrophic failure.

To alleviate these failures, many works perform a sequence of uncertainty reducing actions prior to attempting the task~\cite{cassandra_1996_acting_under_uncertainty, fox_1998_active_localization, hsiao_2008_robust_belief, hebert_next_best_touch_2012, batting_particle_filter}. In this work, we address the efficient automatic construction of such a sequence when information gaining actions are \emph{guarded moves}~\cite{will_1975_guarded}. See \figref{fig:door_touching_ex} for an example sequence which enabled a successful grasp of door handle with a noisy pose estimate.



Ideally, the selected actions reduce uncertainty enough to accomplish the task while optimizing a performance criterion like minimum energy or time. Computing the \emph{optimal} such sequence can be formulated as a Partially Observable Markov Decision Process (POMDP)~\cite{kaelbling_1998_pomdp}. However, finding optimal solutions to POMDPs is PSPACE complete~\cite{papad_1987_mdp_complexity}. Although several promising approximate methods have been developed~\cite{roy_2005_belief_compression,smith_2005_point_pomdp, kurniawati_2008_sarsop, shani_2012_survey_pointbased}, they are still not well suited for many manipulation tasks due to the continuous state and observation spaces. 

Previous work on uncertainty reduction utilizes online planning within the POMDP framework, looking at locally reachable states during each decision step~\cite{ross_2008_online_pomdp}. In general, these methods limit the search to a low horizon~\cite{kaijen_thesis}, often using the \emph{greedy} strategy of selecting actions with the highest expected benefit in one step~\cite{cassandra_1996_acting_under_uncertainty, fox_1998_active_localization, hsiao_2008_robust_belief, hebert_next_best_touch_2012}. This is out of necessity - computational time increases exponentially with the search depth. However, this simple greedy strategy often works surprisingly well.

\begin{figure}[t]
\centering
\setlength\fboxsep{0pt}
\setlength\fboxrule{0.5pt}
  \fbox{\includegraphics[width=0.495\columnwidth, trim=480 300 330 55, clip=true]{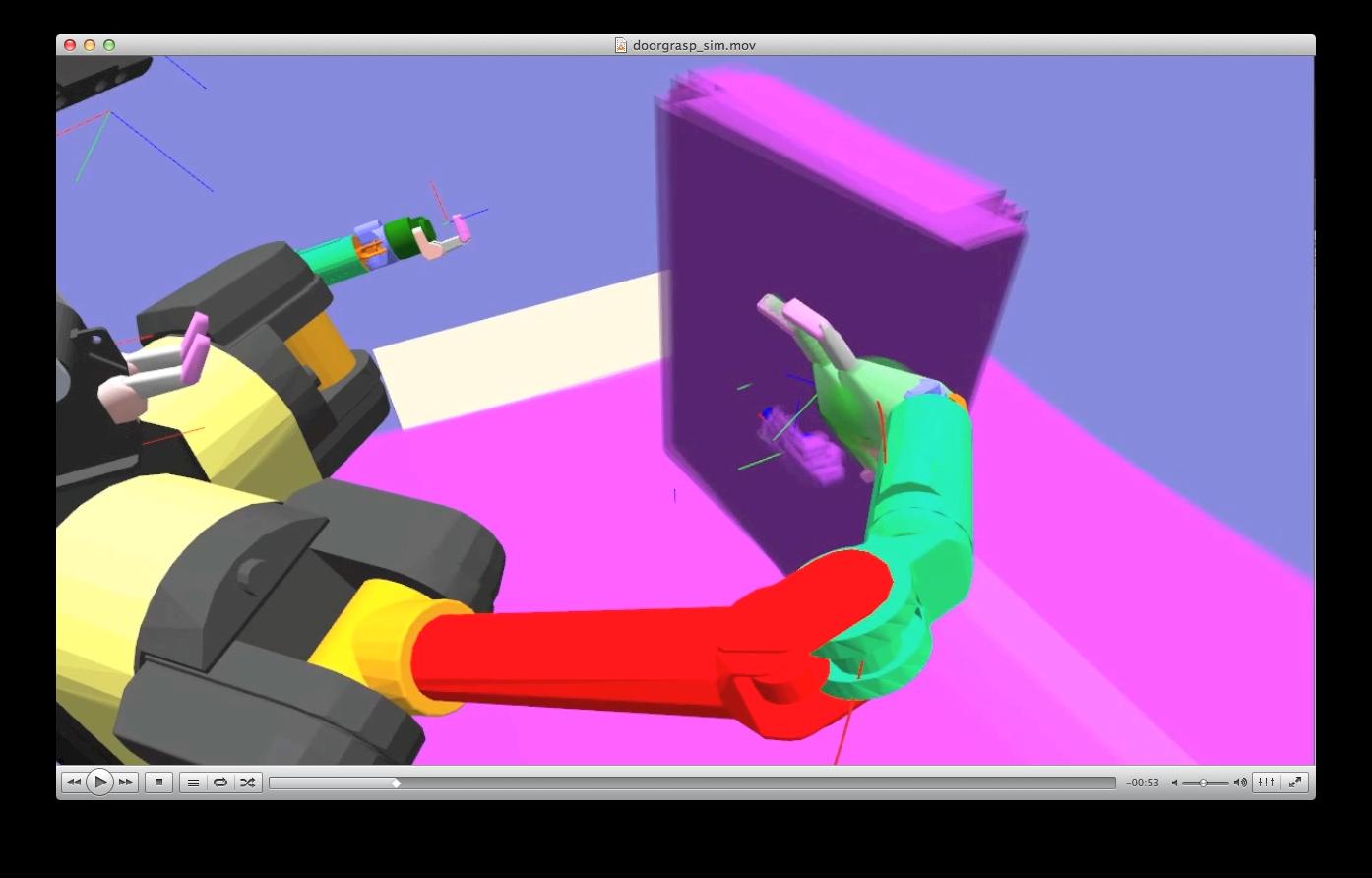}}\hspace{-1.5mm}
	\fbox{\includegraphics[width=0.495\columnwidth, trim=480 300 330 55, clip=true]{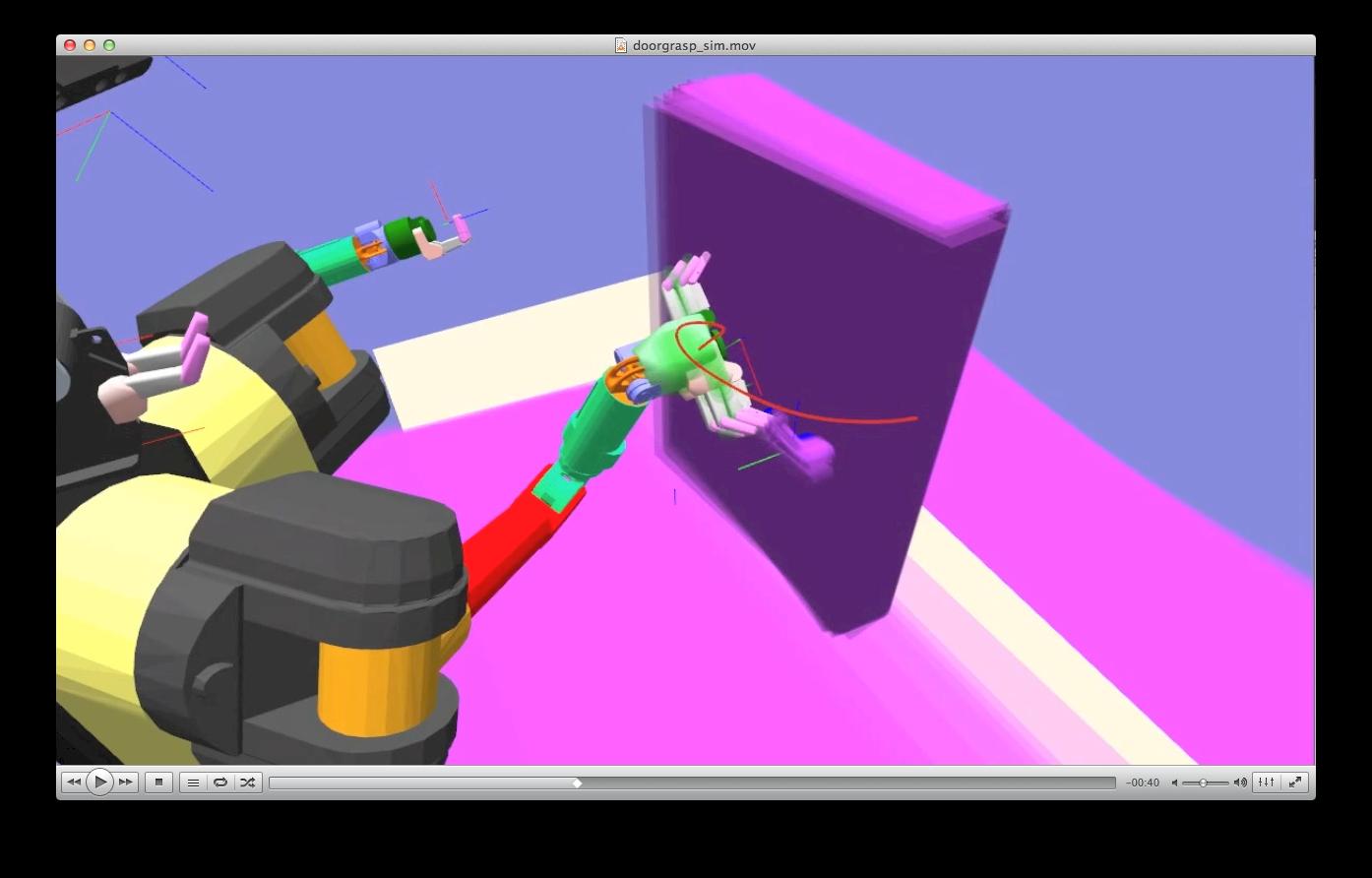}}
  \fbox{\includegraphics[width=0.495\columnwidth, trim=480 300 330 55, clip=true]{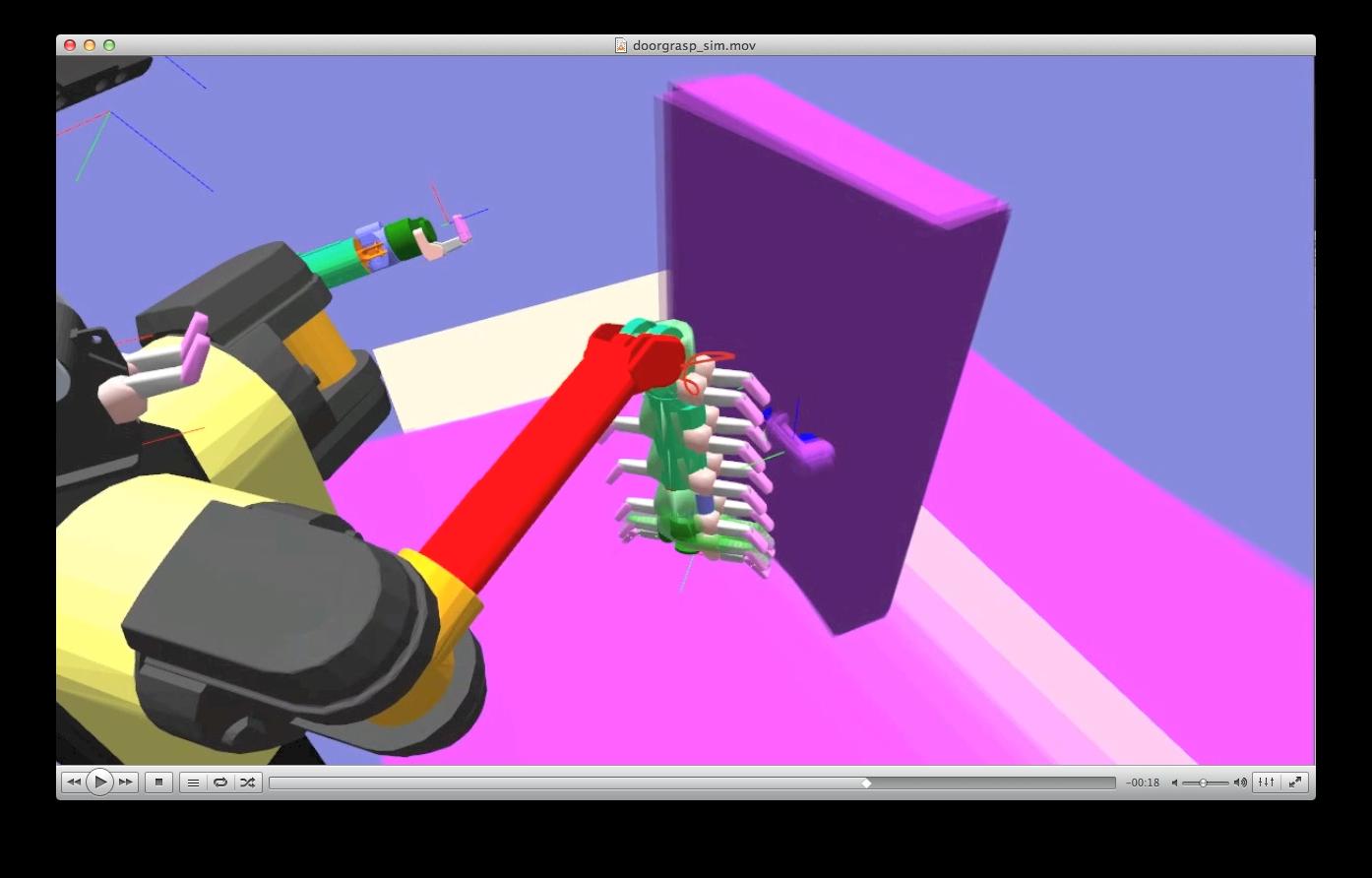}}\hspace{-1.5mm}
	\fbox{\includegraphics[width=0.495\columnwidth, trim=421 129.5 165 300, clip=true]{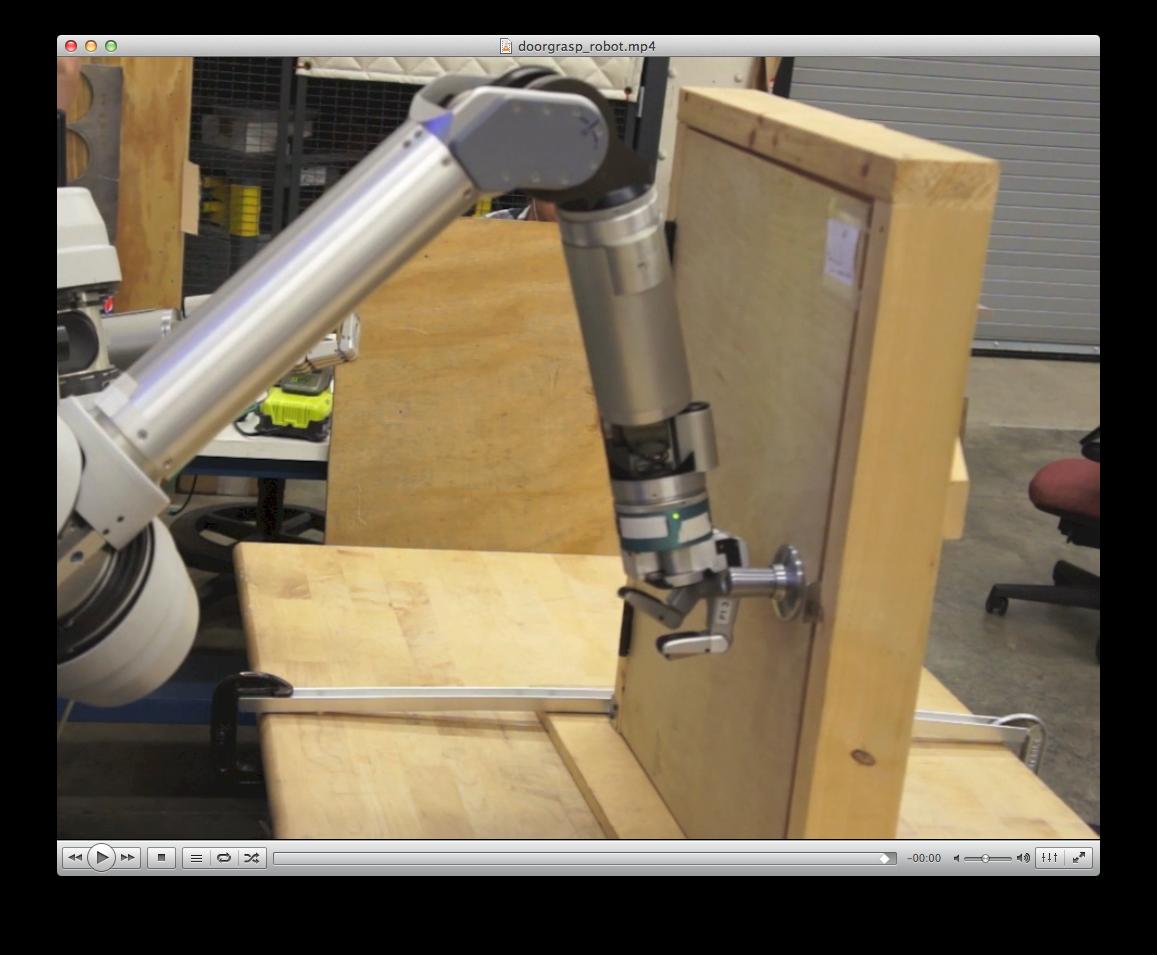}}
	\caption{We adaptively select a sequence of touch actions to reduce uncertainty. Here, we show actions selected by our Hypothesis Pruning method, enabling a successful grasp.}
	\label{fig:door_touching_ex}
\end{figure}

One class of problems known to perform well with a greedy strategy is \emph{submodular maximization}. A metric is submodular if it exhibits the diminishing returns property, which we define rigorously in \sref{sec_submodular}. A striking feature of submodular maximizations is that the greedy strategy is provably near-optimal. Furthermore, no polynomial time algorithm can guarantee optimality~\cite{feige_1998, wolsey_1982}. 


One often used metric for uncertainty reduction is the expected decrease in Shannon entropy~\cite{cassandra_1996_acting_under_uncertainty, fox_1998_active_localization, bourgault_2002_info_exploration, zheng_2005_active_diagnosis, batting_particle_filter, erickson_2008_blind, hsiao_2008_robust_belief, hebert_next_best_touch_2012}. This is referred to as the \emph{information gain} metric, and has been shown to be submodular under certain assumptions~\cite{krause_2005_submodular_entropy}. Not surprisingly, many robotic systems which perform well with a low horizon use this metric~\cite{cassandra_1996_acting_under_uncertainty, fox_1998_active_localization, bourgault_2002_info_exploration, hsiao_2008_robust_belief, hebert_next_best_touch_2012}, though most do not make the connection with submodularity. We note that Hsiao mentions that touch localization could be formulated as a submodular maximization~\cite{kaijen_thesis}. 

The guarantees for submodular maximization only hold in the \emph{non-adaptive} setting. That is, if we were to select a sequence of actions offline, and perform the same sequence regardless of which observations we received online, greedy action selection would be near-optimal. Unfortunately, it has been shown that this can perform exponentially worse than a greedy \emph{adaptive} algorithm for information gain~\cite{hollinger_isrr_2011}. Thus, while there are no formal guarantees for performing a submodular maximizations online, we might hope for good performance.

Recent notions of \emph{adaptive submodularity}~\cite{golovin_adaptive_2011} extend the guarantees of submodularity to the adaptive setting, requiring properties similar to those of submodular functions. Unfortunately, information gain does not have these properties. With information gain as our inspiration, we design a similar metric that does. In addition to providing guarantees with respect to that metric, formulating our problem as an adaptive submodular maximization enables a computational speedup through a lazy-greedy algorithm~\cite{minoux_lazy,golovin_adaptive_2011} which does not reevaluate every action at each step.


We present three greedy approaches for selecting uncertainty reducing actions. The first is our variant of information gain. This approach is similar to previous works~\cite{cassandra_1996_acting_under_uncertainty, fox_1998_active_localization, bourgault_2002_info_exploration, zheng_2005_active_diagnosis, batting_particle_filter, erickson_2008_blind, hsiao_2008_robust_belief, hebert_next_best_touch_2012}, though we also enforce the assumptions required for submodular maximization. The latter two maximize the expected number of hypotheses they disprove. We show these metrics are adaptive submodular. We apply all methods to selecting touch based sensing actions and present results comparing the accuracy and computation time of each in Section~\ref{sec_experiments}. Finally, we show the applicability of these methods on a real robot.


\section{Related work} \label{sec_related_works}


Hsiao et al.~\cite{hsiao_2008_robust_belief, kaijen_thesis} select a sequence of uncertainty reducing tactile actions through forward search in a POMDP. Possible actions consist of pre-specified world-relative trajectories~\cite{hsiao_2008_robust_belief}, motions based on the current highest probability state. Actions are selected using either information gain or probability of success as a metric~\cite{kaijen_thesis}, with a forward search depth of up to three actions. Aggressive pruning and clustering of observations makes online selection tractable. While Hsiao considers a small, focused set of actions (typically \ttexttilde5) at a greater depth, we consider a broad set of actions (typically \ttexttilde150) at a search depth of one action.

Hebert et al.~\cite{hebert_next_best_touch_2012} independently approached the problem of action selection for touch based localization. They utilize a greedy information gain metric, similar to our own. However, they do not make a connection to submodularity, and provide no theoretical guarantees with their approach. Additionally, they model noise only in $X,Y,Z$, while we use $X,Y,Z,\theta$. Furthermore, by using a particle based representation instead of a histogram (as in~\cite{kaijen_thesis, hebert_next_best_touch_2012}), we can model the underlying belief distribution more efficiently.

Others forgo the ability to plan with the entire belief space altogether, projecting onto a low-dimensional space before generating a plan to the goal. During execution, this plan will likely fail, because the true state was not known. Erez and Smart use local controllers to adjust the trajectory~\cite{erez_smart_localpomdp}. Platt et al. note when the belief space diverges from what the plan expected, and re-plan from the new belief~\cite{platt_2011_hypoth_based_non_gaussian}. They prove their approach will eventually converge to the true hypothesis. While these methods plan significantly faster due to their low-dimensional projection, they pick actions suboptimally. Furthermore, by ignoring part of the belief space, they sacrifice the ability to avoid potential failures. For example, these methods cannot guarantee that a trajectory will not collide and knock over an object, since the planner may ignore the part of the belief space where the object is actually located.

Petrovskaya et al.~\cite{petrov_localization} consider the problem of full 6DOF pose estimation of objects through tactile feedback. Their primary contribution is an algorithm capable of running in the full 6DOF space quickly. In their experiments, action selection was done randomly, as they do not attempt to select optimal actions. To achieve an error of \ttexttilde $5mm$, they needed an average of 29 actions for objects with complicated meshes. While this does show that even random actions achieve localization eventually, we note that our methods take significantly fewer actions.

In the DARPA Autonomous Robotic Manipulation Software (ARM-S) competition, teams were required to localize, grasp, and manipulate various objects within a time limit. Many teams first took uncertainty reducing actions before attempting to accomplish tasks~\cite{arms_video}. Similar strategies were used to enable a robot to prepare a meal with a microwave~\cite{herb_video}, where touch-based actions are used prior to pushing buttons. To accomplish these tasks quickly, some of these works rely on hand-tuned motions and policies, specified for a particular object and environment. While this enables very fast localization with high accuracy, a sequence must be created manually for each task and environment. Furthermore, these sequences aren't entirely adaptive.

Dogar and Srinivasa~\cite{dogar_2010_push_grasp} use the natural interaction of an end effector and an object to handle uncertainty with a push-grasp. By utilizing offline simulation, they reduce the online problem to enclosing the object's uncertainty in a pre-computed capture region. Online, they simply plan a push-grasp which encloses the uncertainty inside the capture region. This work is complimentary to ours - the push-grasp works well on objects which slide easily, while we assume objects do not move. We believe each approach is applicable in different scenarios.

Outside of robotics, many have addressed the problem of query selection for identification. In the \emph{noise-free} setting, a simple adaptive algorithm known as generalized binary search (GBS)~\cite{nowak_gbs} is provably near optimal. Interestingly, this algorithm selects queries identical to greedy information gain if there are only two outcomes~\cite{zheng_2005_active_diagnosis}. The GBS method was extended to multiple outcomes, and shown to be adaptive submodular~\cite{golovin_adaptive_2011}. Our Hypothesis Pruning metric is similar to this formulation, but with different action and observation spaces that enable us to model touch actions naturally.

Recently, there have been guarantees made for the case of \emph{noisy} observations. For binary outcomes and independent, random noise, the GBS was extended to noisy generalized binary search~\cite{nowak_noisy_gbs}. For cases of persistent noise, where performing the same action results in the same noisy outcome, adaptive submodular formulations have been developed based on eliminating noisy versions of each hypothesis~\cite{golovin_bayesian_noisy_obs, bellala_2012_query_selection}. In all of these cases, the message is the same - with the right formulation, greedy selection performs well for uncertainty reduction.


\section{Problem Formulation}
\label{sec_formulation}
We review the basic formulation for adaptive submodular maximization. For a more detailed explanation, see~\cite{golovin_adaptive_2011}.

Let a possible object state be $\realization$, called the \emph{realization}. Let $\randrealization$ be a random variable over all realizations. Thus, the probability of a certain state is given by $p(\realization) = \mathbb{P}\left[\randrealization = \realization \right]$. At each decision step, we select an action $\actionitem$ from $\allactionset$, the set of all available actions, which incurs a cost $c(\actionitem)$. Each action will result in some observation $\observationitem$ from $\allobservationset$, the set of all possible observations. We assume that given a realization $\realization$, the outcome of an action $\actionitem$ is deterministic. Let $\actionset \subseteq \allactionset$ be all the actions selected so far. During execution, we maintain the \emph{partial realization} $\partialgivenactions$, a sequence of observations received indexed by $\actionset$. We call it a partial realization as it encodes how realizations $\realization \in \randrealization$ agree with observations.

For the case of tactile localization, $\realization$ is the object pose. $\allactionset$ corresponds to all end-effector guarded move trajectories~\cite{will_1975_guarded}, which terminate when the hand touches an obstacle. $\allobservationset$ encompasses any possible observation, which is the set of all distances along any trajectory within which the guarded move may terminate. The partial realization $\partialgivenactions$ essentially encodes the ``belief state'' used in POMDPs, which we denote by $p(\realization | \partialgivenactions) = \mathbb{P}\left[\randrealization = \realization | \partialgivenactions \right]$.

Our goal is to find an adaptive policy for selecting actions based on observations so far. Formally, a policy $\policy$ is a mapping from a partial realization $\partialgivenactions$ to an action item $\actionitem$. Let $\actionset(\policy, \realization)$ be the set of actions selected by policy $\policy$ if the true state is $\realization$. We define two cost functions for a policy - the average cost and the worst case cost. These are:
\begin{align*}
  c_{avg} &= \mathbb{E}_\randrealization \left[ c(A(\pi, \randrealization))\right]\\
  c_{wc} &= \max_\realization c(A(\pi, \realization))
\end{align*}

Define some utility function $f : 2^\allactionset \times \allobservationset^\allactionset \rightarrow \mathbb{R}_{\geq 0}$, which depends on actions selected and observations received. We would like to find a policy which that will reach some utility threshold $\utilthresh$ while minimizing one of our cost functions. Formally:
\begin{align*}
  \min \hspace{1mm} &c_{\{avg,wc\}}(A(\pi, \randrealization))\\
  & s.t. f(A(\pi, \realization), \realization) \geq Q, \forall \realization
\end{align*}

This is often referred to as the \emph{Minimum Cost Cover} problem, where we achieve some coverage $Q$ while minimizing the cost to do so. We can consider optimal policies $\optpolicy_{avg}$ and $\optpolicy_{wc}$ for the above, optimized for their respective cost functions. Unfortunately, obtaining even approximate solutions is difficult~\cite{feige_1998, golovin_adaptive_2011}. However, a simple greedy algorithm achieves near-optimal performance if our objective function $f$ satisfies properties of adaptive submodularity and monotonicty. We now briefly review these properties.

\subsection{Submodularity}
\label{sec_submodular}

First, let us consider the case when we do not condition on observations, optimizing an offline plan. We call a function $f$ submodular if whenever $X \subseteq Y \subseteq \allactionset$, $a \in \allactionset \backslash Y$:

\textbf{Submodularity} (diminishing returns):
\begin{align*}
  f(X \cup \{a\}) - f(X) \geq f(Y \cup \{a\}) - f(Y)
\end{align*}

The marginal benefit of adding $a$ to a smaller set $X$ is at least as much as adding it to the superset $Y$. We also require monotonicty, or that adding more elements never hurts:

\textbf{Monotonicity} (more never hurts):
\begin{align*}
f(X \cup \{a\}) - f(X) \geq 0
\end{align*}

The greedy algorithm maximizes $\frac{f(\actionset \cup \{a\}) - f(\actionset)}{c(a)}$, the marginal utility per unit cost. As observations are not incorporated, this corresponds to an offline plan. If submodularity and monotonicty are satisfied, the greedy algorithm will be within a $(1+\ln \max_a f(a))$ factor of $c_{avg}(A(\optpolicy_{avg}, \randrealization))$ for integer valued $f$~\cite{wolsey_1982}.

\subsection{Adaptive Submodularity}
\label{sec_adaptivesubmodular}
Now we consider the case where the policy adapts to new observations~\cite{golovin_adaptive_2011}. In this case, the expected marginal benefit of performing an action is:
\begin{align*}
  \Delta(a | \partialgivenactions) &= \mathbb{E}\left[ f(\actionset \cup \{a\}, \randrealization) -   f(\actionset, \randrealization) \right | \partialgivenactions]
\end{align*}

We call a function $f$ adaptive submodular if whenever $X \subseteq Y \subseteq \allactionset$, $a \in \allactionset \backslash Y$:

\textbf{Adaptive Submodularity}:
\begin{align*}
  \Delta(a | X) \geq \Delta(a | Y)
\end{align*}

That is, the expected benefit of adding $a$ to a smaller set $X$ is at least as much as adding it to the superset $Y$, for any set of observations received from actions $Y \backslash X$. We also require strong adaptive monotonicity, or more items never hurts. For any $a \notin X$, and any possible outcome $o$, this requires:

\textbf{Strong Adaptive Monotonicity}:
\begin{align*}
  \mathbb{E}\left[ f(X, \randrealization) | \partialrealization_X \right] &\leq \mathbb{E}\left[ f(X \cup \{a\}, \randrealization) | \partialrealization_X, \partialrealization_a = o \right]
\end{align*}

In this case, the greedy algorithm maximize $\frac{\Delta(a|\partialrealization_X)}{c(a)}$. This encodes an online policy, since at each $\partialrealization_X$ incorporates the new observations. Surprisingly, we can bound the performance of the same algorithm with respect to both the optimal average case policy $\optpolicy_{avg}$ and optimal worst case policy $\optpolicy_{wc}$. This has been shown to have be within a $(1+\ln(Q))$ factor of $\optpolicy_{avg}$, and a $(1+\ln(\frac{Q}{\min_{\realization}p(\realization)}))$ factor of $\optpolicy_{wc}$ approximation for integer valued $f$, for self-certifying instances (see~\cite{golovin_adaptive_2011} for a more detailed explanation).

\section{Application to Touch Localization}
We would like to appeal to the above algorithms and guarantees for touch localization, while still maintaining generality for different objects and motions. Given an object mesh, we model the random realization $\randrealization$ as a set of sampled particles. We can think of each particle $\realization \in \randrealization$ representing some hypothesis of the true object pose.

Each action $\actionitem \in \allactionset$ corresponds to an end-effector trajectory which stops when the object is touched. The cost $c(\actionitem)$ is the time it would take to run this entire trajectory, plus some fixed time for moving to the start pose. An observation $\observationitem \in \mathbb{R}$ corresponds to the time it takes for the end-effector to make contact with the object. We define $\actionitem_\realization$ as the time during trajectory $\actionitem$ where contact first occurs if the true state were $\realization$. See Figure~\ref{fig:handstops} for an example. If the swept path of $\actionitem$ does not contact object $\realization$, then $\actionitem_\realization = \infty$, which is a valid observation.

With this formulation, we first discuss some assumptions made about interactions with the world. We then present our different utility functions $f$, which capture the idea of reducing the uncertainty in $\randrealization$. In general, our objective will be to achieve a certain amount of uncertainty reduction while minimizing the time to do so.


\begin{figure}[t]
  \includegraphics[trim=\setcoverfigleft px \setcoverfigbottom px \setcoverfigright px \setcoverfigtop px, clip=true, width=\setcoverfigscale]{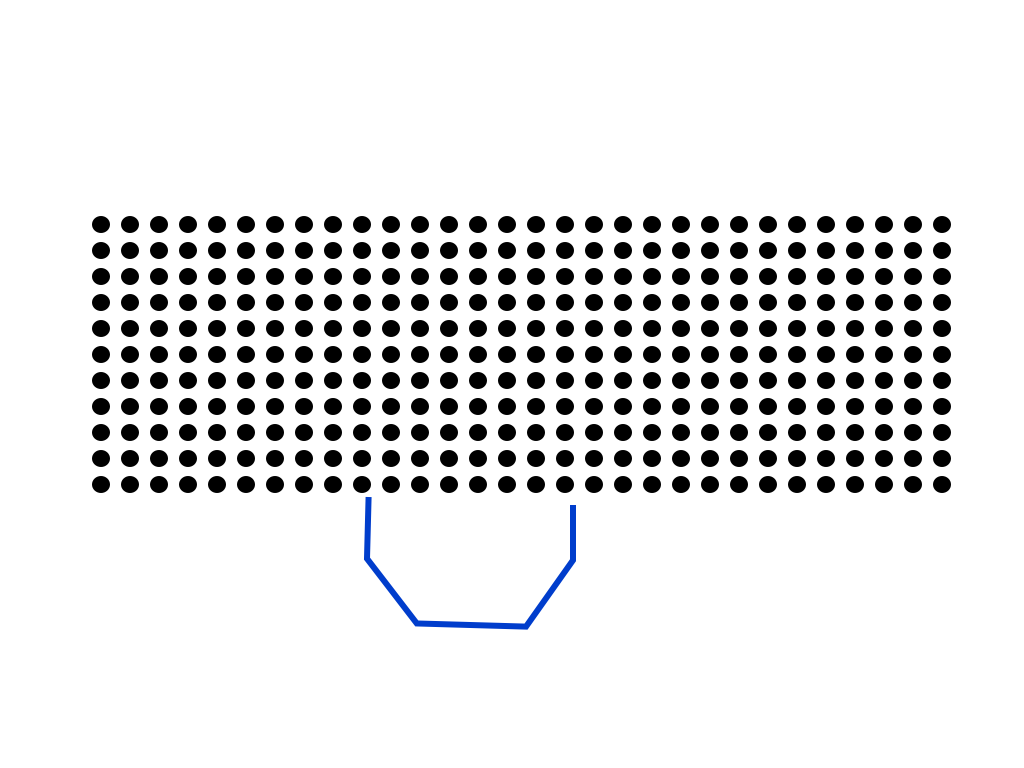} \includegraphics[trim=\setcoverfigleft px \setcoverfigbottom px \setcoverfigright px \setcoverfigtop px, clip=true, width=\setcoverfigscale]{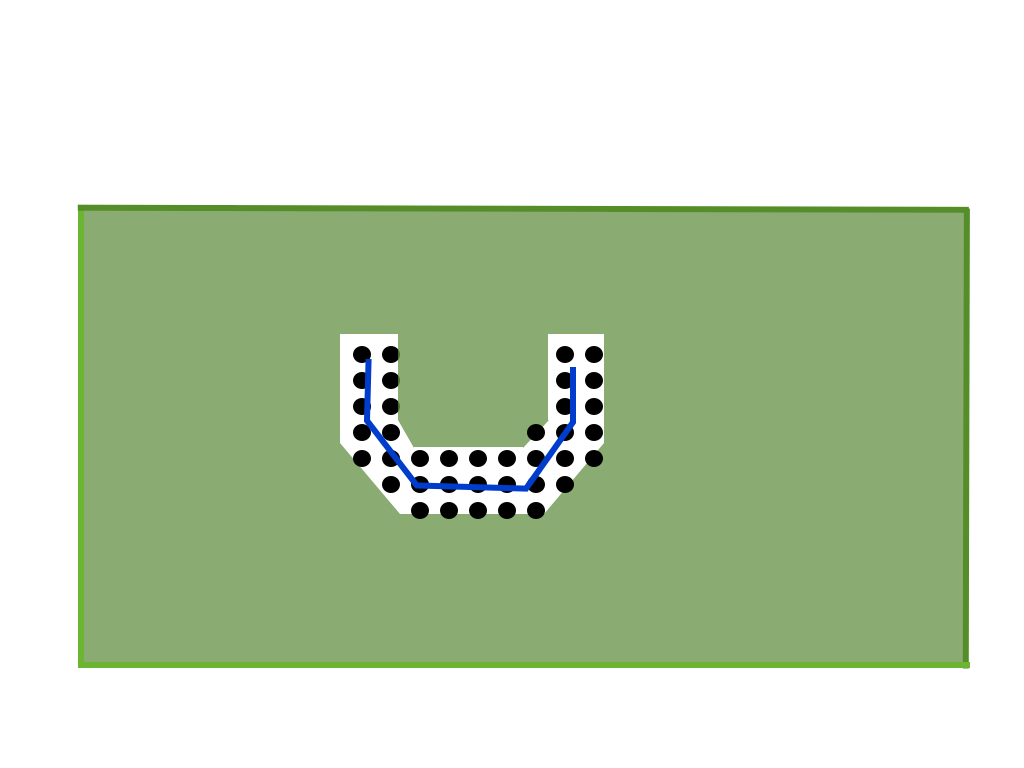} \includegraphics[trim=\setcoverfigleft px \setcoverfigbottom px \setcoverfigright px \setcoverfigtop px, clip=true, width=\setcoverfigscale]{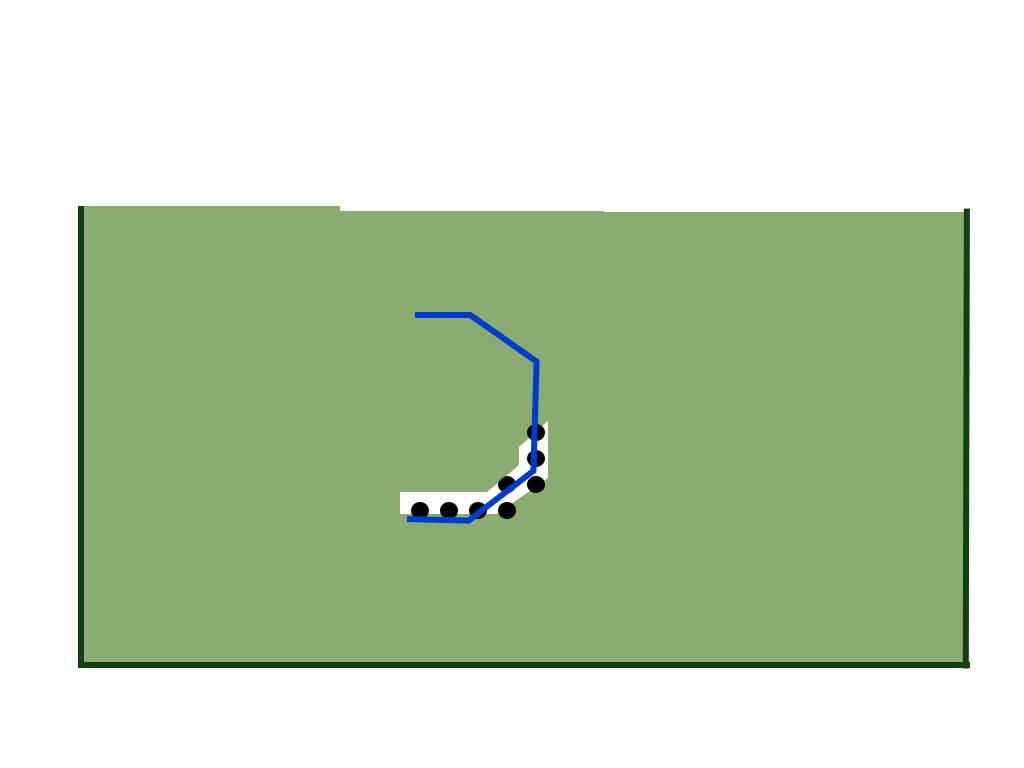}
  \caption{We can think of tactile localization as a problem of set cover, which is adaptive submodular~\cite{golovin_adaptive_2011}. Each observation amounts to covering (green area) the hypotheses (black dots) which do not agree. Our objective is to maximize our coverage, or rule out as many hypotheses as possible.}
  \label{fig:hand_setcover_ims}
\end{figure}

\subsection{Submodularity Assumptions for Touch Localization}
\label{sec_submod_assumptions}
Fitting into the framework of submodular maximization necessitates certain assumptions. First, all actions must be available at every step. Intuitively, this makes sense as a necessity for diminishing returns - if actions are generated at each step, then a new action may simply be better than anything so far. In some sense, non-greedy methods which generate actions based on the current belief state are optimizing both the utility of the current action, and the potential of actions that could be generated in the next step. Instead, we generate a large, fixed set of information gathering trajectories at the start.

Second, we cannot alter the underlying realization $\realization$, so actions are not allowed to change the state of the environment or objects. Therefore, we cannot intentionally reposition objects, or model object movement caused by contact.

When applied to object localization, this frameworks lends itself towards heavy objects that remain stationary when touched. For such problems, we believe having an efficient algorithm with guaranteed near-optimality outweighs these limitations. To alleviate some of these limitations, we hope to explore near-touch sensors in the future~\cite{kaijen_reactive_optical,pretouch_smith}.

\begin{figure}
  \includegraphics[trim=\actionfigleft px \actionfigbottom px \actionfigright px \actionfigtop px, clip=true, scale=\actionfig]{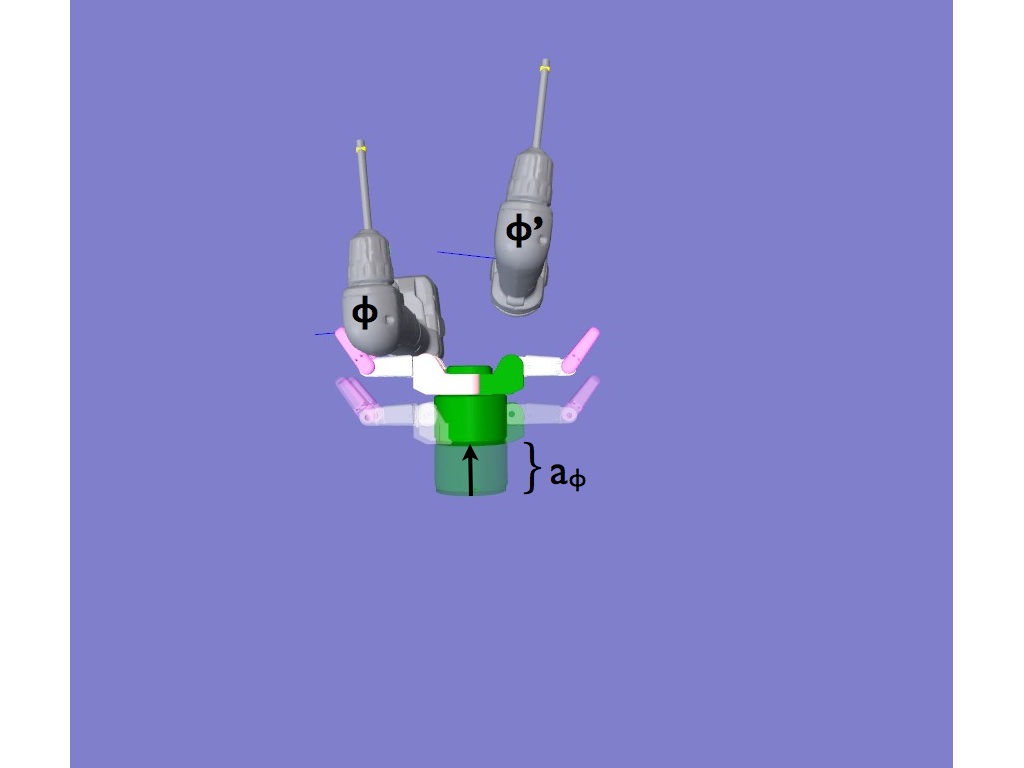} \includegraphics[trim=\actionfigleft px \actionfigbottom px \actionfigright px \actionfigtop px, clip=true, scale=\actionfig]{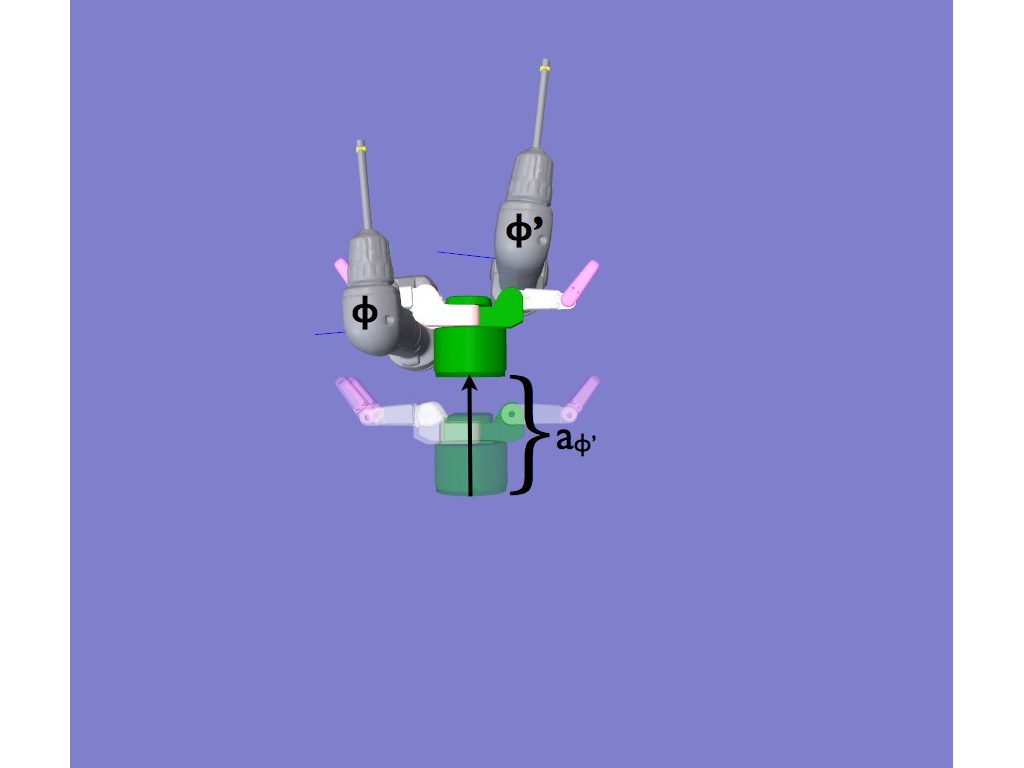} \includegraphics[trim=\actionfigleft px \actionfigbottom px \actionfigright px \actionfigtop px, clip=true, scale=\actionfig]{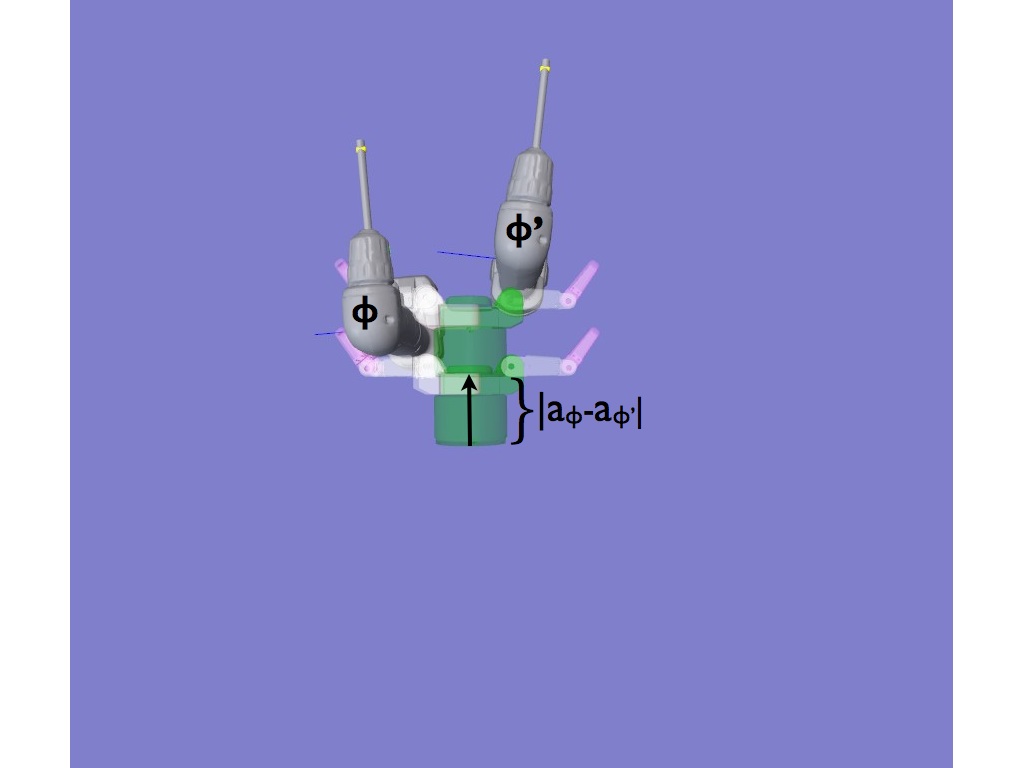}
  \caption{The observations for action $\actionitem$ and realizations $\realization$ and $\realization'$. Each observation $\actionitem_\realization$ and $\actionitem_\realization'$ corresponds to the time along the straight line trajectory when contact first occurs with the object. We use the difference of times $|\actionitem_\realization - \actionitem_{\realization'}|$ when measuring how far apart observations are.}
  \label{fig:handstops}
\end{figure}

\subsection{Information Gain}
\label{sec_infogain}
Information gain has been applied to touch localization before~\cite{hsiao_2008_robust_belief, hebert_next_best_touch_2012}. In contrast to these, we utilize a large fixed set of actions, enforce the assumptions from~\sref{sec_submod_assumptions}, and use a particle-based model (as opposed to a histogram).

Following Krause and Guestrin~\cite{krause_2005_submodular_entropy}, we define the information gain as the reduction in Shannon entropy. Let $\partialgivenactionsrandom$ be the random variable over $\partialgivenactions$. Then we have
\begin{align*}
    \infogain(\randrealization; \partialgivenactionsrandom) = H(\randrealization) - H(\randrealization | \partialgivenactionsrandom)
\end{align*}

As they show, this function is monotone submodular if the observations $\partialrealizationrandom_{\allactionset}$ are conditionally independent given the state $\realization$. Thus, if we are evaluating this offline, we would be near-optimal compared to the optimal offline solution. However, this can actually perform exponentially worse than the online solution~\cite{hollinger_isrr_2011}. Therefore, we greedily select actions based on the marginal utility of a single action:
\begin{align*}
    \Delta_{IG}(a) = H(\randrealization) - \mathbb{E}_o\left[H(\randrealization | o) \right]
\end{align*}

We also need to define the probability of an observation. We consider a ``blurred'' measurement model where the probability of stopping at $o$ conditioned on realization $\realization$ is weighted based on the time difference between $o$ and $a_\realization$ (the time of contact had $\realization$ been the true state), with $\sigma$ modelling the measurement noise:
\begin{align*}
    p(a_{\randrealization} = o | \realization ) &\propto \text{exp}\left({-\frac{|o - a_\realization|}{2 \sigma^2}}\right)
\end{align*}

We could consider evaluating $H(\randrealization)$ with a discrete entropy calculation, where each $\realization \in \randrealization$ is treated as an individual item. However, our particle set $\randrealization$ is modeling an underlying continuous distribution, and we would like to capture that. Thus, we instead fit a Gaussian to the current set $\randrealization$ and evaluate the entropy of that distribution. Let $\Sigma_o$ be the covariance over the weighted set of hypotheses, and $N$ the number of parameters (typically x, y, z, $\theta$). We use the approximated entropy:
\begin{align*}
    H(\randrealization | o) &\approx \frac{1}{2} \ln((2\pi e)^N |\Sigma_o|)
\end{align*}

After performing the selected action, we update the belief by reweighting hypotheses as described above. We repeat the action selection process, setting $\randrealization$ to be the updated distribution, until we reach some desired entropy.

\subsection{Hypothesis Pruning}
\label{sec_hypothprune}
Intuitively, information gain is attempting to reduce uncertainty by removing probability mass. Here, we formulate a method with this underlying idea that also satisfies properties of adaptive submodularity and strong adaptive monotonicity. We refer to this as Hypothesis Pruning, since the idea is to prune away hypotheses that do not agree with observations. Golovin et al. describe the connection between this type of objective and adaptive submodular set cover~\cite{golovin_adaptive_2011}. Our formulation is similar - see \figref{fig:hand_setcover_ims} for a visualization.

We note that adaptive submodular functions~\cite{golovin_adaptive_2011} cannot handle noise - they require any realization $\realization$ be consistent with only one observation per action. However, we would like to model sensor noise. A standard method for alleviating this is to construct a non-noisy problem by generating a realization for every possible noisy observations of every $\realization \in \randrealization$. Let $\noisingfuncactionset(\realization) = \{\noisyrealization_{1}, \noisyrealization_{2},\dots\}$ be the function that generates the new realizations $\noisyrealization_{i}$ for actions $\actionset$. Underlying our formulation, we consider constructing this non-noisy problem. Luckily, we can compute our objective function on the original $\randrealization$, and do not need to explicitly perform this construction. We present this more efficient computation below, and show how to construct the equivalent non-noisy problem in~\sref{sec_appendix}.

As before, we consider a ``blurred'' measurement model through two different observation models. In the first, we define a cutoff threshold $d_T$. If a hypothesis is within the threshold, we keep it entirely. Otherwise, it is removed. We call this metric Hypothesis Pruning (HP). In the second, we downweight hypotheses with a (non-normalized) Gaussian, effectively removing a portion of the hypothesis. We call this metric Weighted Hypothesis Pruning (WHP). The weighting functions are:
\begin{align*}
    \weighting^{HP}_o(\actionitem_\realization) &= 
    \begin{cases}
    1  & \text{if } \  |o-a_\realization| \leq d_T \\
    0       & \text{else} 
    \end{cases} \\
    \weighting^{WHP}_o(\actionitem_\realization) &= \text{exp}\left({-\frac{ |o-a_\realization|^2}{2 \sigma^2}}\right)
\end{align*}

For a partial realization $\partialrealization$, we take the product of weights:
\begin{align*}
    p_\partialrealization (\realization) &= \left(\prod_{\{\actionitem, \observationitem\} \in \partialrealization} \weighting_\observationitem (\actionitem_\realization) \right) p(\realization)
\end{align*}

Note that this can never increase the probability - for any $\partialrealization$, $p_\partialrealization (\realization) \leq p (\realization)$.

Define $\probmassall_{\partialrealization}$ as the non-normalized probability remaining after partial realization $\partialrealization$, and $\probmass_{\partialrealization,\actionitem,\observationitem}$ as the probability remaining after an additional action $\actionitem$ and observation $\observationitem$:
\begin{align*}
    \probmassall_{\partialrealization} &= \sum_{\realization' \in \randrealization} p_{\partialrealization} (\realization')\\
    \probmass_{\partialrealization,\actionitem,\observationitem} &= \sum_{\realization' \in \randrealization} p_{\partialrealization} (\realization') \weighting_{\observationitem} (\actionitem_{\realization'})
\end{align*}

We can now define our objective function for any partial realization $\partialrealization$, corresponding to removing probability mass:
\begin{align*}
    f(\partialrealization) &= 1 - \probmassall_{\partialrealization}
\end{align*}

In practice, we need to discretize the infinite observation set $\allobservationset$. Formally, we require that an equal number of observations per $\realization$ are considered. That is, for any action $\actionitem$ and any realizations $\realization_i, \realization_j$, $|\noisingfunc_\actionitem(\realization_i)| = |\noisingfunc_\actionitem(\realization_j)|$\footnote{Note that we must be consistent between contact and no-contact observations. That is, if we believe action $\actionitem$ will contact $\realization_i$ but not $\realization_j$, it still must be that $|\noisingfunc_\actionitem(\realization_i)| = |\noisingfunc_\actionitem(\realization_j)|$. Thus, we also have multiple noisy no-contact observations. See \sref{sec_appendix} for details.}. In practice, we sample observations uniformly along the trajectory to approximately achieve this effect.

To calculate the expected marginal gain, we also need to define the probability of receiving an observation. We present it here, and show the derivation in \sref{sec_appendix}. Intuitively, this will be proportional to how much probability mass agrees with the observation. Let $\allobservationsetaction$ be the set of all possible observations for action $\actionitem$:
\begin{align*}
    p(a_{\randrealization} = o | \partialrealization) = \frac{\probmass_{\partialrealization, \actionitem, \observationitem}}{ \sum_{\observationitem' \in \allobservationsetaction} \probmass_{\partialrealization,\actionitem, \observationitem'}}
\end{align*}

We note the marginal utility is the additional probability mass removed. For an action $\actionitem$ and observation $\observationitem$ this is $f_{\partialrealization,\actionitem,\observationitem} = \probmassall_{\partialrealization} - \probmass_{\partialrealization,\actionitem,\observationitem}$. The expected marginal gain is:
\begin{align*}
    \Delta(a|\partialrealization) &=  \mathbb{E}_\observationitem \left[f_{\partialrealization,\actionitem,\observationitem}\right]\\
    &= \sum_{\observationitem \in \allobservationsetaction} \frac{\probmasspartialaction}{ \sum_{\observationitem' \in \allobservationsetaction} \probmasspartialactionprime} \left[ \probmasspartial - \probmasspartialaction \right]
\end{align*}

The greedy policy $\greedypolicy$ maximizes the expected probability mass removed per unit cost, or $\frac{\Delta(a|\partialrealization)}{c(a)}$. After selecting an action and receiving an observation, hypotheses are removed or downweighted as described above, and action selection is iterated. We now present our main guarantee:

\begin{theorem} \label{theorem_hp_as}
  Let our utility function be $f$ as defined above, utilizing either weighting function $\weighting^{HP}$ or $\weighting^{WHP}$. Define a quality threshold $Q$ such that $\min_{\realization, \noisingfuncallaction(\realization)} f(\{\allactionset, \allactionset_{\noisingfuncallaction(\realization)} \}) = Q$.\footnote{If we have a target uncertainty $Q'$, we can define a truncated function $g(\partialrealization) = \min(f(\partialrealization), Q')$ to decrease our quality threshold. Truncation preserves adaptive submodularity~\cite{golovin_adaptive_2011}, so $g$ is adaptive submodular if $f$ is.} Let $\eta$ be any value such that $f(\partialrealization) > Q - \eta$ implies $f(\partialrealization) \geq Q$ for all $\partialrealization$. Let $\optpolicy_{avg}$ and $\optpolicy_{wc}$ be the optimal policies minimizing the expected and worst-case number of items selected, respectively. The greedy policy $\greedypolicy$ satisfies:
\begin{align*}
  c_{avg}(\greedypolicy) &\leq c_{avg}(\optpolicy_{avg}) \left( \ln \left( \frac{Q}{\eta} \right) + 1 \right)\\
  c_{wc}(\greedypolicy) &\leq c_{wc}(\optpolicy_{wc}) \left( \ln \left( \frac{Q}{\delta \eta} \right) + 1 \right)
\end{align*}
With $\delta$ a constant based on the underlying non-noisy problem (see \sref{sec_appendix}).
\end{theorem}

\begin{proof}
  In order to prove Theorem~\ref{theorem_hp_as}, we show that our objective $f$ is adaptive submodular, strongly adaptive monotone, and self-certifying in an equivalent non-noisy problem. We show this in our \sref{sec_appendix}.
\end{proof}

If we utilize $\weighting^{HP}$ as our weighting function, we can use $\eta = \min_{\realization} p(\realization)$. If we utilize $\weighting^{WHP}$, $\eta$ is related to how we discretize observations.

In addition to being within a logarithmic factor of optimal, we utilize a lazy-greedy algorithm which does not reevaluate all actions at every step, speeding up computation~\cite{minoux_lazy, golovin_adaptive_2011}.

\section{Experiments}
\label{sec_experiments}
We implement a greedy action selection scheme with each of the metrics described above (IG, HP, WHP). In addition, we compare against two other schemes - random action selection, and a simple human-designed scheme which approaches the object orthogonally along the X, Y and Z axes. Each object pose $\realization$ consist of a 4-tuple $(x, y, z, \theta) \in \mathbb{R}^4$, where $(x, y, z)$ are the coordinates of the object's center, and $\theta$ is the rotation about the $z$ axis.  

We implement our algorithms using a 7-dof Barret arm with an attached 4-dof Barret hand. We localize two objects: a drill upright on a table, and a door. We define an initial sensed location $X_{s} \in \mathbb{R}^4$. To generate the initial $\randrealization$, we sample a Gaussian distribution $N(\mu, \Sigma)$, where $\mu = X_{s}$, and $\Sigma$ is the prior covariance of the sensor's noise. For simulation experiments, we also define the ground truth pose $X_{t} \in \mathbb{R}^4$.

For efficiency purposes, we also use a fixed number of particles $| \randrealization |$ at all steps, and resample after each selection, adding small noise to the resampled set of particles.

\subsection{Action Generation}
We generate linear motions of the end effector, consisting of a starting pose and a movement vector. Each action starts outside of all hypotheses, and moves as far as necessary to contact every hypotheses along the path. Note that using straight-line trajectories is not a requirement for our algorithm. We generate actions via three main techniques.

\subsubsection{Sphere Sampling} \label{sphere_sample_section}
Starting positions are generated by sampling a sphere around the sensed position $X_{s}$. For each starting position, the end-effector is oriented to face the object, and the movement direction set to $X_{s}$. A random rotation is applied about the movement direction, and a random translation along the plane orthogonal to the movement.

\subsubsection{Normal Sampling}\label{normal_sample_section}
These actions are intended to have the hand's fingers contact the object orthogonally. First, we uniformly sample random contacts from the surface of the object. Then, for each fingertip, we align its pre-defined contact point and normal with our random sample, and randomly rotate the hand about the contact normal. We then set the movement direction as the contact normal.

\subsubsection{Table Contacting}\label{plane_sample_section}
We generate random start points around the sensed position  $X_{s}$, and orient the end effector in the $-z$ direction. These are intended to contact the table and reduce uncertainty in $z$.

\subsection{Simulation Experiments Setup} \label{sec:random_seed_experiments}
We simulate an initial sensor error as $X_{t} - X_{s} = (0.015, -0.015, -0.01, 0.05)$ (in meters and radians). Our initial random realization $\randrealization$ is sampled from $N(\mu, \Sigma)$ with $\mu = X_{s}$, and $\Sigma$ a diagonal matrix with $\Sigma_{xx} = 0.03$, $\Sigma_{yy} = 0.03$, $\Sigma_{zz} = 0.03$, $\Sigma_{\theta\theta} = 0.1$. We fix $| \randrealization | = 1500$ hypotheses.

We then generate an identical action set $\allactionset$ for each metric. The set consists of the 3 human designed trajectories, 30 sphere sampled trajectories (\sref{sphere_sample_section}), 160 normal trajectories (\sref{normal_sample_section}), and 10 table contact trajectories (\sref{plane_sample_section}), giving $| \allactionset | =  203$.

We run 10 experiments using a different random seed for each, generating a different set $\allactionset$ and $\randrealization$, but ensuring each metric has the same $\allactionset$ and initial $\randrealization$ for a random seed. Each method chooses a sequence of five actions, except the human designed sequence which consists of only three actions.

\subsection{Simulation Experiments Results}

\begin{figure} 
  \begin{subfigure}[t]{0.49\columnwidth}
    \centering
    \includegraphics[trim=0 0 8 0, clip=true, width=1.0\columnwidth]{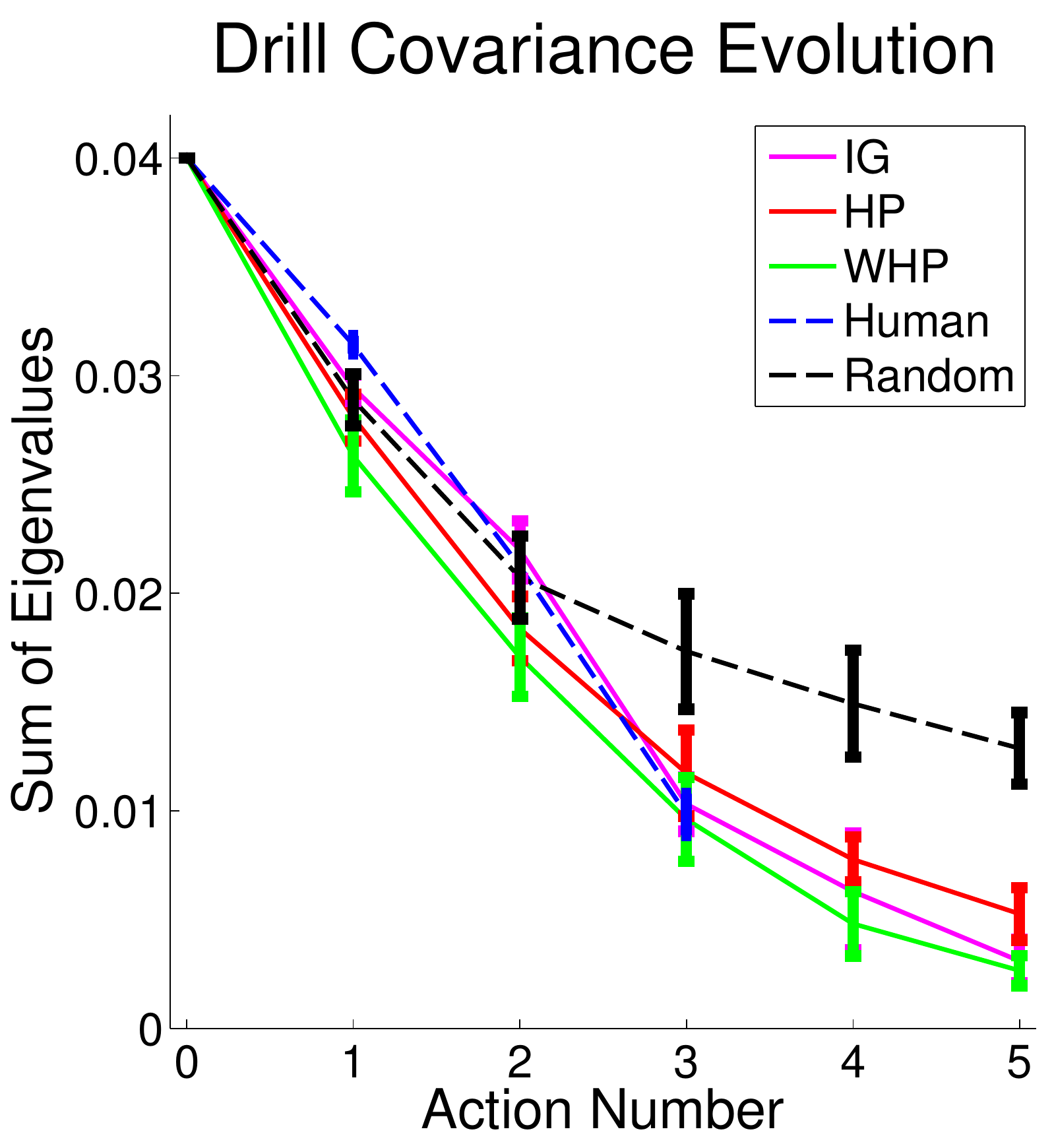}
  \end{subfigure}
  \begin{subfigure}[t]{0.49\columnwidth}
    \centering
    \includegraphics[trim=0 0 8 0, clip=true, width=1.0\columnwidth]{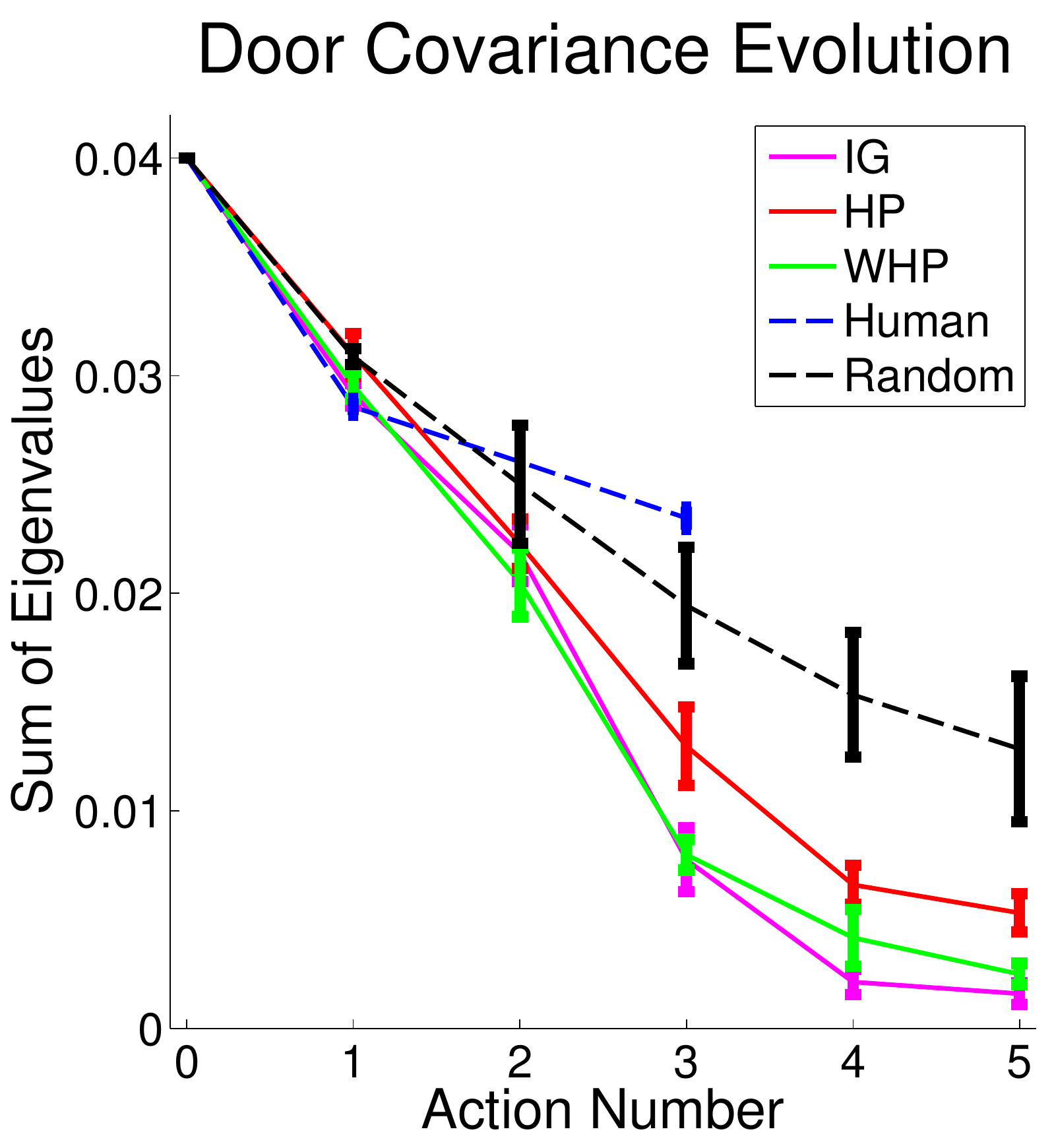}
  \end{subfigure}
  \caption{Uncertainty after each action for drill and door experiments. The bars show the mean and $95\%$ CI of the sum of eigenvalues of the covariance matrix over experiments described in in \sref{sec:random_seed_experiments}.} 
  \label{fig:covariance}
\end{figure}

\begin{table}
  \centering
  \begin{tabular}{rccc}
\toprule
& IG & HP & WHP \\
\midrule
    \textbf{Time (s)} & $47.171 \pm 0.25$ & $\bf{8.41 \pm 0.58}$ & $25.70 \pm 0.29$ \\
    \bottomrule
  \end{tabular} 
    \caption{Time to select one action for each metric, average and $95\%$ CI over drill experiments described in \sref{sec:random_seed_experiments}} 
    \label{tab:times}
\end{table}

We analyze the uncertainty reduction of each metric as the sum of eigenvalues of the covariance matrix, as in \figref{fig:covariance}. All metrics were able to reduce the uncertainty significantly -- confirming our speculation in \sref{sec_related_works} that even random actions reduce uncertainty. However, as the uncertainty is reduced, the importance of action selection increases, as evidenced by the relatively poor performance of random selection for the later actions. Additionally, we see the human designed trajectories are effective for the drill, but perform poorly on the door. Unlike the drill, the door is not radially symmetric, and its flat surface and protruding handle offer geometric landmarks that our action selection metrics can exploit, making action selection more useful.

For one drill experiment, we also display the hypothesis set after each action in \figref{fig:drill_quivers}, and the first 3 actions selected in \tabref{table:eachaction_yellow}. Note that the actions selected are very different, while performance appears similar.

\noindent \textbf{Observation 1:} Information Gain (IG), Hypothesis Pruning (HP), and Weighted Hypothesis Pruning (WHP) all perform similarly well. On the one hand, you might expect IG to perform poorly with adaptive greedy selection, as we don't have any guarantees. On the other, Shannon entropy has many properties that make it a good measure of uncertainty. \figref{fig:covariance} displays the covariance of all particles, which is the criterion IG directly optimizes. Note that, surprisingly, HP and WHP perform comparably despite not directly optimizing this measure.

\noindent \textbf{Observation 2:} The HP and WHP perform faster than IG. This is due to their inherent simplicity and the more efficient lazy-greedy algorithm~\cite{minoux_lazy,golovin_adaptive_2011}. See \tabref{tab:times} for times. Additionally, we lose little performance with large computational gains with the non-weighted observation model of HP.

\begin{figure}[t]
  \centering
  \includegraphics[trim=0 0 0 0, clip=true, width=0.49\columnwidth]{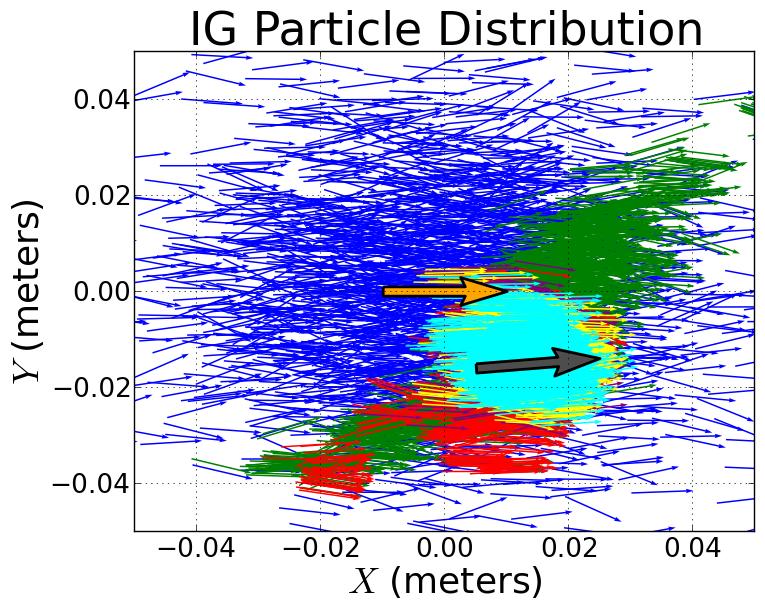} \hspace{-1mm} 
  \includegraphics[trim=0 0 0 0, clip=true, width=0.49\columnwidth]{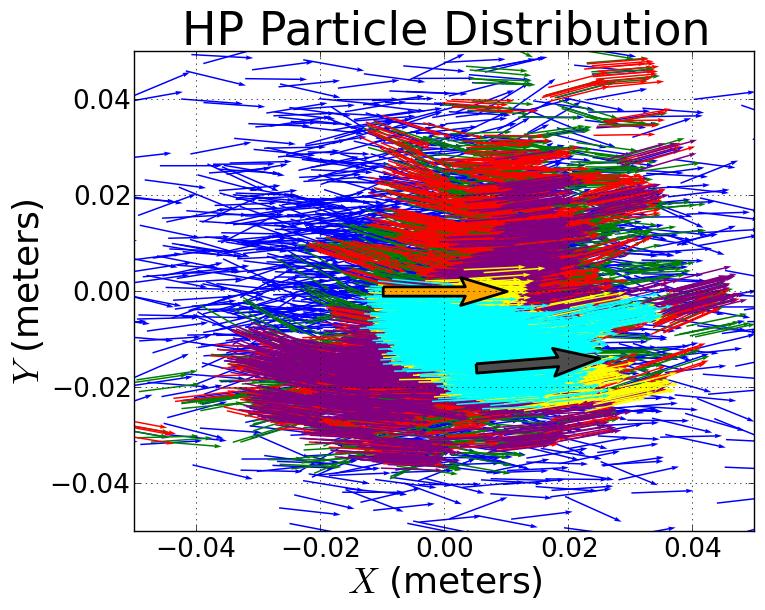}  \hspace{-1mm}
  \includegraphics[trim=0 0 0 0, clip=true, width=0.49\columnwidth]{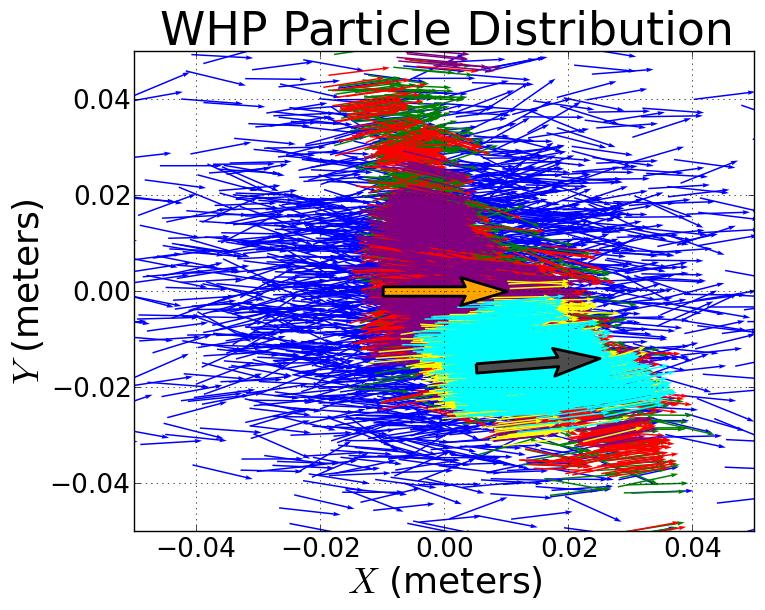} \hspace{-1mm}
  \includegraphics[trim=0 0 0 0, clip=true, width=0.49\columnwidth]{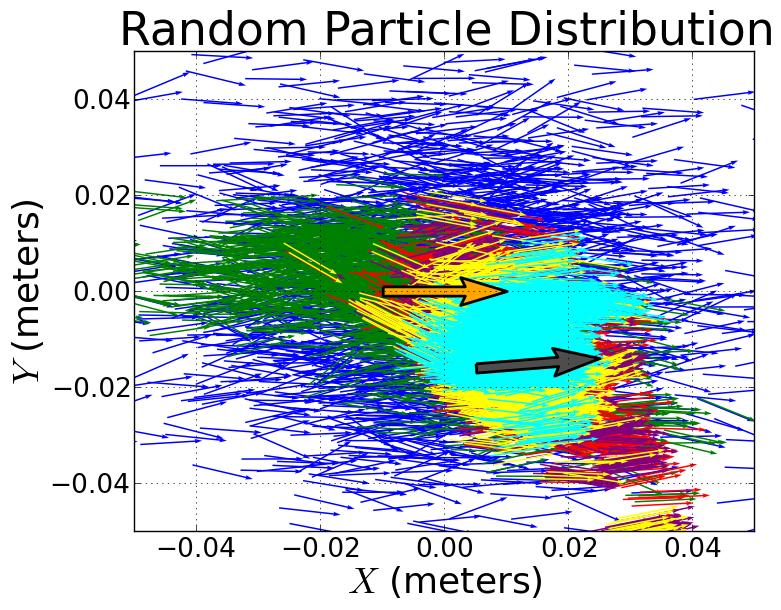}  \hspace{-1mm}
  \includegraphics[trim=0 0 0 0, clip=true, width=0.49\columnwidth]{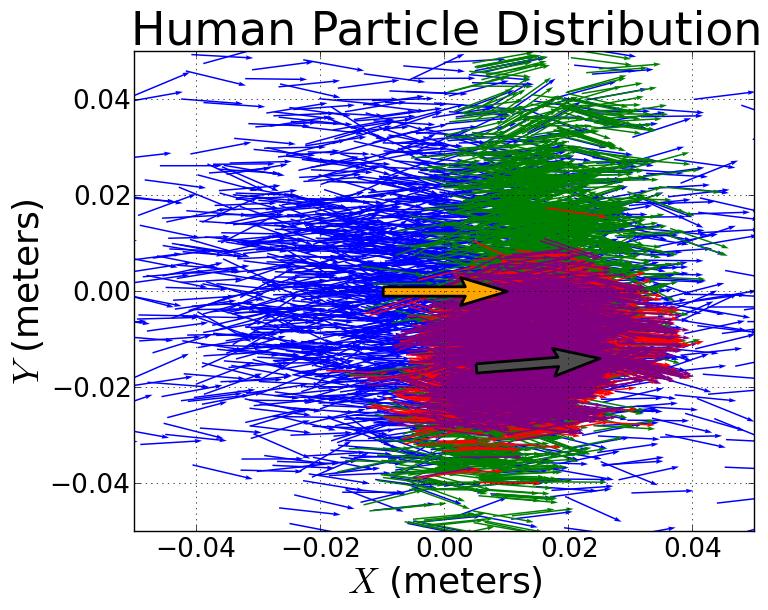}
	\includegraphics[trim=10 30 50 20, clip=true, width=0.49\columnwidth]{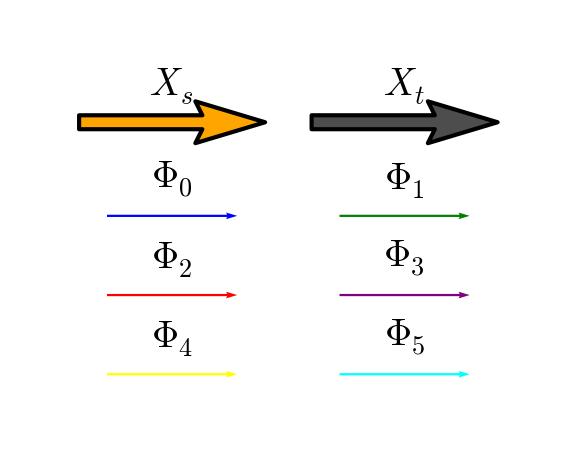}
  \caption{The particle sets $\randrealization$ from a single drill experiment after each update. Plotted positions corresponds to the $x,y$ parameter of each $\realization \in \randrealization$, rotated  by $\theta$. $X_s$ is the sensed position, $X_t$ the true position, and $\randrealization_i$ the particles after update $i$. Arrow lengths are approximately the length of the drill base. The initial $\randrealization$ was sampled from a normal distribution with $\sigma_{x}=0.02,$ $\sigma_{y}=0.02$, $\sigma_{z}=0.02$, $\sigma_\theta=0.2$.}
  \label{fig:drill_quivers}
\end{figure}

\begin{table}[t]\setlength{\tabcolsep}{1pt}
\centering
\begin{tabular}{rccc}
    & \large Action 1 & \large Action 2 & \large Action 3  \\
  \begin{sideways} \hspace{25pt} \Large IG \end{sideways} & \includegraphics[trim=\scalestopsyellowcolleft px \scalestopsyellowcolbottom px \scalestopsyellowcolright px \scalestopsyellowcoltop px, clip=true, scale=\scalestopsyellowcol]{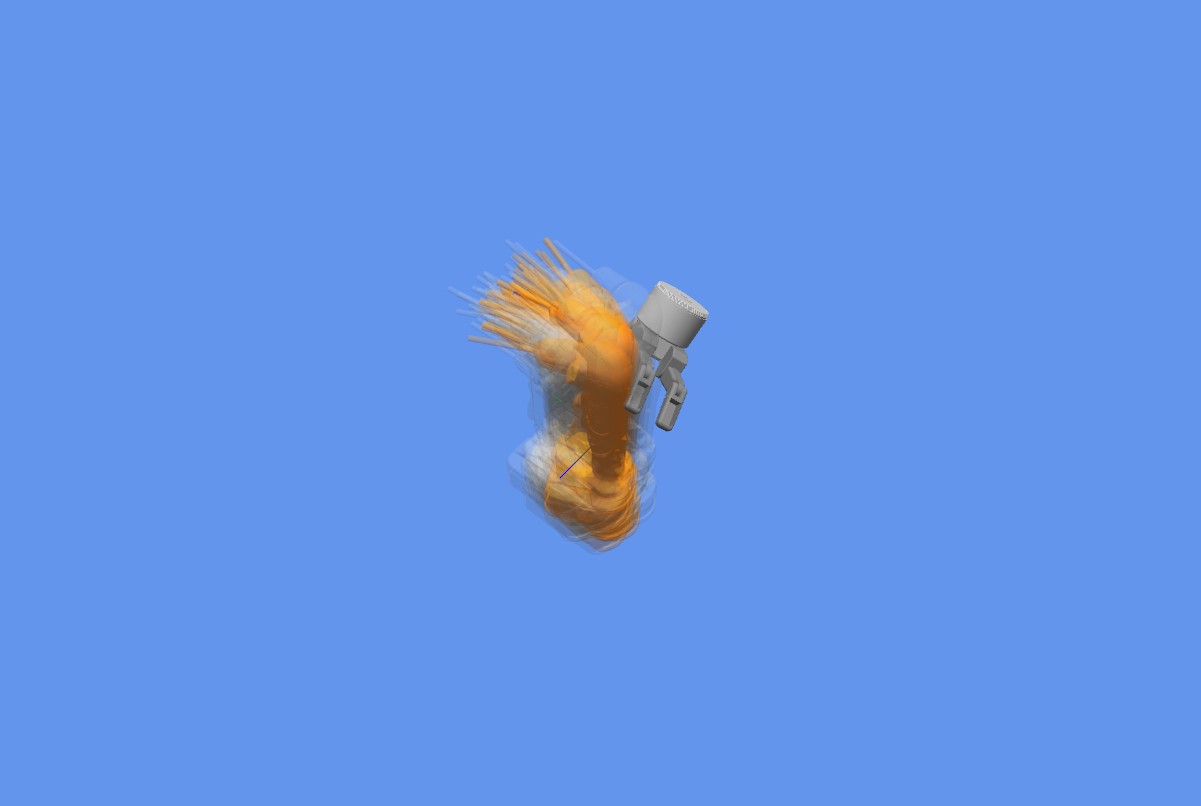} &  \includegraphics[trim=\scalestopsyellowcolleft px \scalestopsyellowcolbottom px \scalestopsyellowcolright px \scalestopsyellowcoltop px, clip=true, scale=\scalestopsyellowcol]{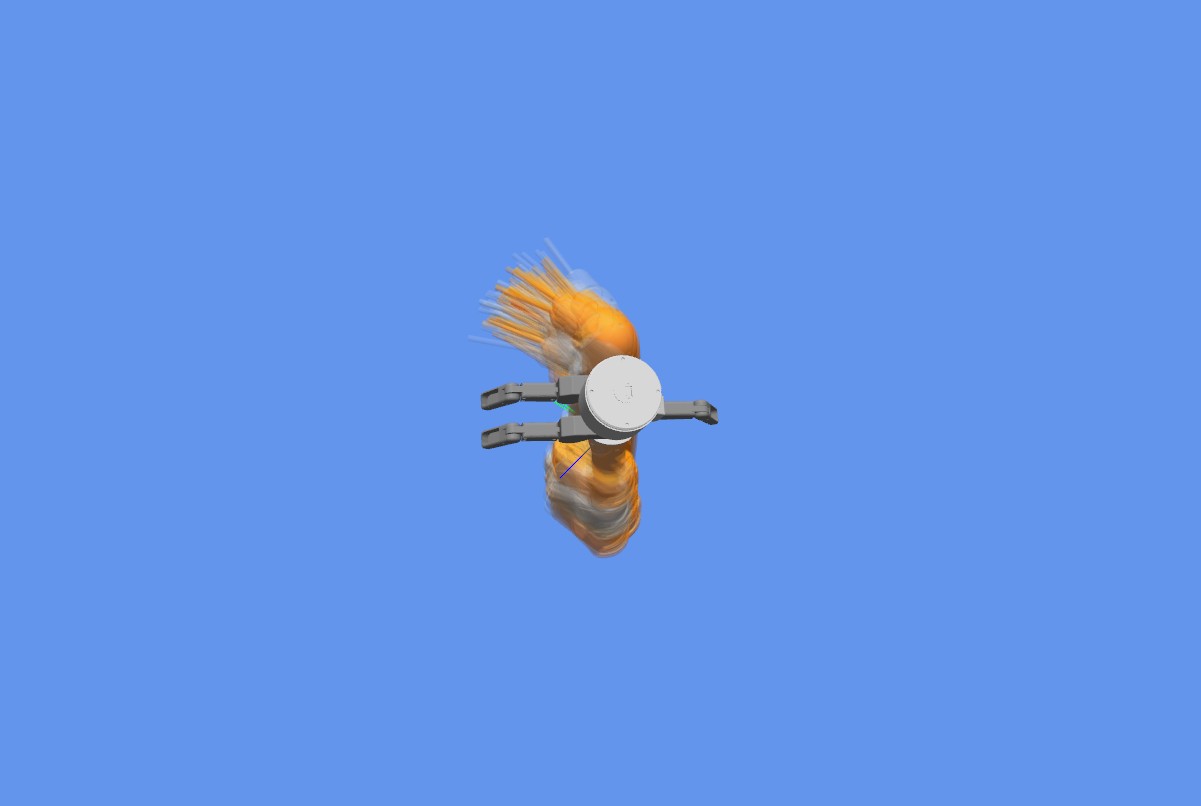} & \includegraphics[trim=\scalestopsyellowcolleft px \scalestopsyellowcolbottom px \scalestopsyellowcolright px \scalestopsyellowcoltop px, clip=true, scale=\scalestopsyellowcol]{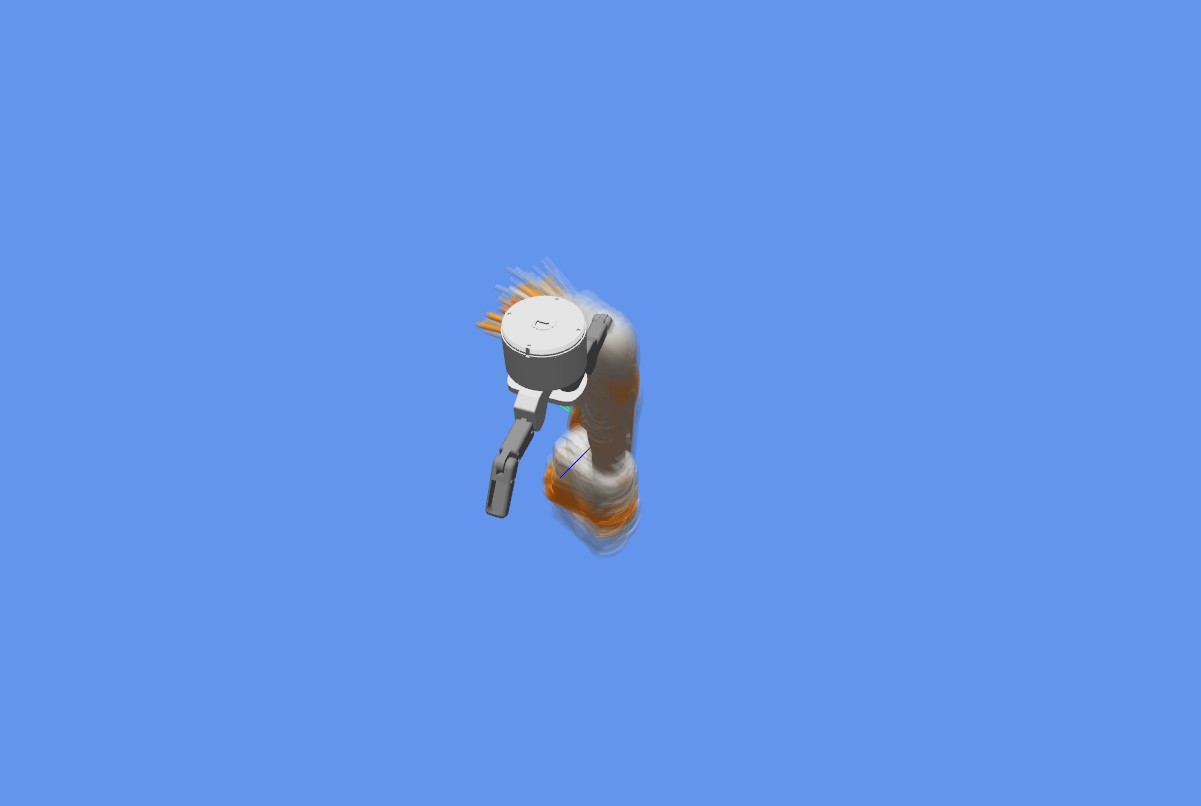}  \vspace{-2pt} \\
 \noalign{\hrule height 2pt} \vspace{-8pt} \\ 
  \begin{sideways} \hspace{25pt} \Large HP \end{sideways} & \includegraphics[trim=\scalestopsyellowcolleft px \scalestopsyellowcolbottom px \scalestopsyellowcolright px \scalestopsyellowcoltop px, clip=true, scale=\scalestopsyellowcol]{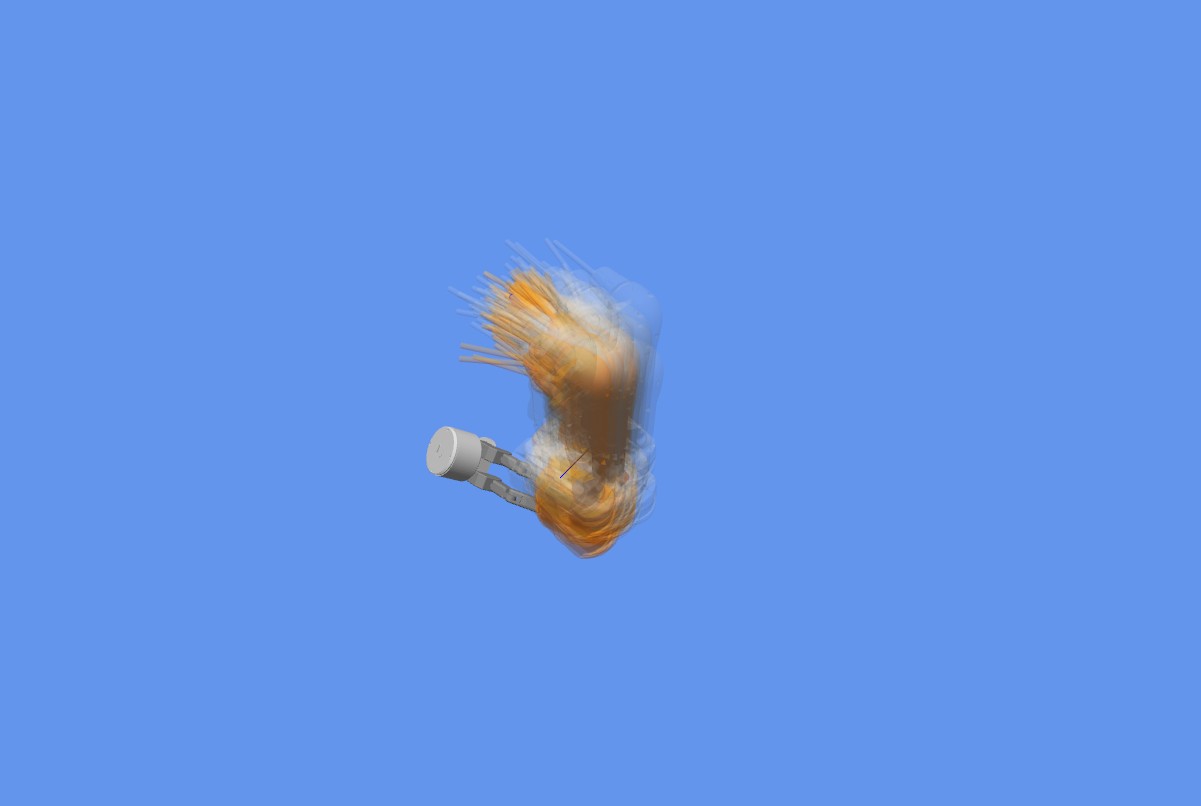} &  \includegraphics[trim=\scalestopsyellowcolleft px \scalestopsyellowcolbottom px \scalestopsyellowcolright px \scalestopsyellowcoltop px, clip=true, scale=\scalestopsyellowcol]{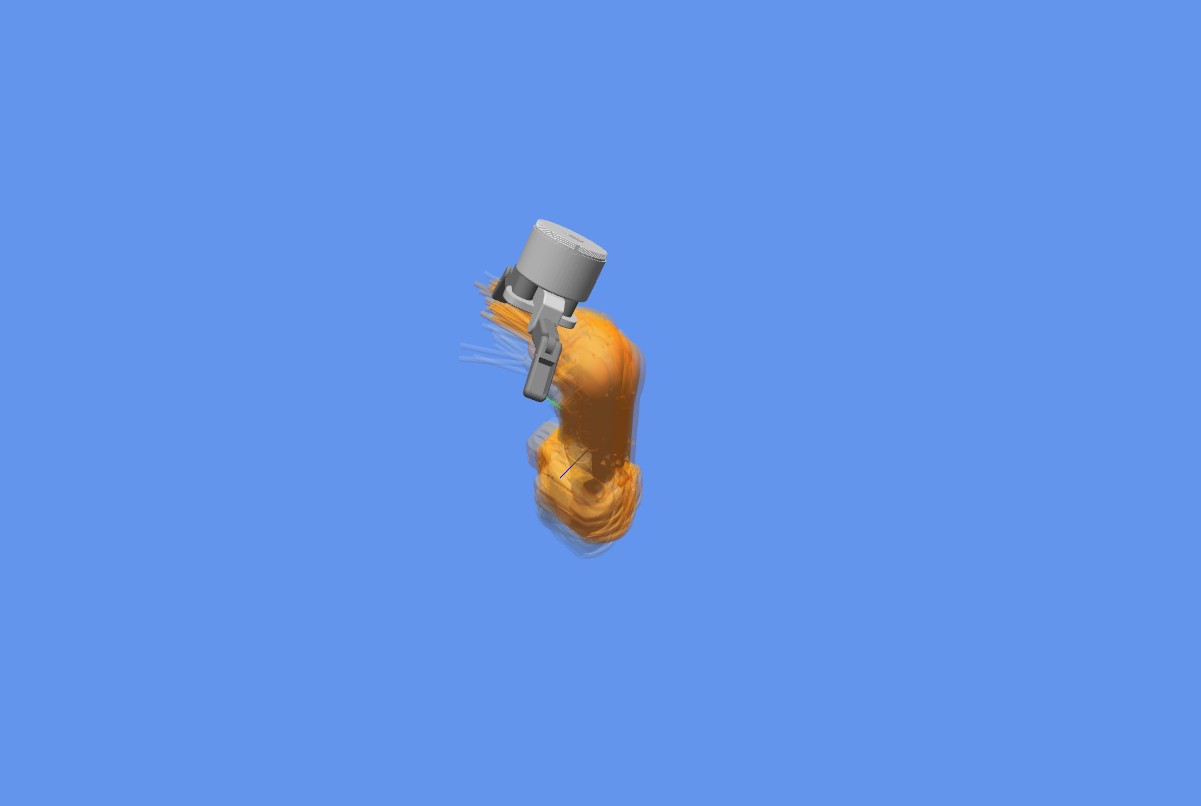} & \includegraphics[trim=\scalestopsyellowcolleft px \scalestopsyellowcolbottom px \scalestopsyellowcolright px \scalestopsyellowcoltop px, clip=true, scale=\scalestopsyellowcol]{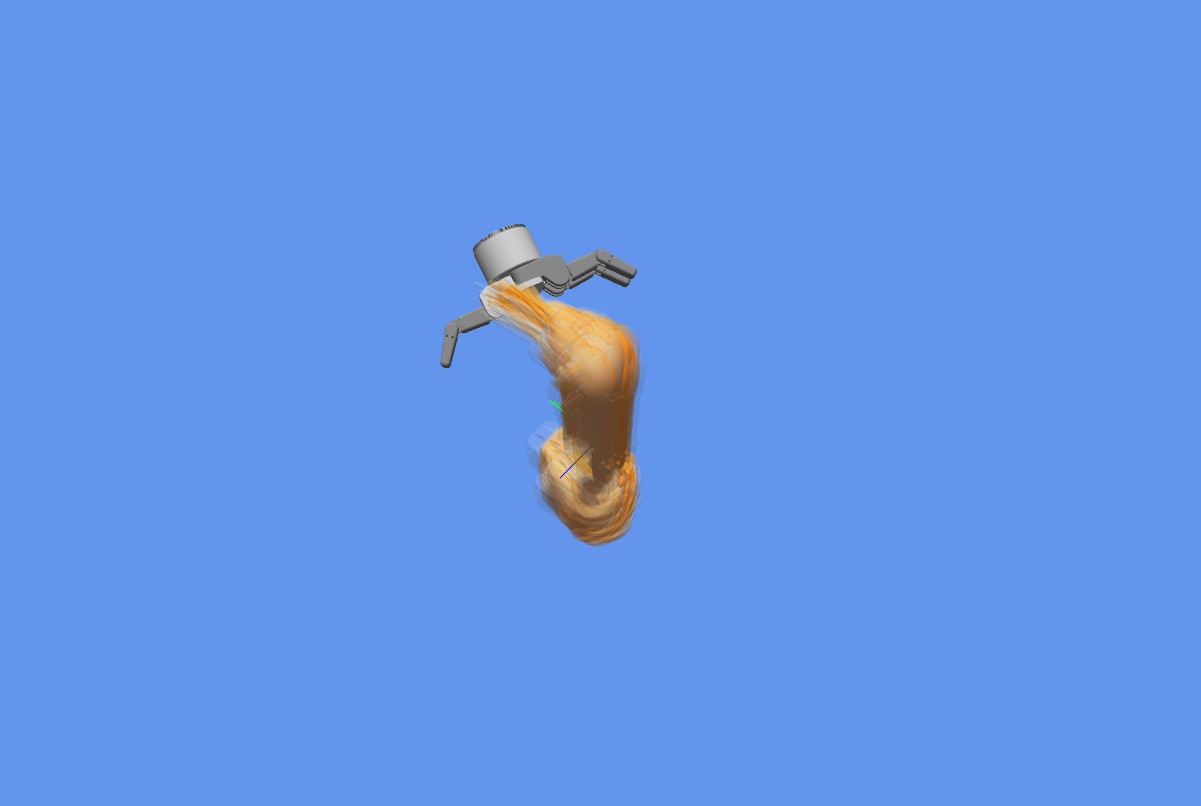}  \vspace{-2pt} \\
  \noalign{\hrule height 2pt} \vspace{-8pt} \\ 
  \begin{sideways} \hspace{20pt} \Large WHP \end{sideways} & \includegraphics[trim=\scalestopsyellowcolleft px \scalestopsyellowcolbottom px \scalestopsyellowcolright px \scalestopsyellowcoltop px, clip=true, scale=\scalestopsyellowcol]{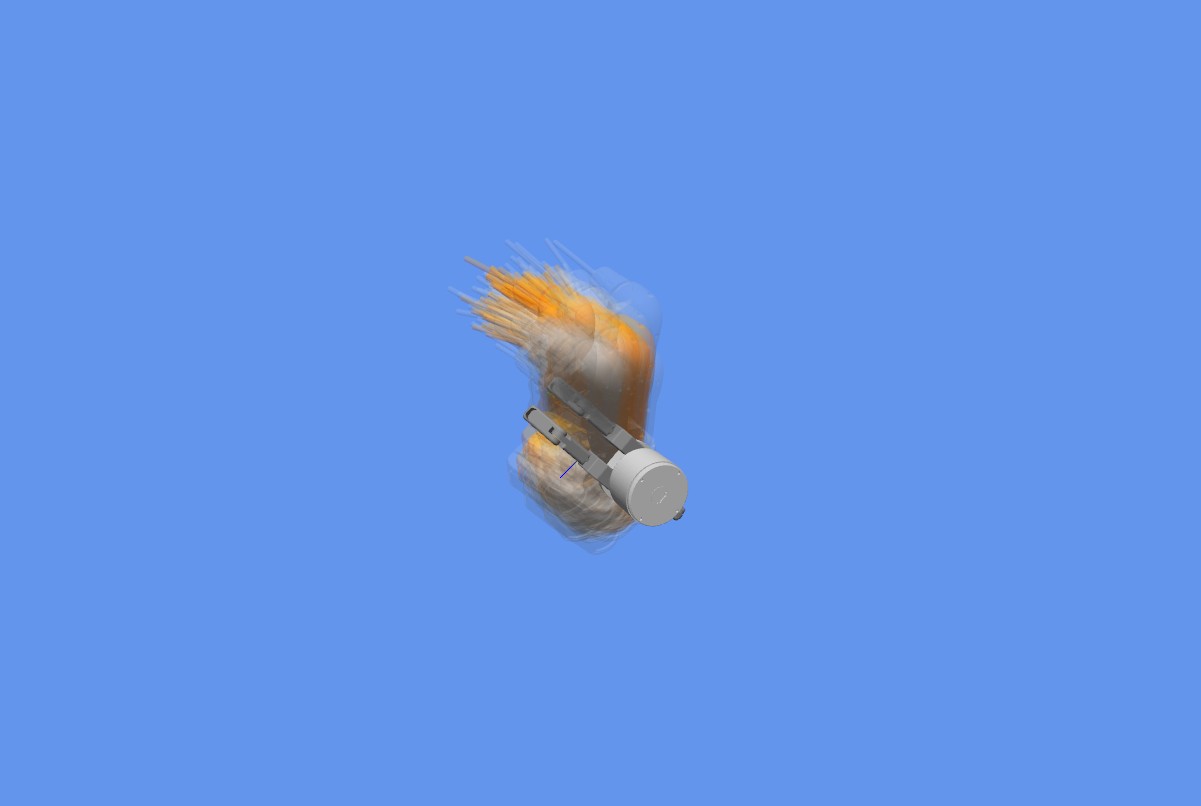} &  \includegraphics[trim=\scalestopsyellowcolleft px \scalestopsyellowcolbottom px \scalestopsyellowcolright px \scalestopsyellowcoltop px, clip=true, scale=\scalestopsyellowcol]{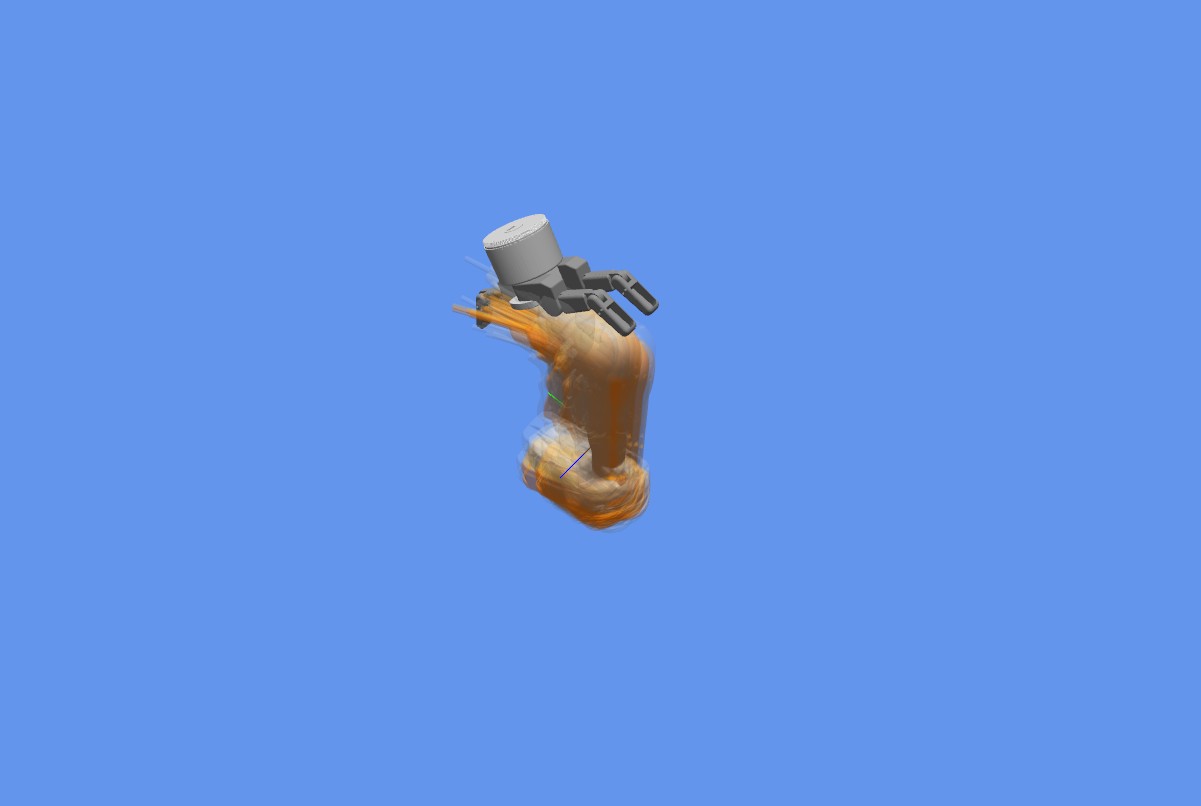} & \includegraphics[trim=\scalestopsyellowcolleft px \scalestopsyellowcolbottom px \scalestopsyellowcolright px \scalestopsyellowcoltop px, clip=true, scale=\scalestopsyellowcol]{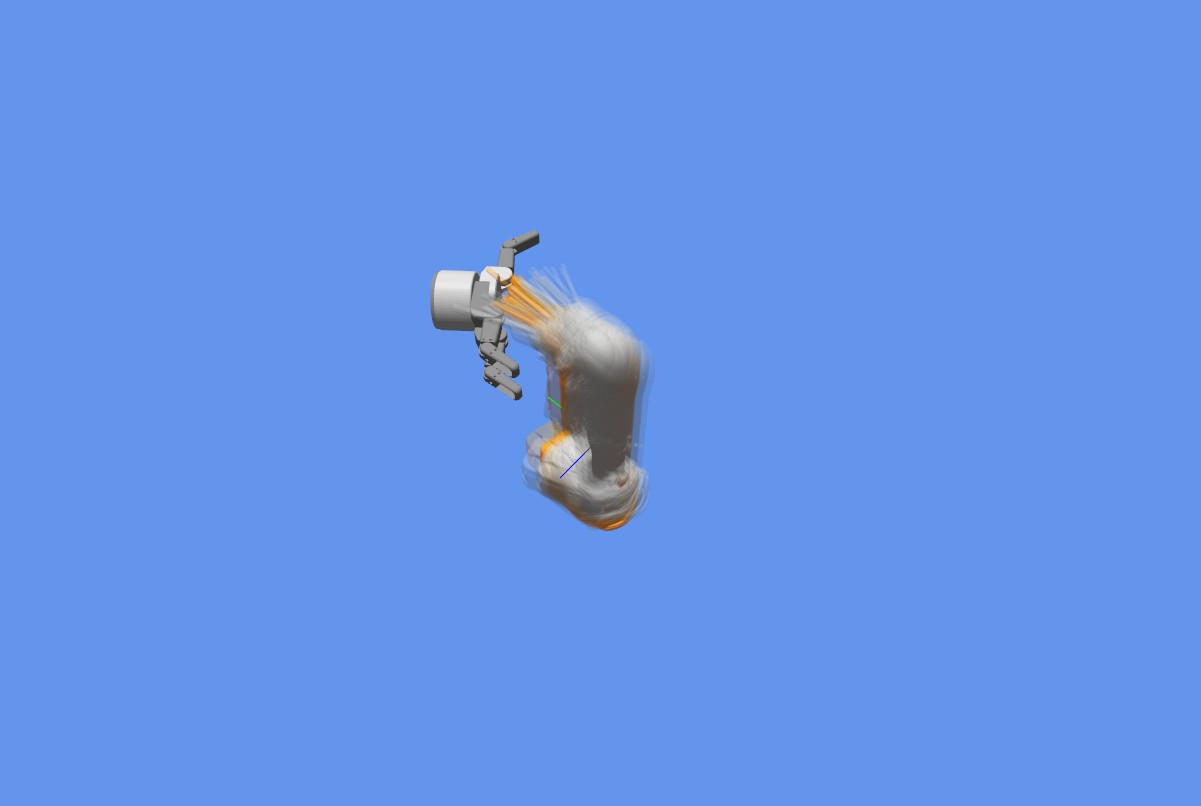}  \vspace{-2pt}   \\
  \noalign{\hrule height 2pt} \vspace{-8pt} \\ 
  \begin{sideways} \hspace{15pt} \Large Random \end{sideways} & \includegraphics[trim=\scalestopsyellowcolleft px \scalestopsyellowcolbottom px \scalestopsyellowcolright px \scalestopsyellowcoltop px, clip=true, scale=\scalestopsyellowcol]{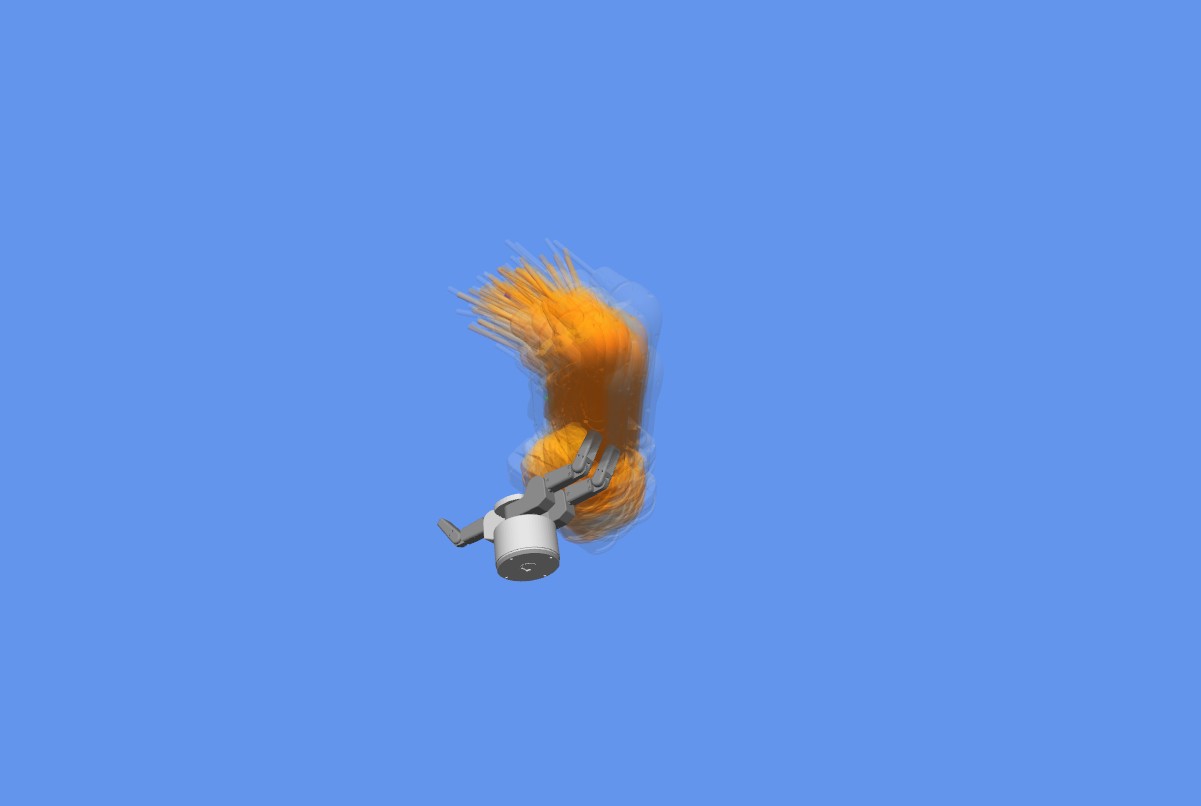} &  \includegraphics[trim=\scalestopsyellowcolleft px \scalestopsyellowcolbottom px \scalestopsyellowcolright px \scalestopsyellowcoltop px, clip=true, scale=\scalestopsyellowcol]{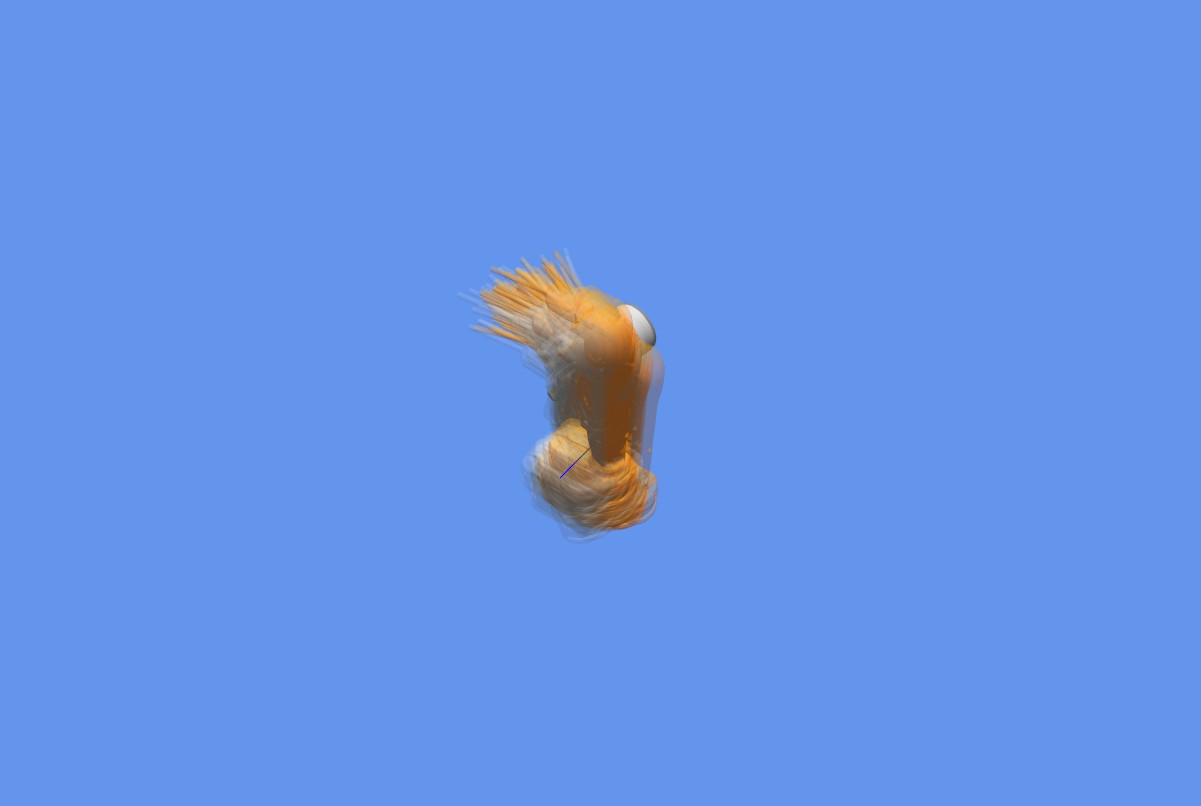} & \includegraphics[trim=\scalestopsyellowcolleft px \scalestopsyellowcolbottom px \scalestopsyellowcolright px \scalestopsyellowcoltop px, clip=true, scale=\scalestopsyellowcol]{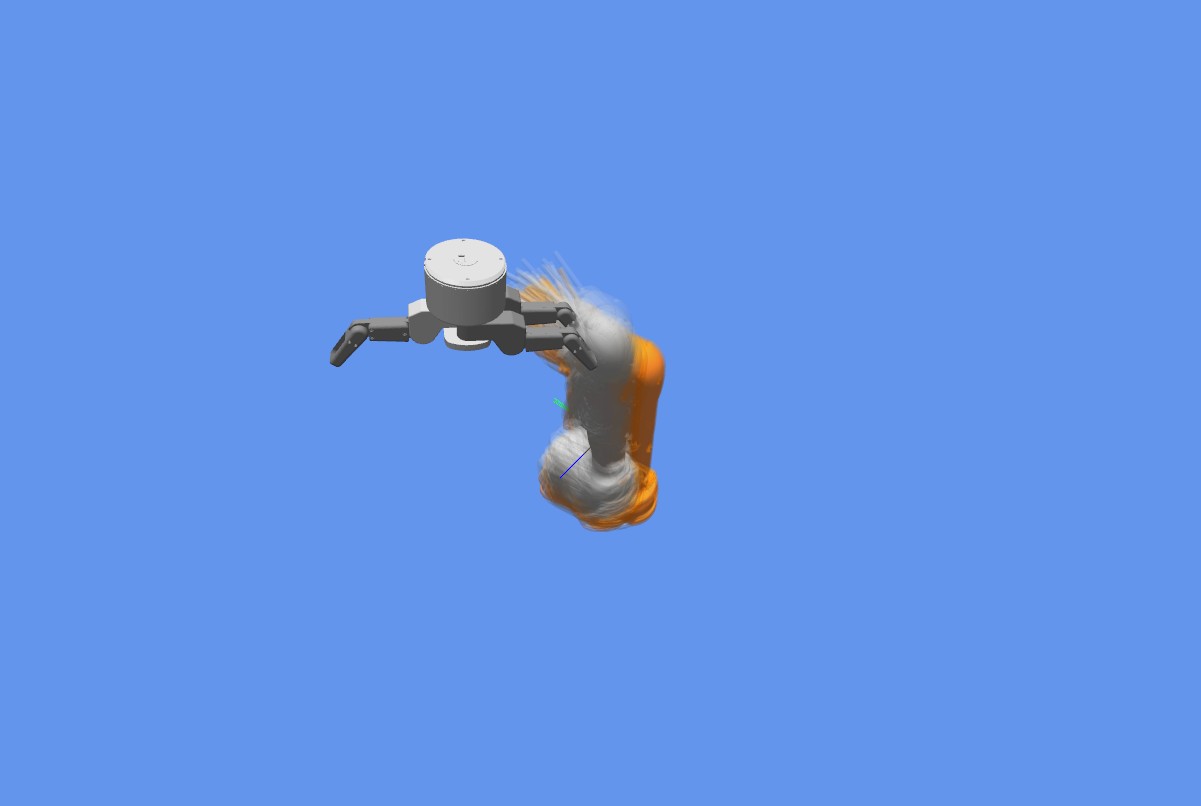} \vspace{-2pt} \\
  \noalign{\hrule height 2pt} \vspace{-8pt} \\ 
  \begin{sideways} \hspace{16pt} \Large Human \end{sideways} & \includegraphics[trim=\scalestopsyellowcolleft px \scalestopsyellowcolbottom px \scalestopsyellowcolright px \scalestopsyellowcoltop px, clip=true, scale=\scalestopsyellowcol]{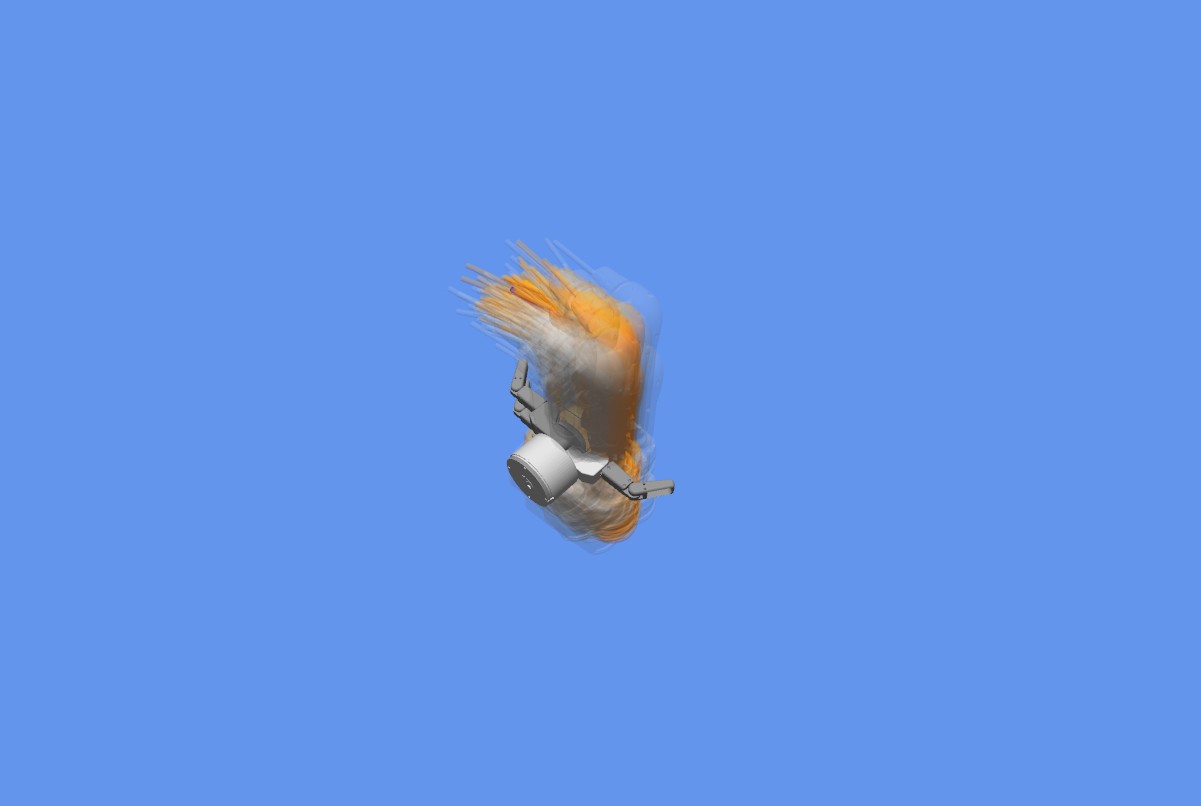} &  \includegraphics[trim=\scalestopsyellowcolleft px \scalestopsyellowcolbottom px \scalestopsyellowcolright px \scalestopsyellowcoltop px, clip=true, scale=\scalestopsyellowcol]{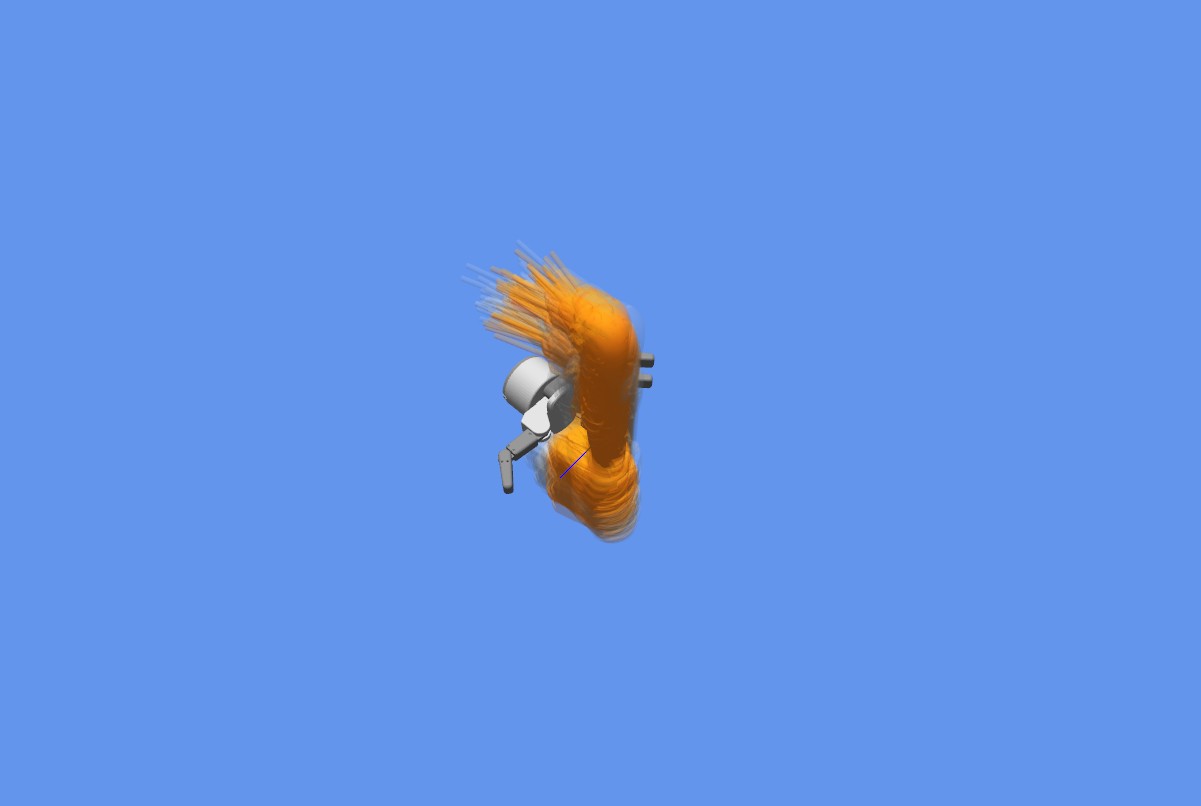} & \includegraphics[trim=\scalestopsyellowcolleft px \scalestopsyellowcolbottom px \scalestopsyellowcolright px \scalestopsyellowcoltop px, clip=true, scale=\scalestopsyellowcol]{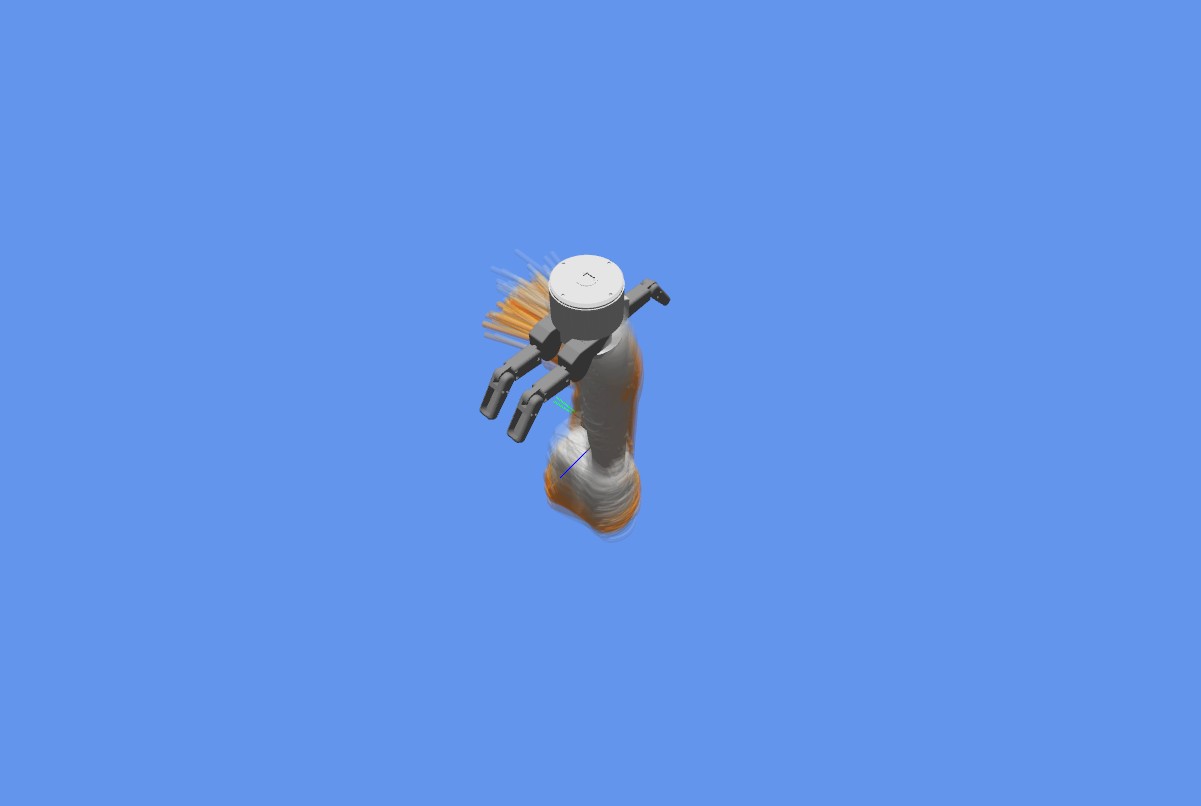} \vspace{-2pt}  
   \end{tabular}
   \caption{First three actions selected for each metric from the experiment in \figref{fig:drill_quivers}. Particles prior to the action are grey, and particles updated after observation are yellow.}
\label{table:eachaction_yellow}
\end{table}



\subsection{Robot Experiments} \label{sec:robot_experiments}
We implemented each of our methods (IG, HP, WHP) on a robot with a Barret arm and hand, and attempted to open a door. $X_{s}$ is initialized with a vision system corrupted with an artificial error of $0.035m$ in the $y$ direction. Our initial random realization $\randrealization$ is sampled from $N(\mu, \Sigma)$ with $\mu = X_{s}$, and $\Sigma$ a diagonal matrix with $\Sigma_{xx} = 0.02$, $\Sigma_{yy} = 0.04$, $\Sigma_{zz} = 0.02$, $\Sigma_{\theta\theta} = 0.08$. We fix $| \randrealization | = 2000$ hypotheses. We initially generate 600 normal action trajectories (\sref{normal_sample_section}), though after checking for kinematic feasibility, only about $70$ remain.

We utilize each of our uncertainty reducing methods prior to using an open-loop sequence to grasp the door handle. Once our algorithm selects the next action, we motion plan to its start pose and perform the straight line guarded-move using a task space controller. We sense contact by thresholding the magnitude reported by a force torque sensor in the Barret hand.

Without touch localization, the robot missed the door handle entirely. With any of our localization methods, the robot successfully opened the door, needing only two uncertainty reducing actions to do so. Selected actions are shown in \tabref{table:robot_results}, and full videos are provided online\footlabel{website_video}{\url{http://www.youtube.com/watch?v=_HiyKKDStBE}}.

\noindent \textbf{Observation 3:} Using our faster adaptive submodular metrics, selecting an action takes approximately as long as planning and executing it. This suggests that adaptive action selection will often outperform a non-adaptive plan generated offline that requires no planning time, but more actions.

\begin{table}[t]\setlength{\tabcolsep}{1pt}
\centering
\begin{tabular}{rccc}
    & \large IG & \large HP & \large WHP \\ 
  \begin{sideways} \hspace{22pt} \Large Action 1 \end{sideways} & \includegraphics[trim=\robotresultsleft px \robotresultsbottom px \robotresultsright px \robotresultstop px, clip=true, width=\robotresultsscale]{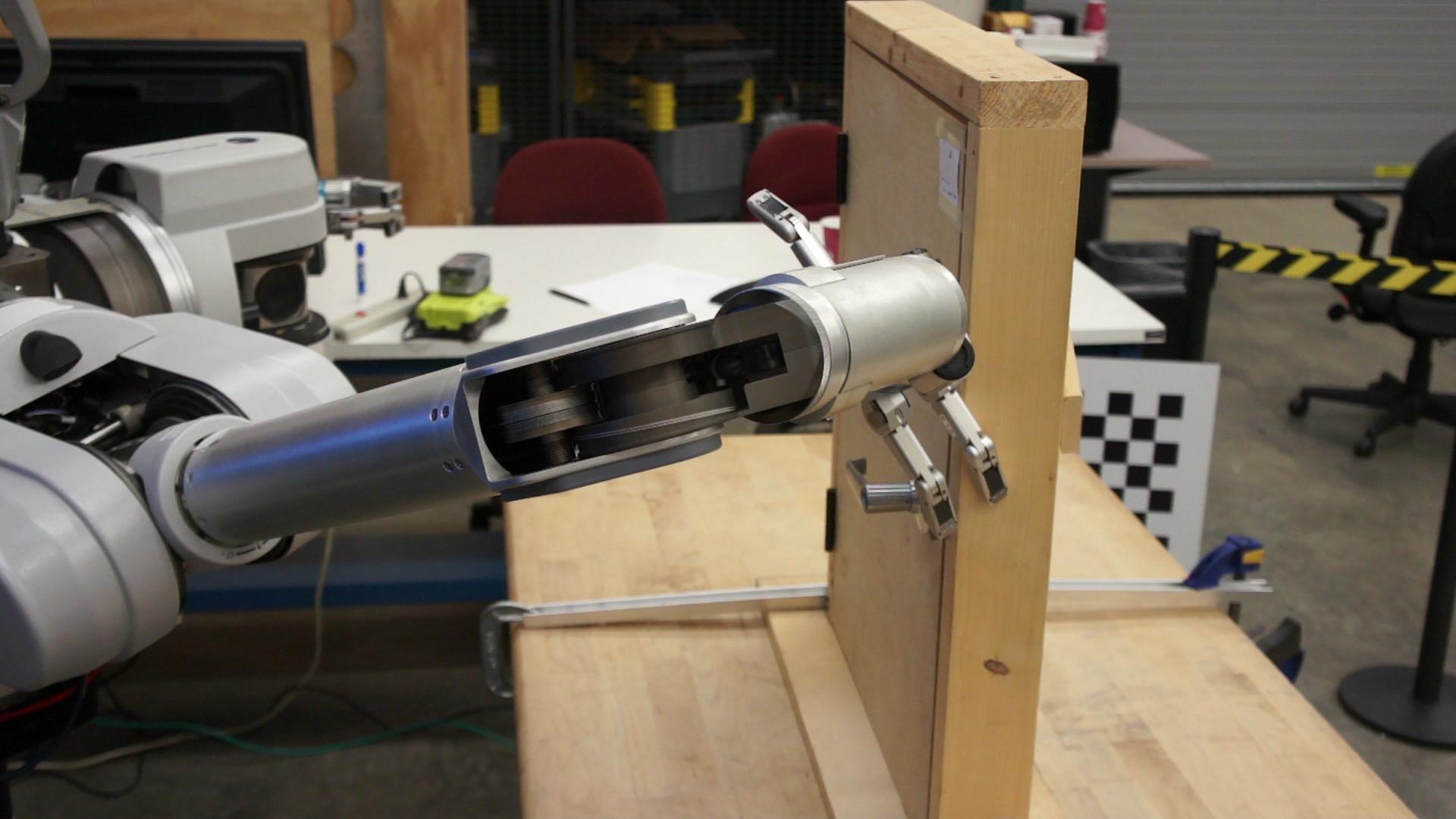} &  \includegraphics[trim=\robotresultsleft px \robotresultsbottom px \robotresultsright px \robotresultstop px, clip=true, width=\robotresultsscale]{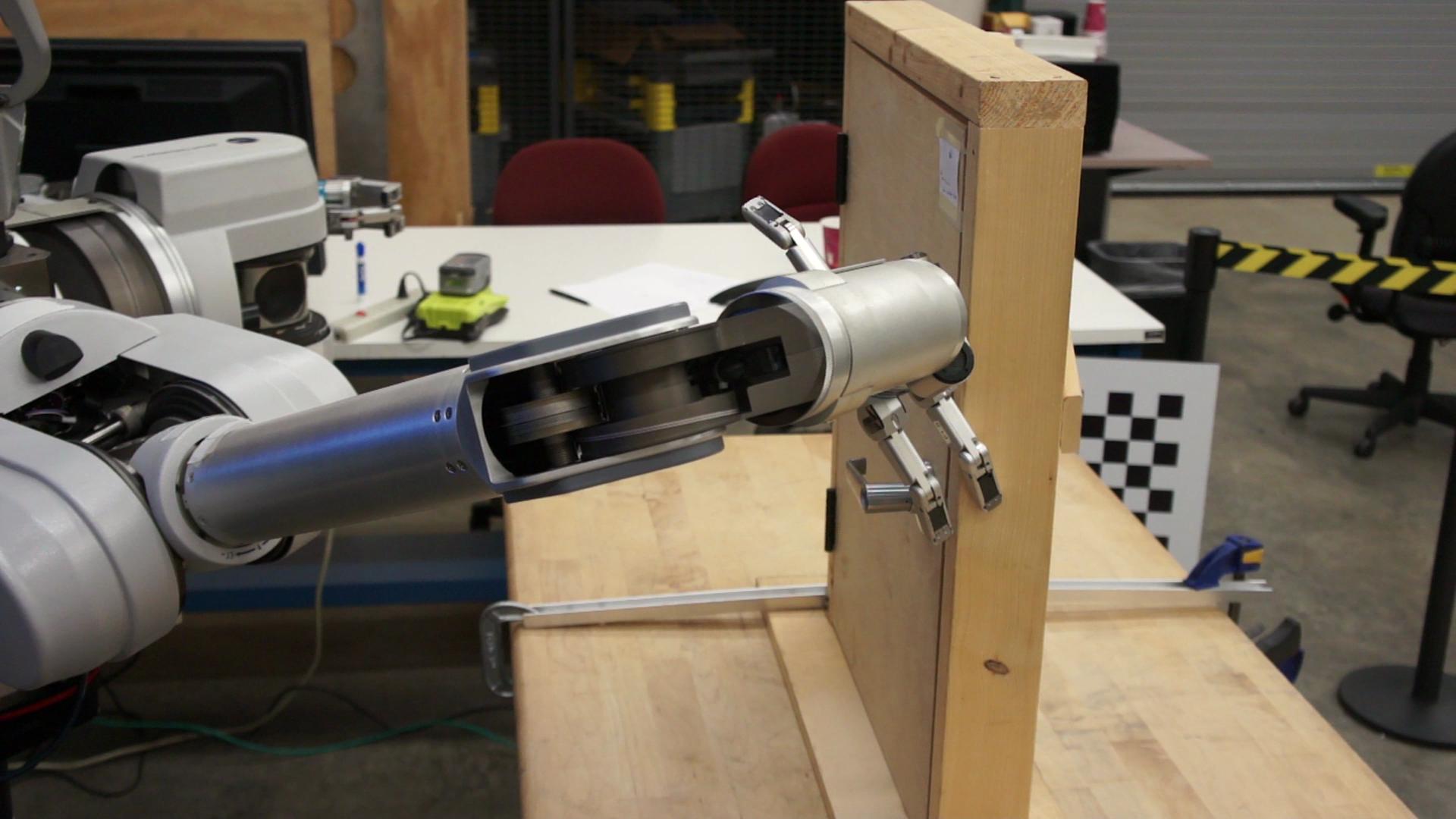} & \includegraphics[trim=\robotresultsleft px \robotresultsbottom px \robotresultsright px \robotresultstop px, clip=true, width=\robotresultsscale]{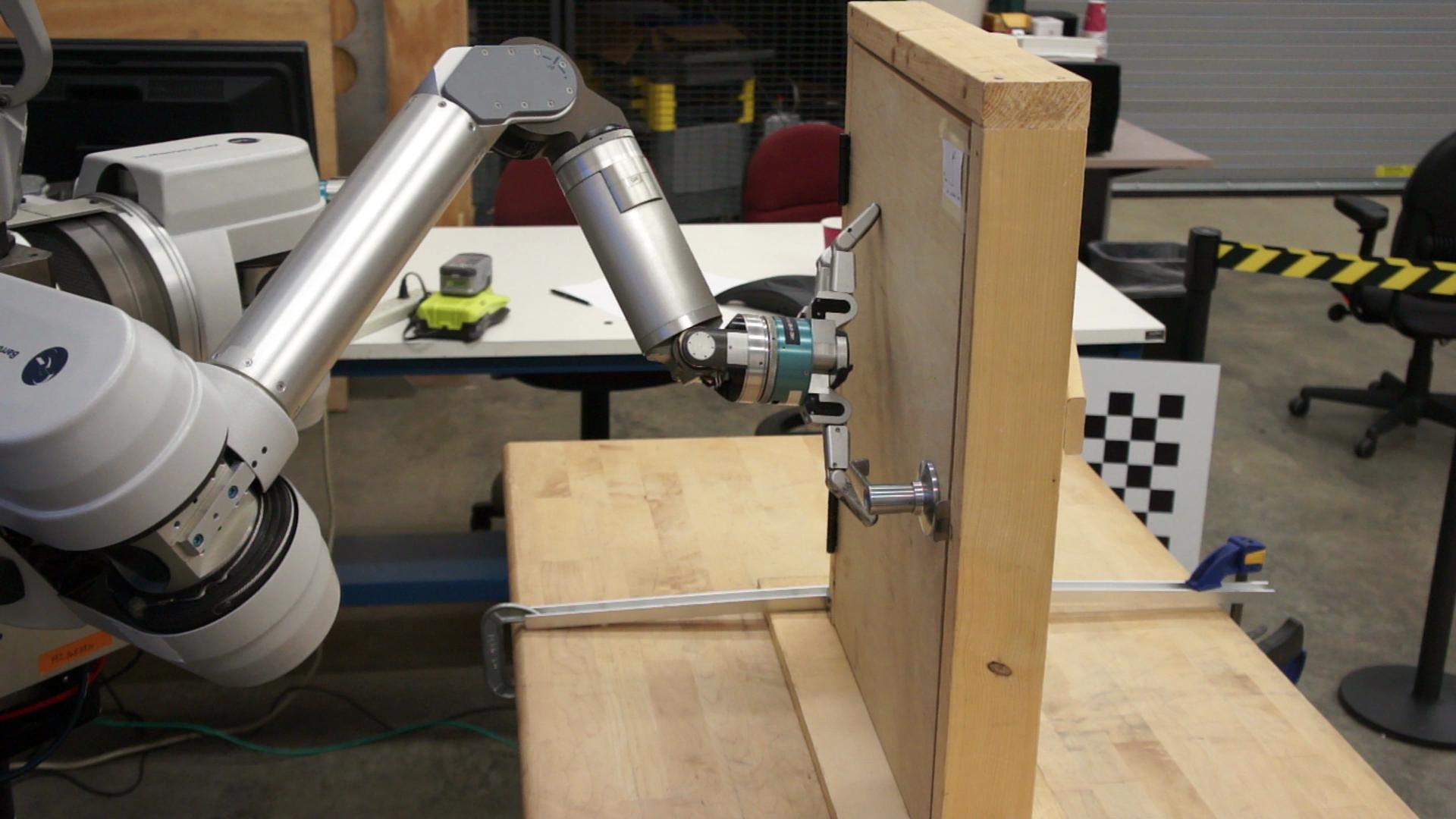}  \vspace{-2pt} \\
 \noalign{\hrule height 2pt} \vspace{-8pt} \\ 
  \begin{sideways} \hspace{22pt} \Large Action 2 \end{sideways} & \includegraphics[trim=\robotresultsleft px \robotresultsbottom px \robotresultsright px \robotresultstop px, clip=true, width=\robotresultsscale]{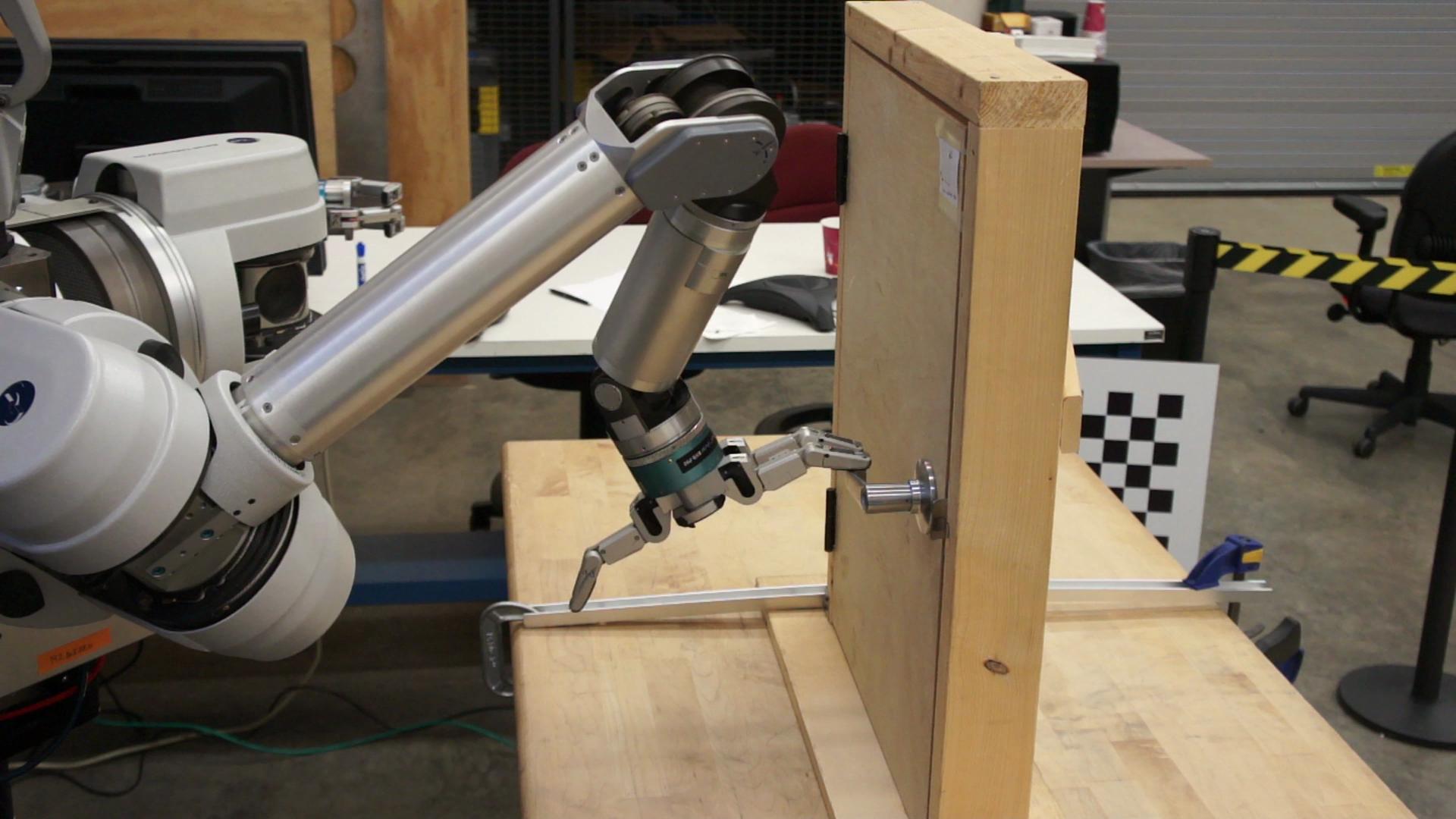} &  \includegraphics[trim=\robotresultsleft px \robotresultsbottom px \robotresultsright px \robotresultstop px, clip=true, width=\robotresultsscale]{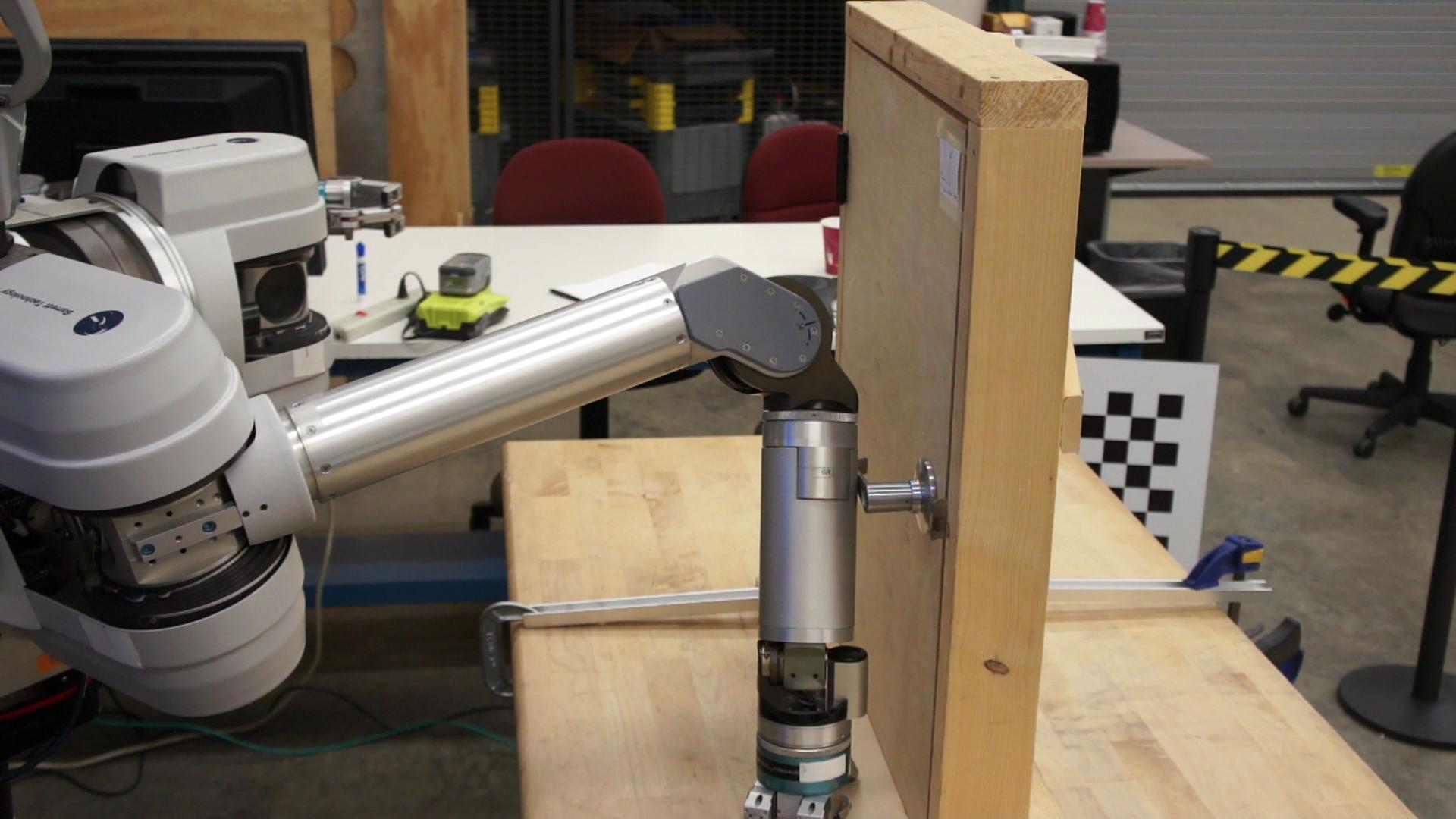} & \includegraphics[trim=\robotresultsleft px \robotresultsbottom px \robotresultsright px \robotresultstop px, clip=true, width=\robotresultsscale]{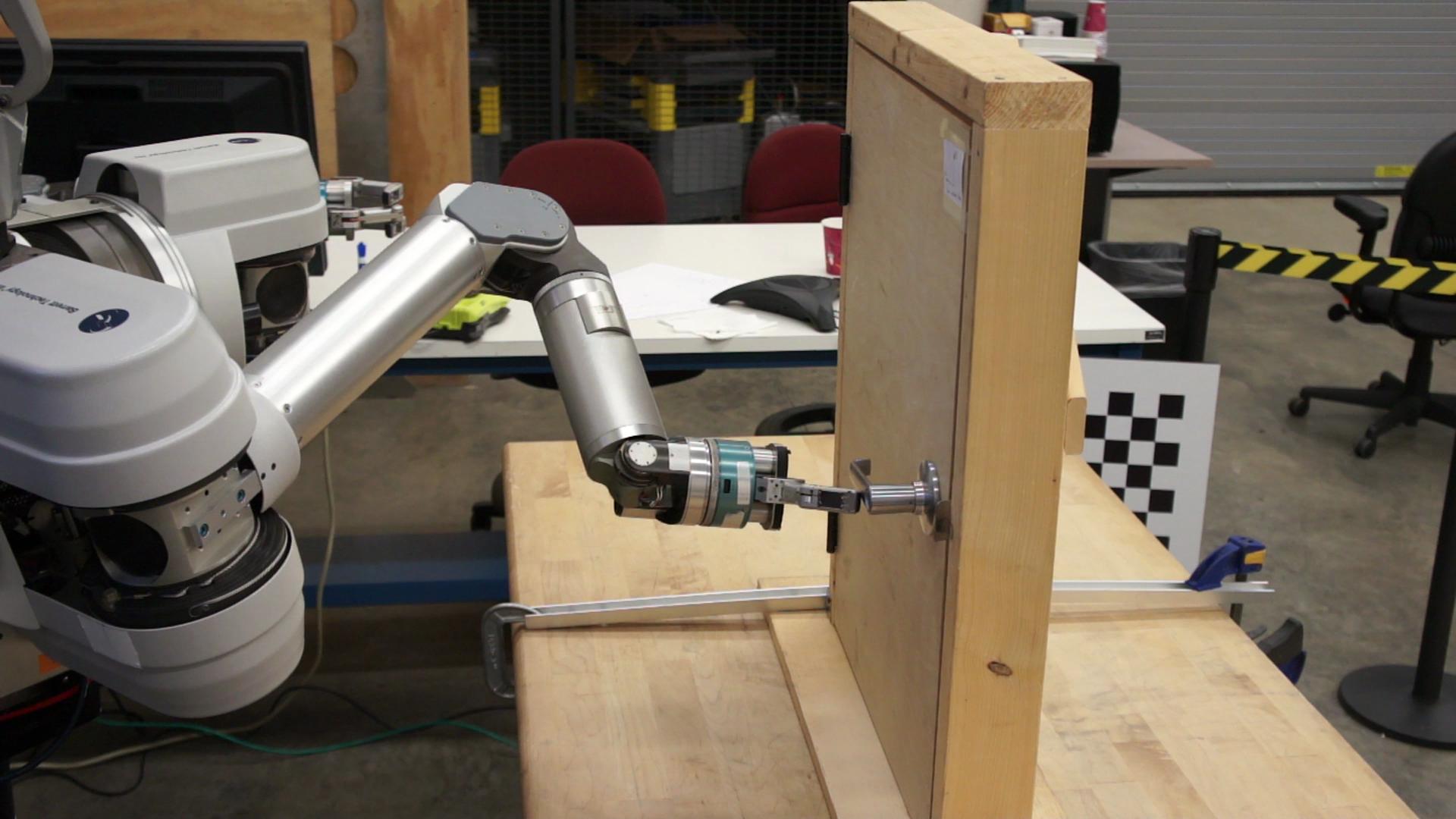}  \vspace{-2pt} \\
   \end{tabular}
   \caption{Actions selected during robot experiment. Interestingly, IG and HP select the same first action. All metrics led to a successful grasp of the door handle.}
\label{table:robot_results}
\end{table}

\section{Conclusion and Future Work}
In this work, we aim to show that greedy selection with the proper metric works well for information gathering actions, both theoretically and practically. To do so, we draw an explicit connection between submodularity and touch based localization. We start with Information Gain (IG), which has been used extensively for uncertainty reduction~\cite{cassandra_1996_acting_under_uncertainty, fox_1998_active_localization, bourgault_2002_info_exploration, batting_particle_filter, erickson_2008_blind, hsiao_2008_robust_belief, hebert_next_best_touch_2012}. We note the assumptions necessary for this metric to be submodular, rendering the greedy algorithm near-optimal in the \emph{offline} setting. We design our own metrics, Hypothesis Pruning (HP) and Weighted Hypothesis Pruning (WHP), which we show are adaptive submodular. Thus, greedy selection is guaranteed to provide near-optimal performance in the \emph{online} setting. In addition, these metrics are much faster, both due to their simplicity and a more efficient lazy-greedy algorithm~\cite{minoux_lazy,golovin_adaptive_2011}.


The methods presented here naively reduce uncertainty without considering the underlying task. In actuality, a task and planner may not require the exact pose, but that all uncertainty lie within a particular distribution. While we can apply our naive methods and terminate when this holds, we may achieve the task-based target more quickly by optimizing for it directly.

\section*{Acknowledgments}
We gratefully acknowledge support of the Army Research Laboratory under Cooperative Agreement Number W911NF-10-2-0016, DARPA-BAA-10-28, NSF-EEC-0540865, and the Intel Embedded Computing ISTC.


\bibliographystyle{IEEEtran}
\bibliography{refs}

\begin{thebibliography}{10}
\providecommand{\url}[1]{#1}
\csname url@rmstyle\endcsname
\providecommand{\newblock}{\relax}
\providecommand{\bibinfo}[2]{#2}
\providecommand\BIBentrySTDinterwordspacing{\spaceskip=0pt\relax}
\providecommand\BIBentryALTinterwordstretchfactor{4}
\providecommand\BIBentryALTinterwordspacing{\spaceskip=\fontdimen2\font plus
\BIBentryALTinterwordstretchfactor\fontdimen3\font minus
  \fontdimen4\font\relax}
\providecommand\BIBforeignlanguage[2]{{%
\expandafter\ifx\csname l@#1\endcsname\relax
\typeout{** WARNING: IEEEtran.bst: No hyphenation pattern has been}%
\typeout{** loaded for the language `#1'. Using the pattern for}%
\typeout{** the default language instead.}%
\else
\language=\csname l@#1\endcsname
\fi
#2}}

\bibitem{lazano_1984_fine_motion}
T.~Lozano-P\'erez, M.~T. Mason, and R.~H. Taylor, ``Automatic synthesis of
  fine-motion strategies for robots,'' \emph{{IJRR}}, vol.~3, no.~1, pp. 3--24,
  1984.

\bibitem{cassandra_1996_acting_under_uncertainty}
A.~R. Cassandra, L.~P. Kaelbling, and J.~A. Kurien, ``Acting under uncertainty:
  Discrete bayesian models for mobile-robot navigation,'' in \emph{{IEEE/RSJ}
  {IROS}}, 1996.

\bibitem{fox_1998_active_localization}
D.~Fox, W.~Burgard, and S.~Thrun, ``Active markov localization for mobile
  robots,'' \emph{Robotics and Autonomous Systems}, vol.~25, pp. 195--207,
  1998.

\bibitem{hsiao_2008_robust_belief}
K.~Hsiao, T.~Lozano-P\'erez, and L.~P. Kaelbling, ``Robust belief-based
  execution of manipulation programs,'' in \emph{{WAFR}}, 2008.

\bibitem{hebert_next_best_touch_2012}
P.~Hebert, J.~W. Burdick, T.~Howard, N.~Hudson, and J.~Ma, ``Action inference:
  The next best touch,'' in \emph{{RSS} Workshop on Mobile Manipulation}, 2012.

\bibitem{batting_particle_filter}
J.~Fu, S.~Srinivasa, N.~Pollard, and B.~Nabbe, ``Planar batting under shape,
  pose, and impact uncertainty,'' in \emph{{IEEE} {ICRA}}, 2007.

\bibitem{will_1975_guarded}
P.~M. Will and D.~D. Grossman, ``An experimental system for computer controlled
  mechanical assembly,'' \emph{IEEE Trans. Computers}, vol.~24, no.~9, pp.
  879--888, 1975.

\bibitem{kaelbling_1998_pomdp}
L.~P. Kaelbling, M.~L. Littman, and A.~R. Cassandra, ``Planning and acting in
  partially observable stochastic domains,'' \emph{Artificial Intelligence},
  vol. 101, pp. 99--134, 1998.

\bibitem{papad_1987_mdp_complexity}
C.~Papadimitriou and J.~N. Tsitsiklis, ``The complexity of markov decision
  processes,'' \emph{Math. Oper. Res.}, vol.~12, no.~3, pp. 441--450, 1987.

\bibitem{roy_2005_belief_compression}
N.~Roy, G.~Gordon, and S.~Thrun, ``Finding approximate pomdp solutions through
  belief compression,'' \emph{{JAIR}}, vol.~23, pp. 1--40, 2005.

\bibitem{smith_2005_point_pomdp}
T.~Smith and R.~G. Simmons, ``Point-based {POMDP} algorithms: Improved analysis
  and implementation,'' in \emph{{UAI}}, 2005.

\bibitem{kurniawati_2008_sarsop}
H.~Kurniawati, D.~Hsu, and W.~S. Lee, ``Sarsop: Efficient point-based pomdp
  planning by approximating optimally reachable belief spaces,'' in
  \emph{{RSS}}, 2008.

\bibitem{shani_2012_survey_pointbased}
G.~Shani, J.~Pineau, and R.~Kaplow, ``A survey of point-based pomdp solvers,''
  \emph{{AAMAS}}, pp. 1--51, 2012.

\bibitem{ross_2008_online_pomdp}
S.~Ross, J.~Pineau, S.~Paquet, and B.~Chaib-draa, ``Online planning algorithms
  for pomdps,'' \emph{{JAIR}}, vol.~32, pp. 663--704, 2008.

\bibitem{kaijen_thesis}
K.~Hsiao, ``Relatively robust grasping,'' Ph.D. dissertation, Massachusetts
  Institute of Technology, 2009.

\bibitem{feige_1998}
U.~Feige, ``A threshold of ln n for approximating set cover,'' \emph{{JACM}},
  vol.~45, no.~4, pp. 634--652, 1998.

\bibitem{wolsey_1982}
L.~A. Wolsey, ``An analysis of the greedy algorithm for the submodular set
  covering problem,'' \emph{Combinatorica}, vol.~2, no.~4, pp. 385--393, 1982.

\bibitem{bourgault_2002_info_exploration}
F.~Bourgault, A.~A. Makarenko, S.~B. Williams, B.~Grocholsky, and H.~F.
  Durrant-Whyte, ``Information based adaptive robotic exploration,'' in
  \emph{{IEEE/RSJ} {IROS}}, 2002.

\bibitem{zheng_2005_active_diagnosis}
A.~X. Zheng, I.~Rish, and A.~Beygelzimer, ``Efficient test selection in active
  diagnosis via entropy approximation,'' in \emph{{UAI}}, 2005.

\bibitem{erickson_2008_blind}
L.~Erickson, J.~Knuth, J.~M. O'Kane, and S.~M. LaValle, ``Probabilistic
  localization with a blind robot,'' in \emph{{IEEE} {ICRA}}, 2008.

\bibitem{krause_2005_submodular_entropy}
A.~Krause and C.~Guestrin, ``Near-optimal nonmyopic value of information in
  graphical models,'' in \emph{{UAI}}, 2005.

\bibitem{hollinger_isrr_2011}
G.~A. Hollinger, U.~Mitra, and G.~S. Sukhatme, ``Active classiﬁcation: Theory
  and application to underwater inspection,'' in \emph{{ISRR}}, 2011.

\bibitem{golovin_adaptive_2011}
D.~Golovin and A.~Krause, ``Adaptive submodularity: Theory and applications in
  active learning and stochastic optimization,'' \emph{{JAIR}}, vol.~42, no.~1,
  pp. 427--486, 2011.

\bibitem{minoux_lazy}
M.~Minoux, ``Accelerated greedy algorithms for maximizing submodular set
  functions,'' in \emph{Optimization Techniques}, 1978, vol.~7.

\bibitem{erez_smart_localpomdp}
T.~Erez and W.~D. Smart, ``A scalable method for solving high-dimensional
  continuous pomdps using local approximation,'' in \emph{{UAI}}, 2010.

\bibitem{platt_2011_hypoth_based_non_gaussian}
R.~Platt, L.~Kaelbling, T.~Lozano-P\'erez, and R.~Tedrake, ``A hypothesis-based
  algorithm for planning and control in non-gaussian belief spaces,''
  Massachusetts Institute of Technology, Tech. Rep. CSAIL-TR-2011-039, 2011.

\bibitem{petrov_localization}
A.~Petrovskaya and O.~Khatib, ``Global localization of objects via touch,''
  \emph{{IEEE} Trans. on Robotics}, vol.~27, no.~3, pp. 569--585, 2011.

\bibitem{arms_video}
DARPAtv, ``Darpa autonomous robotic manipulation (arm) - phase 1,''
  \url{http://www.youtube.com/watch?v=jeABMoYJGEU}.

\bibitem{herb_video}
A.~Collet, C.~Dellin, M.~Dogar, A.~Dragan, S.~Javdani, K.~Strabala, M.~V.
  Weghe, and S.~Srinivasa, ``Herb prepares a meal,''
  \url{http://www.youtube.com/watch?v=9Oav3JajR7Q}.

\bibitem{dogar_2010_push_grasp}
M.~Dogar and S.~Srinivasa, ``Push-grasping with dexterous hands: Mechanics and
  a method,'' in \emph{{IEEE/RSJ} {IROS}}, 2010.

\bibitem{nowak_gbs}
R.~Nowak, ``Generalized binary search,'' in \emph{Proc. Allerton Conference on
  Communications, Control, and Computing}, 2008.

\bibitem{nowak_noisy_gbs}
------, ``Noisy generalized binary search,'' in \emph{{NIPS}}, 2009.

\bibitem{golovin_bayesian_noisy_obs}
D.~Golovin, A.~Krause, and D.~Ray, ``Near-optimal bayesian active learning with
  noisy observations,'' in \emph{{NIPS}}, 2010.

\bibitem{bellala_2012_query_selection}
G.~Bellala, S.~K. Bhavnani, and C.~Scott, ``Group-based active query selection
  for rapid diagnosis in time-critical situations,'' \emph{{IEEE} Transactions
  on Information Theory}, vol.~58, no.~1, pp. 459--478, 2012.

\bibitem{kaijen_reactive_optical}
K.~Hsiao, P.~Nangeroni, M.~Huber, A.~Saxena, and A.~Y. Ng, ``Reactive grasping
  using optical proximity sensors,'' in \emph{{IEEE} {ICRA}}, 2009.

\bibitem{pretouch_smith}
L.-T. Jiang and J.~R. Smith, ``Seashell effect pretouch sensing for robotic
  grasping,'' in \emph{{IEEE} {ICRA}}, 2012.

\end{thebibliography}

\clearpage
\onecolumn
\section{Appendix}
\label{sec_appendix}
Here we present the theorems and proofs showing the Hypothesis Pruning metrics are near-optimal. To do so, we prove our metrics are adaptive submodular, strongly adaptive monotone, and self-certifying. Note that the bounds on adaptive submodular functions require that observations are not noisy - that is, for a fixed realization $\realization$, an action is consistent with exactly one observation. In our case, we would like multiple observations to be consistent with a $\realization$, as our sensors are noisy. Thus, we first construct a non-noisy problem by creating many weighted ``noisy'' copies of each realization $\realization$. We then show how to compute our objective on the original problem. Finally, we prove our bounds.

\subsection{Constructing the Non-Noisy Problem}
\label{sec_nonnoisy}
Similar to previous approaches in active learning, we construct a non-noisy problem by creating ``noisy'' copies of each realization $\realization$ for every possible noisy observation~\cite{golovin_bayesian_noisy_obs, bellala_2012_query_selection}. Let $\noisingfuncaction(\realization) = \{\noisyrealization_{1}, \dots\, \noisyrealization_{\numnoisy}\}$ be the function that creates $\numnoisy$ noisy copies for action $\actionitem$. Here, the original probability of $\realization$ is distributed among all the noisy copies, $p(\realization) = \sum_{\noisyrealization \in \noisingfuncaction(\realization)} p(\noisyrealization)$. For convenience, we will also consider overloading $\noisingfunc$ to take sets of actions, and sets of realizations. Let $\noisingfuncactionset(\realization)$ recursively apply $\noisingfunc$ for each $\actionitem \in \actionset$. That is, if we let $\actionset = \{ \actionitem_1, \actionitem_2, \dots \}$ we apply $\noisingfunc_{\actionitem_1}$ to $\realization$, then $\noisingfunc_{\actionitem_2}$ to every output of $\noisingfunc_{\actionitem_1}(\realization)$, and so on. Note that we still have $p(\realization) = \sum_{\noisyrealization \in \noisingfuncactionset(\realization)} p(\noisyrealization)$. Additionally, let $\noisingfuncactionset(\randrealization)$ apply $\noisingfuncactionset$ to each $\realization \in \randrealization$ and combine the set.

The probability of each noisy copy comes from our weighting functions defined in~\sref{sec_hypothprune}:
\begin{align*}
  \noisingfuncactionset(\realization) &= \{\noisyrealization_{1}, \dots\, \noisyrealization_{\numnoisy}\} \\
  p(\noisyrealization) &= p(\realization) \frac{\displaystyle \weighting_{\actionitem_{\noisyrealization}} (\actionitem_{\realization})}{\displaystyle \sum_{\noisyrealization' \in \noisingfuncaction(\realization)} \weighting_{\actionitem_{\noisyrealization'} } (\actionitem_{\realization})} \hspace{5mm} \quad \text{(one action)}\\
  p(\noisyrealization) &= p(\realization) \prod_{\actionitem \in \actionset} \frac{\displaystyle \weighting_{\actionitem_{\noisyrealization}} (\actionitem_{\realization})}{\displaystyle \sum_{\noisyrealization' \in \noisingfuncaction(\realization)} \weighting_{\actionitem_{\noisyrealization'} } (\actionitem_{\realization})} \quad \text{(multiple actions)}
\end{align*}

For simplicity, we also assume that the maximum value of our weighting function is equal to one for any action. We note that our weighting functions in~\sref{sec_hypothprune} have this property for the non-noisy observation where $\actionitem_\noisyrealization = \actionitem_\realization$:
\begin{align}
  \max_{\noisyrealization \in \noisingfuncaction(\realization)} \weighting_{\actionitem_{\noisyrealization}}(\actionitem_\realization) &= 1 \quad \forall \realization, \actionitem \label{eq_maxweightval}
\end{align}

We build our set of non-noisy realizations $\noisyrandrealization = \noisingfuncallaction(\randrealization)$. Our objective function is over $\noisyrandrealization$, specifying the probability mass removed form the original problem. One property we desire is if our observations are consistent with one noisy copy of $\realization$, then we keep some proportion of all of the noisy copy (proportional to our weighting function $\weighting^{HP}$ or $\weighting^{WHP}$). In our HP algorithm for example, if any noisy copy of $\realization$ remains, the objective function acts as if all of the probability mass remains. Following Golovin and Krause~\cite{golovin_adaptive_2011}, we define our utility function over a set of actions $\actionset$, and the observations we would receive if $\noisyrealization$ were the true state:
\begin{align*}
  f(\actionset, \noisyrealization) &= 1 - \sumrealization \left( \prod_{\actionitem \in \actionset} \frac{p(\realization)}{\max p(\noisingfuncaction(\realization))}\right) \left( \sumnoisyfromrealization{\allactionset} p(\noisyrealization') \prod_{\actionitem \in \actionset} \delta_{\actionitem_\noisyrealization \actionitem_{\noisyrealization'} } \right)
\end{align*}

Where $\delta_{\actionitem_\noisyrealization \actionitem_{\noisyrealization'} }$ is the Kronecker delta function, equal to $1$ if $\actionitem_\noisyrealization  = \actionitem_{\noisyrealization'} $ and $0$ otherwise, $\realization$ is the original realization from which $\noisyrealization$ was produced, and $\max p(\noisingfuncaction(\realization))$ is the highest probability of the ``noisy'' copies. By construction, any action will keep at most $\max p(\noisingfuncaction(\realization))$ probability mass per action, since at most one noisy copy from $\noisingfuncaction(\realization)$ will be consistent with the observation. Intuitively, multiplying by $\frac{p(\realization)}{\max p(\noisingfuncaction(\realization))}$ will make it so if we kept the highest weighted noisy copy of $\realization$, our objective would be equivalent to keeping the entire realization $\realization$.
\begin{align*}
  f(\actionset, \noisyrealization) &= 1 - \sumrealization \left( \prod_{\actionitem \in \actionset} \frac{p(\realization)}{\max p(\noisingfuncaction(\realization))}\right) \left( \sum_{\noisyrealization' \in \noisingfuncallaction(\realization)} p(\noisyrealization') \prod_{\actionitem \in \actionset} \delta_{\actionitem_\noisyrealization \actionitem_{\noisyrealization'} } \right)\\
  &= 1 - \sumrealization \left( \prod_{\actionitem \in \actionset} \frac{p(\realization)}{\max p(\noisingfuncaction(\realization))}\right) \left( \sum_{\noisyrealization' \in \noisingfuncactionset(\realization)} \sum_{\noisyrealization'' \in \noisingfuncallotheraction(\noisyrealization')} p(\noisyrealization'') \prod_{\actionitem \in \actionset} \delta_{\actionitem_\noisyrealization \actionitem_{\noisyrealization''} } \right)\\
  &= 1 - \sumrealization \left( \prod_{\actionitem \in \actionset} \frac{p(\realization)}{\max p(\noisingfuncaction(\realization))}\right) \left( \sum_{\noisyrealization' \in \noisingfuncactionset(\realization)} \left(\prod_{\actionitem \in \actionset} \delta_{\actionitem_\noisyrealization \actionitem_{\noisyrealization'} } \right) \sum_{\noisyrealization'' \in \noisingfuncallotheraction(\noisyrealization')} p(\noisyrealization'')  \right)\\
  &= 1 - \sumrealization \left( \prod_{\actionitem \in \actionset} \frac{p(\realization)}{\max p(\noisingfuncaction(\realization))}\right) \left( \sum_{\noisyrealization' \in \noisingfuncactionset(\realization)}  p(\noisyrealization')  \prod_{\actionitem \in \actionset}\delta_{\actionitem_\noisyrealization \actionitem_{\noisyrealization'} }  \right)
\end{align*}

Here, we separate the recursive splitting over the realization $\realization$ into those split based on actions in $\actionset$ and those split from other actions. Since $\prod_{\actionitem \in \actionset} \delta_{\actionitem_\noisyrealization \actionitem_{\noisyrealization''} }$ only depends on the response to actions in $\actionset$, it only depends on noisy copies made from $\noisingfuncactionset$. Thus, we can factor those out. Additionally, we marginalize over the copies of $\noisyrealization'$ as $\sum_{\noisyrealization'' \in \noisingfuncallotheraction(\noisyrealization')} p(\noisyrealization'') = p(\noisyrealization')$. Overall, this simplification enables us to only consider the copies from the actions in $\actionset$. We further simplify:
\begin{align}
  f(\actionset, \noisyrealization) &= 1 - \sumrealization \left( \prod_{\actionitem \in \actionset} \frac{p(\realization)}{\max p(\noisingfuncaction(\realization))}\right) \left( \sum_{\noisyrealization' \in \noisingfuncactionset(\realization)}  p(\noisyrealization')  \prod_{\actionitem \in \actionset}\delta_{\actionitem_\noisyrealization \actionitem_{\noisyrealization'} }  \right) \notag\\
  &= 1 - \sumrealization \left( \prod_{\actionitem \in \actionset} p(\realization)  \left( \frac{\noisyprobwithweightingdenom}{ \noisyprobwithweightingmaxval p(\realization) }  \right)  \right) \left( \sum_{\noisyrealization' \in \noisingfuncactionset(\realization)}  p(\noisyrealization')  \prod_{\actionitem \in \actionset}\delta_{\actionitem_\noisyrealization \actionitem_{\noisyrealization'} }  \right) \label{eq_subinnoising} \\
  &= 1 - \sumrealization \left( \prod_{\actionitem \in \actionset} \noisyprobwithweightingdenom \right) \left( \sum_{\noisyrealization' \in \noisingfuncactionset(\realization)}  p(\noisyrealization')  \prod_{\actionitem \in \actionset}\delta_{\actionitem_\noisyrealization \actionitem_{\noisyrealization'} }  \right) \label{eq_usedmaxweight}\\
  &= 1 - \sumrealization \left( \prod_{\actionitem \in \actionset} \noisyprobwithweightingdenom \right) \left( \sum_{\noisyrealization' \in \noisingfuncactionset(\realization)} p(\realization) \left( \prod_{\actionitem \in \actionset} \frac{\noisyprobwithweightingnumer}{\noisyprobwithweightingdenom} \right) \prod_{\actionitem \in \actionset}\delta_{\actionitem_\noisyrealization \actionitem_{\noisyrealization'} }  \right) \notag \\
  &= 1 - \sumrealization   p(\realization) \sum_{\noisyrealization' \in \noisingfuncactionset(\realization)} \left( \prod_{\actionitem \in \actionset} \noisyprobwithweightingdenom \right) \left( \prod_{\actionitem \in \actionset} \frac{\noisyprobwithweightingnumer}{\noisyprobwithweightingdenom} \right) \left(\prod_{\actionitem \in \actionset}\delta_{\actionitem_\noisyrealization \actionitem_{\noisyrealization'} } \right) \notag\\
  &= 1 -  \sumrealization  p(\realization) \sum_{\noisyrealization' \in \noisingfuncactionset(\realization)} \left( \prod_{\actionitem \in \actionset} \noisyprobwithweightingnumer \right) \left( \prod_{\actionitem \in \actionset}\delta_{\actionitem_\noisyrealization \actionitem_{\noisyrealization'} } \right) \notag\\
  &= 1 - \sumrealization   p(\realization) \sum_{\noisyrealization' \in \noisingfuncactionset(\realization)} \prod_{\actionitem \in \actionset} \noisyprobwithweightingnumer \delta_{\actionitem_\noisyrealization \actionitem_{\noisyrealization'} } \notag
\end{align}

Where \eref{eq_subinnoising} plugging in the value of $\noisingfuncaction(\realization)$, \eref{eq_usedmaxweight} used \eref{eq_maxweightval} above. Now we consider how the function $\noisingfunc$ generates noisy copies. We require that exactly one noisy copy  $\noisyrealization' \in \noisingfuncactionset(\realization)$ agree with every observation received so far, and thus only one term will have a nonzero product $\prod_{\actionitem \in \actionset} \delta_{\actionitem_\noisyrealization \actionitem_{\noisyrealization'} }$. We defer further specific details of $\noisingfunc$ until the next section. We get:
\begin{align*}
  f(\actionset, \noisyrealization) &= 1 - \sumrealization   p(\realization) \sum_{\noisyrealization' \in \noisingfuncactionset(\realization)} \prod_{\actionitem \in \actionset} \noisyprobwithweightingnumer \delta_{\actionitem_\noisyrealization \actionitem_{\noisyrealization'} } \\
  &= 1 - \sumrealization   p(\realization) \prod_{\actionitem \in \actionset} \noisyprobwithweightingnumernoapos
\end{align*}

At this point we can see how this equals the objective function for partial realization $\partialrealization$ in \sref{sec_hypothprune}. Let $\partialrealization_\noisyrealization = \{\actionset,\actionset_\noisyrealization\}$, where $\actionset$ are the actions taken and $\actionset_\noisyrealization$ are the observations. We can see that:
\begin{align*}
  f(\actionset, \noisyrealization) &= 1 - \sumrealization   p(\realization) \prod_{\actionitem \in \actionset} \noisyprobwithweightingnumernoapos \\
  &= 1 - \sumrealization p_{\partialrealization_\noisyrealization}(\realization)\\
  &= 1 - \probmassall_{\partialrealization_\noisyrealization}\\
  &= f(\partialrealization_\noisyrealization)
\end{align*}

\subsection{Observation Probabilities}
To compute expected marginal utilities, we will need to define our space of possible observations, and the corresponding probability for these observations. Let $\partialrealization=\{\tactionset,\tobservationset\}$ be the partial realization, and $p(\actionitem_\randrealization = \observationitem | \partialrealization)$ be the probability of receiving observation $\observationitem$ from performing action $\actionitem$ conditioned on the partial realization $\partialrealization$. Intuitively, this will correspond to how much probability mass agrees with the observation. More formally:
\begin{align*}
  p(\actionitem_\randrealization = \observationitem | \partialrealization) &\propto \sumrealization \sumnoisyfromrealization{\allactionset} p(\noisyrealization) \delta_{\actionitem_\noisyrealization \observationitem} \prodpartialdelta
\end{align*}

Similar to before, we will be able to consider noisy copies made from only actions in $\tactionset$ and $\actionitem$ (the derivation follows exactly as in \sref{sec_nonnoisy}). This will simplify to:
\begin{align*}
  p(\actionitem_\randrealization = \observationitem | \partialrealization) &\propto \sumrealization \sumnoisyfromrealization{\{\tactionset \cup \actionitem\}} p(\noisyrealization) \delta_{\actionitem_\noisyrealization \observationitem} \prodpartialdelta\\
  &= \sumrealization p(\realization) \sumnoisyfromrealization{\tactionset} \sumnoisyfromnoisyrealization{\actionitem}  \left(\noisyprobwithweighting   \delta_{\actionitem_{\noisyrealization'} \observationitem} \right) \left( \prodpartialdeltaweight \right)
\end{align*}

The first term in parenthesis comes from the weighting of performing action $\actionitem$ and receiving observation $\observationitem$, where we would like the only noisy copy of $\realization$ that agrees with that observation. The second term comes from that same operation, but for all actions and observations in $\partialrealization$. Again, we know by construction that exactly one noisy copy agrees with all observations. Hence, we can write this as:
\begin{align*}
  p(\actionitem_\randrealization = \observationitem | \partialrealization) &\propto \sumrealization p(\realization) \sumnoisyfromrealization{\tactionset} \sumnoisyfromnoisyrealization{\actionitem}  \left(\noisyprobwithweighting   \delta_{\actionitem_{\noisyrealization'} \observationitem} \right) \left( \prodpartialdeltaweight \right)\\
  &= \sumrealization p(\realization) \left(\noisyprobwithweightingobs \right) \left( \prodpartialweightobs \right)
\end{align*}

Finally, we would also like for $\noisyprobwithweightingdenom$ to be constant for all actions $\actionitem$ and realizations $\realization$, enabling us to factor those terms out. To approximately achieve this, we generate noisy copies by discretizing the trajectory uniformly along the path, and generate a noisy copy of each realization $\realization$ at every discrete location. We approximate our realization to be set at one of the discrete locations, such that $\actionitem_\realization$ is equal to the nearest discrete location. For many locations, the weighting function will be less than some negligible $\epsilon$. Let there be $\numnoisy$ discrete locations for any $\realization$ and $\actionitem$ where $\weighting_{\actionitem_\realization} > \epsilon$. We say that $|\noisingfuncaction(\realization)| = K$. Thus, we can fix the value of $\noisyprobwithweightingdenom = \kappa \quad \forall \actionitem, \realization$. Note that we also need to be consistent with observations corresponding to not contacting an object anywhere along the trajectory. Therefore, we also consider $\numnoisy$ noisy copies for this case. Under these assumptions, we can further simplify:

\begin{align*}
  p(\actionitem_\randrealization = \observationitem | \partialrealization) &\propto \sumrealization p(\realization) \left(\noisyprobwithweightingobs \right) \left( \prodpartialweightobs \right)\\
  &\approx \sumrealization p(\realization) \left(\frac{\noisyprobwithweightingnumerobs}{\kappa} \right) \left( \prodpartialweightobskappa \right)\\
  &\propto \sumrealization p(\realization) \noisyprobwithweightingnumerobs \prodpartial \noisyprobwithweightingnumertobs\\
  &= \sumrealization p_{\partialrealization}(\realization) \noisyprobwithweightingnumerobs\\
  &= \probmass_{\partialrealization,\actionitem,\observationitem}
\end{align*}

Finally, we need to normalize all observations to get:
\begin{align*}
    p(\actionitem_\randrealization = \observationitem | \partialrealization) &=\frac{\probmass_{\partialrealization, \actionitem, \observationitem}}{ \sum_{\observationitem' \in \allobservationset_\actionitem} \probmass_{\partialrealization,\actionitem, \observationitem'}}
\end{align*}

Where $\allobservationset_\actionitem$ consists of all the discrete stopping points sampled, and the $\numnoisy$ observations for non-contact.

\subsection{Proving the Bound}
We showed that our utility function is equivalent to the mass removed from the original $\randrealization$: $\hat{f}(\partialrealization) = 1 - \probmassall_{\partialrealization}$. This function can utilize either of the two reweighting functions $w^{HP}$ or $w^{WHP}$ defined in Section~\ref{sec_hypothprune}. Our objective is a truncated version of this: $f(\partialrealization) = \min \left\{Q, \hat{f}(\partialrealization) \right\}$, where $Q$ is the target value for how much probability mass we wish to remove. We assume that the set of all actions $\allactionset$ is sufficient such that $f(\allactionset,\noisyrealization) = Q, \forall \noisyrealization \in \noisyrandrealization$. Note that adaptive monotone submodularity is preserved by truncation, so showing these properties for $\hat{f}$ implies them for $f$.

Using our utility function and observation probability, it is not hard to see that the expected marginal benefit of an action is given by:
\begin{align}
    \Delta(a | \partialgivenactions) &= \mathbb{E}\left[ \hat{f}(\actionset \cup \{a\}, \noisyrandrealization) -   \hat{f}(\actionset, \noisyrandrealization) \right | \partialgivenactions] \notag \\
    &= \sum_{\observationitem \in \allobservationsetaction} p(\observationitem | \partialgivenactions) \left[ (1 - \probmasspartialargaction{\actionset}) - (1 - \probmasspartialarg{\actionset}) \right] \notag \\
    &= \sum_{\observationitem \in \allobservationsetaction} \frac{\probmasspartialargaction{\actionset}}{ \sum_{\observationitem' \in \allobservationsetaction} \probmass_{\partialgivenactions,\actionitem, \observationitem'}} \left[ \probmasspartialarg{\actionset} - \probmasspartialargaction{\actionset} \right]  \label{eq_marginal_hp}
\end{align}

This shows the derivation of the marginal utility, as defined in \sref{sec_hypothprune}. We now provide the proof for Theorem \ref{theorem_hp_as}, by showing that this utility function is adaptive submodular, strongly adaptive monotone, and self-certifying:

\begin{lemma}
    Let $A \subseteq \allactionset$, which result in partial realizations $\partialrealization_A$. Our objective function defined above is strongly adaptive monotone.
\end{lemma}

\begin{proof}
    We need to show that for any action and observation, our objective function will not decrease in value. Intuitively, our objective is strongly adaptive monotone, since we only remove probability mass and never add hypotheses. More formally:
\begin{align*}
  \mathbb{E}\left[ \hat{f}(A, \noisyrandrealization) | \partialrealization_A \right] &\leq \mathbb{E}\left[ \hat{f}(A \cup \{a\}, \noisyrandrealization) | \partialrealization_A, \partialrealization_a = o \right] \\
  \Leftrightarrow 1 - \probmassall_{\partialrealization_A} &\leq 1 - \probmassall_{ \{\partialrealization_A \cup \{a,o\} \} } \\
  \Leftrightarrow 1 - \probmassall_{\partialrealization_A} &\leq 1 - \probmasspartialaction \\
  \Leftrightarrow  \probmasspartialaction &\leq \probmassall_{\partialrealization_A} \\
  \Leftrightarrow \sum_{\realization \in \randrealization} p_{\partialrealization} (\realization) \weighting_{\observationitem} (\actionitem_{\realization'}) &\leq \sum_{\realization \in \randrealization} p_{\partialrealization} (\realization)
\end{align*}

As noted before, both of the weighting functions defined in Section~\ref{sec_hypothprune} never have a value greater than one. Thus each term in the sum from the LHS is smaller than the equivalent term in the RHS.
\end{proof}

\begin{lemma}
    Let $X \subseteq Y \subseteq \allactionset$, which result in partial realizations $\partialrealization_X \subseteq \partialrealization_Y$. Our objective function defined above is adaptive submodular.
\end{lemma}

\begin{proof}
  For the utility function $f$ to be adaptive submodular, it is required that the following holds over expected marginal utilities:
  \begin{align*}
    \Delta (a | \partialrealization_Y) &\leq \Delta (a | \partialrealization_X) \\
    \sum_{\observationitem \in \allobservationsetaction} \frac{\probmasspartialargaction{Y}}{ \sum_{\observationitem' \in \allobservationsetaction} \probmass_{\partialrealization_Y,\actionitem, \observationitem'}} \left[ \probmasspartialarg{Y} - \probmasspartialargaction{Y} \right] 
    &\leq
    \sum_{\observationitem \in \allobservationsetaction} \frac{\probmasspartialargaction{X}}{ \sum_{\observationitem' \in \allobservationsetaction} \probmass_{\partialrealization_X,\actionitem, \observationitem'}} \left[ \probmasspartialarg{X} - \probmasspartialargaction{X} \right] 
  \end{align*}

  We simplify notation a bit for the purposes of this proof. As the action is fixed, we will replace $\allobservationsetaction$ with $\allobservationset$. For a fixed partial realization $\partialrealization_X$ and action $\actionitem$, let $\probmasspartialargaction{X} = \proofprobmass_o$. Let $k_\observationitem = \probmasspartialargaction{X} - \probmasspartialargaction{Y}$, which represents the difference of probability mass remaining between partial realizations $\partialrealization_Y$ and $\partialrealization_X$ if we performed action $\actionitem$ and received observation $o$. We note that $k_o \geq 0 \ \forall o$, which follows from the strong adaptive monotonicity, and $k_o \leq \probmasspartialargaction{X}$, which follows from $\probmasspartialargaction{Y} \geq 0$. Rewriting the equation above:
\begin{align*}
    \sum_{\observationitem \in \allobservationset} \frac{\proofprobmass_o - k_o}{ \sum_{\observationitem' \in \allobservationset} \proofprobmass_{o'} - k_{o'}} \left[ \probmasspartialarg{Y} - \proofprobmass_o + k_o \right]    &\leq     \sum_{\observationitem \in \allobservationset} \frac{\proofprobmass_o}{ \sum_{\observationitem' \in \allobservationset} \proofprobmass_{o'}} \left[ \probmasspartialarg{X} - \proofprobmass_o \right] \\
    \Leftrightarrow \left( \sum_{\observationitem \in \allobservationset} \probmasspartialarg{Y}\proofprobmass_o - \proofprobmass_o^2 + \proofprobmass_o k_o - \probmasspartialarg{Y} k_o + \proofprobmass_o k_o - k_o^2  \right) \left(\sum_{\observationitem' \in \allobservationset} \proofprobmass_{o'}\right) &\leq  \left( \sum_{\observationitem \in \allobservationset} \probmasspartialarg{X}\proofprobmass_o - \proofprobmass_o^2 \right) \left(  \sum_{\observationitem' \in \allobservationset} \proofprobmass_{o'} - k_{o'} \right) \\
    \Leftrightarrow \sum_{\observationitem \in \allobservationset} \sum_{\observationitem' \in \allobservationset}  \probmasspartialarg{Y}\proofprobmass_o\proofprobmass_{o'} - \proofprobmass_o^2\proofprobmass_{o'} + \proofprobmass_o \proofprobmass_{o'} k_o - \probmasspartialarg{Y} \proofprobmass_{o'} k_o + \proofprobmass_o \proofprobmass_{o'} k_o - \proofprobmass_{o'} k_o^2    &\leq    \sum_{\observationitem \in \allobservationset} \sum_{\observationitem' \in \allobservationset}   \probmasspartialarg{X}\proofprobmass_o\proofprobmass_{o'} -\probmasspartialarg{X}\proofprobmass_o k_{o'} - \proofprobmass_o^2  \proofprobmass_{o'} + \proofprobmass_o^2 k_{o'} \\
    \Leftrightarrow \sum_{\observationitem \in \allobservationset} \sum_{\observationitem' \in \allobservationset}  \probmasspartialarg{Y}(\proofprobmass_o\proofprobmass_{o'} - \proofprobmass_{o'} k_o) + 2\proofprobmass_o \proofprobmass_{o'} k_o - \proofprobmass_{o'} k_o^2    &\leq    \sum_{\observationitem \in \allobservationset} \sum_{\observationitem' \in \allobservationset}   \probmasspartialarg{X} (\proofprobmass_o\proofprobmass_{o'} - \proofprobmass_o k_{o'}) + \proofprobmass_o^2 k_{o'} \\
\end{align*}

We also note that $\displaystyle \probmasspartialarg{X} - \probmasspartialarg{Y} \geq \max_{\hat{o} \in \allobservationset}(k_{\hat{o}})$. That is, the total difference in probability mass is greater than or equal to the difference of probability mass remaining if we received any single observation, for any observation.

\begin{align*}
    \Leftrightarrow \sum_{\observationitem \in \allobservationset} \sum_{\observationitem' \in \allobservationset}   2\proofprobmass_o \proofprobmass_{o'} k_o - \proofprobmass_{o'} k_o^2    &\leq    \sum_{\observationitem \in \allobservationset} \sum_{\observationitem' \in \allobservationset}   (\probmasspartialarg{X} - \probmasspartialarg{Y} ) (\proofprobmass_o\proofprobmass_{o'} - \proofprobmass_o k_{o'}) + \proofprobmass_o^2 k_{o'} \\
    \Leftarrow \sum_{\observationitem \in \allobservationset} \sum_{\observationitem' \in \allobservationset}   2\proofprobmass_o \proofprobmass_{o'} k_o - \proofprobmass_{o'} k_o^2    &\leq    \sum_{\observationitem \in \allobservationset} \sum_{\observationitem' \in \allobservationset}  \max_{\hat{o} \in \allobservationset}(k_{\hat{o}}) (\proofprobmass_o\proofprobmass_{o'} - \proofprobmass_o k_{o'}) + \proofprobmass_o^2 k_{o'} \\
    \Leftarrow \sum_{\observationitem \in \allobservationset} \sum_{\observationitem' \in \allobservationset}   2\proofprobmass_o \proofprobmass_{o'} k_o - \proofprobmass_{o'} k_o^2    &\leq    \sum_{\observationitem \in \allobservationset} \sum_{\observationitem' \in \allobservationset}  \max(k_o,k_{o'}) (\proofprobmass_o\proofprobmass_{o'} - \proofprobmass_o k_{o'}) + \proofprobmass_o^2 k_{o'} \\
\end{align*}

In order to show the inequality for the sum, we will show it holds for any pair $o,o'$. First, if $o = o'$, than we have an equality and it holds trivially. For the case when $o \neq o'$, we assume that $k_o > k_{o'}$ WLOG, and show the inequality for the sum:
\begin{align*}
    2\proofprobmass_o \proofprobmass_{o'} (k_o + k_{o'}) - \proofprobmass_{o'} k_o^2 - \proofprobmass_{o} k_{o'}^2    &\leq  2\proofprobmass_o\proofprobmass_{o'}k_o - \proofprobmass_o k_{o'}k_o - \proofprobmass_{o'} k_{o}^2 + \proofprobmass_o^2 k_{o'} + \proofprobmass_{o'}^2 k_{o}\\
    \Leftrightarrow 2\proofprobmass_o \proofprobmass_{o'} k_{o'} - \proofprobmass_{o} k_{o'}^2    &\leq  \proofprobmass_o^2 k_{o'} + \proofprobmass_{o'}^2 k_{o}- \proofprobmass_o k_{o}k_{o'} \\
    \Leftrightarrow 0 &\leq  k_{o'}(\proofprobmass_o - \proofprobmass_{o'})^2 - (k_o - k_{o'}) k_{o'} (\proofprobmass_o - \proofprobmass_{o'}) + (k_o - k_{o'}) \proofprobmass_{o'} (\proofprobmass_{o'} - k_{o'})\\
    \Leftarrow 0 &\leq  k_{o'}(\proofprobmass_o - \proofprobmass_{o'})^2 - (k_o - k_{o'}) k_{o'} (\proofprobmass_o - \proofprobmass_{o'}) + (k_o - k_{o'}) k_{o'} (\proofprobmass_{o'} - k_{o'})
\end{align*}

We split into 3 cases:

\subsection{$k_{o'} = 0$}
This holds trivially, since the RHS is zero

\subsection{$k_{o'} \neq 0, \proofprobmass_{o} \leq 2\proofprobmass_{o'} - k_{o'}$}

Since $k_{o'} \neq 0$, we can rewrite:

\begin{align*}
    0 &\leq  (\proofprobmass_o - \proofprobmass_{o'})^2 - (k_o - k_{o'})(\proofprobmass_o - \proofprobmass_{o'}) + (k_o - k_{o'})(\proofprobmass_{o'} - k_{o'}) \\
    \Leftarrow 0 &\leq  - (k_o - k_{o'})(\proofprobmass_o - \proofprobmass_{o'}) + (k_o - k_{o'})(\proofprobmass_{o'} - k_{o'}) \\
    \Leftarrow (\proofprobmass_o - \proofprobmass_{o'}) &\leq  (\proofprobmass_{o'} - k_{o'}) \\
\end{align*}

Which follows from the assumption for this case.

\subsection{$\proofprobmass_{o} \geq 2\proofprobmass_{o'} - k_{o'}$}

We show this step by induction. Let $\proofprobmass_{o} = 2\proofprobmass_{o'} - k_{o'} + x, x \geq 0$ 

\textbf{Base Case:} $x=0$, which we showed in the previous case.

\textbf{Induction} Assume this inequality holds for $\proofprobmass_{o} = 2\proofprobmass_{o'} - k_{o'} + x$ . Let $\widehat{\proofprobmass_{o}} = \proofprobmass_o + 1$. We now show that this holds for $\widehat{\proofprobmass_{o}}$:

\begin{align*}
    0 &\leq  (\widehat{\proofprobmass_o} - \proofprobmass_{o'})^2 - (k_o - k_{o'})(\widehat{\proofprobmass_o} - \proofprobmass_{o'}) + (k_o - k_{o'})(\proofprobmass_{o'} - k_{o'}) \\
    \Leftrightarrow 0 &\leq  (\proofprobmass_o - \proofprobmass_{o'} + 1)^2 - (k_o - k_{o'})(\proofprobmass_o - \proofprobmass_{o'} + 1) + (k_o - k_{o'})(\proofprobmass_{o'} - k_{o'}) \\
    \Leftrightarrow 0 &\leq  (\proofprobmass_o - \proofprobmass_{o'})^2 - (k_o - k_{o'})(\proofprobmass_o - \proofprobmass_{o'}) + (k_o - k_{o'})(\proofprobmass_{o'} - k_{o'}) + 2\proofprobmass_o - 2\proofprobmass_{o'} + 1  + k_o - k_{o'}\\
    \Leftarrow 0 &\leq  2\proofprobmass_o - 2\proofprobmass_{o'} + 1  - k_o + k_{o'} & \text{\hspace{5mm} by inductive hypothesis}\\
    \Leftarrow 0 &\leq  \proofprobmass_o + 1  - k_o & \text{\hspace{5mm} by assumption from case}\\
    \Leftarrow 0 &\leq 1 \\
\end{align*}

And thus, we have shown the inequality holds for any pair $o,o'$. \QED

Finally, it is easy to see that the sum can be decomposed into pairs of $o,o'$. Therefore, we can see the inequality over the sum also holds.
\end{proof}

\begin{lemma}
    Let $A \subseteq \allactionset$, which result in partial realizations $\partialrealization_A$. The utility function $f$ defined above is self-certifying.
\end{lemma}

\begin{proof}
    An instance is self-certifying if whenever the maximum value is achieved for the utility function $f$, it is achieved for all realizations consistent with the observation. See~\cite{golovin_adaptive_2011} for a more rigorous definition. Golovin and Krause point out that any instance which only depends on the state of items in $A$ is automatically self-certifying (Proposition $5.6$ in~\cite{golovin_adaptive_2011}.) That is the case here, since the objective function $f = \min\left\{ Q, 1 - M_{\partialrealization_A} \right\}$ only depends on the outcome of actions in $A$. Therefore, our instance is self-certifying.
\end{proof}

As we have shown our objective is adaptive submodular, strongly adaptive monotone, and self-certifying, Theorem~\ref{theorem_hp_as} follows from Theorems 5.8 and 5.9 from~\cite{golovin_adaptive_2011}. Following their notation, we let $\eta$ be any value such that $f(\partialrealization) > Q - \eta$ implies $f(\partialrealization) \geq Q$ for all $\partialrealization$. For Hypothesis Pruning, for example, we have $\eta = \min_\realization p(\realization)$. Additionally, the bound on the worst case cost includes $\delta = \min_{\noisyrealization} p(\noisyrealization) \ \noisyrealization \in \noisyrandrealization$. The specific values of these constants are related to the weighting function and how discretization is done. Nonetheless, for either weighting function and any way we discretize, we can guarantee the greedy algorithm selects a near-optimal sequence.

\end{document}